%% file: example_paper.tex
\documentclass{article}

\usepackage{microtype}
\usepackage{graphicx}
\usepackage{subfigure}
\usepackage{booktabs} 

\usepackage{hyperref}

\usepackage[accepted]{icml2023}

\usepackage{amsmath}
\usepackage{amssymb}
\usepackage{mathtools}
\usepackage{amsthm}
\usepackage{xcolor}
\usepackage{chngcntr}
\usepackage{multirow}
\usepackage{longtable}

\usepackage[capitalize,noabbrev]{cleveref}

\input{math_commands.tex}

\theoremstyle{plain}
\newtheorem{theorem}{Theorem}[section]

\newtheorem{lemma}[theorem]{Lemma}
\newtheorem{corollary}[theorem]{Corollary}
\theoremstyle{definition}
\newtheorem{definition}[theorem]{Definition}
\newtheorem{assumption}[theorem]{Assumption}
\theoremstyle{remark}
\newtheorem{remark}[theorem]{Remark}

\theoremstyle{plain}
\newtheorem{innerrepeatthm}{Theorem}
\newenvironment{repeatthm}[1]
{\renewcommand\theinnerrepeatthm{#1}\innerrepeatthm}
{\endinnerrepeatthm}

\newtheorem{innerrepeatcoro}{Corollary}
\newenvironment{repeatcoro}[1]
{\renewcommand\theinnerrepeatcoro{#1}\innerrepeatcoro}
{\endinnerrepeatcoro}

\usepackage[textsize=tiny]{todonotes}

\icmltitlerunning{Shortest Edit Path Crossover}

\begin{document}

\twocolumn[
\icmltitle{Shortest Edit Path Crossover: A Theory-driven Solution to the Permutation Problem in Evolutionary Neural Architecture Search}




\begin{icmlauthorlist}
	\icmlauthor{Xin Qiu}{cognizant}
	\icmlauthor{Risto Miikkulainen}{cognizant,ut}
\end{icmlauthorlist}

\icmlaffiliation{cognizant}{Cognizant AI Labs, San Francisco, USA}
\icmlaffiliation{ut}{The University of Texas at Austin, Austin, USA}

\icmlcorrespondingauthor{Xin Qiu}{qiuxin.nju@gmail.com}
\icmlcorrespondingauthor{Risto Miikkulainen}{risto@cognizant.com}

\icmlkeywords{Neural Architecture Search, Crossover, Evolutionary Computation}

\vskip 0.3in
]



\printAffiliationsAndNotice{}  

\begin{abstract}
	Population-based search has recently emerged as a possible alternative to Reinforcement Learning (RL) for black-box neural architecture search (NAS). It performs well in practice even though it is not theoretically well understood. In particular, whereas traditional population-based search methods such as evolutionary algorithms (EAs) draw much power from crossover operations, it is difficult to take advantage of them in NAS. The main obstacle is believed to be the permutation problem: The mapping between genotype and phenotype in traditional graph representations is many-to-one, leading to a disruptive effect of standard crossover. This paper presents the first theoretical analysis of the behaviors of mutation, crossover and RL in black-box NAS, and proposes a new crossover operator based on the shortest edit path (SEP) in graph space. The SEP crossover is shown theoretically to overcome the permutation problem, and as a result, have a better expected improvement compared to mutation, standard crossover and RL. Further, it empirically outperform these other methods on state-of-the-art NAS benchmarks. The SEP crossover therefore allows taking full advantage of population-based search in NAS, and the underlying theory can serve as a foundation for deeper understanding of black-box NAS methods in general.
\end{abstract}

\section{Introduction}
Neural architecture search (NAS), a technique for automatically designing architectures for neural networks, outperforms human-designed models in many tasks \citep{Zoph18,Chen18,miikkulainen:gecco21}. One major branch of NAS approaches are the black-box NAS methods, which require only zeroth-order information about the objectives. While reinforcement learning (RL) contributed to the early success of black-box NAS methods \citep{Zoph17}, population-based search has emerged recently as a popular and empirically more powerful alternative \citep{Real17, Real19, Ying19}, achieving SOTA performance in various benchmarks and real-world domains \citep{Real17,ElskenMH19,Real19,David21,Gao22}. Population-based NAS is usually based on evolutionary algorithms (EAs) \citep{Liu21}, which mimic natural evolution by maintaining a population of solutions and evolving them through mutation and crossover. Mutation provides for local search (i.e.\ refinement), while crossover implements a directed global search, and thus constitutes the engine behind evolutionary discovery.  However, most recent evolutionary NAS methods are limited to mutation only \citep{Real17,Fernando17,Liu18,ElskenMH19,Real19,David21,Co-reyes21,Gao22}, which has also been used extensively in simple hill-climbing/local search methods \citep{White21a, White21b}.

The main obstacle in applying crossover to NAS is the \emph{permutation problem} \citep{Radcliffe92,Radcliffe93}, also known as the \emph{competing conventions problem} \citep{Montana89,Schaffer92}. This problem is due to isomorphisms in graph space, i.e., functionally identical architectures are mapped to different encodings/representations, making crossover operations disruptive. A number of possible solutions to this problem have been proposed in the neuroevolution community \cite{Thierens96,Stanley02,Dragoni14,Mahmood07,Wang18,Uriot20}. However, they either only work on fixed or constrained network topologies, or are limited to one particular algorithm or search space; none of them generalize to arbitrary graphs or architectures such as those that might arise from NAS. Moreover, prior work has focused only on empirical verification without a theoretical analysis of potential solutions. Theoretical understanding of search efficiency of mutation, crossover, and RL is still lacking in black-box NAS.

To meet the above challenges, this paper first proposes a new crossover operator based on shortest edit path (SEP) in the original graph space. The SEP crossover does not impose any constraints on other algorithmic components or application scope, thereby forming a simple and generalizable solution to the permutation problem. Second, a theory is derived for analyzing mutation, standard crossover, RL, and the proposed SEP crossover in the NAS domain. The SEP crossover is shown to have the best expected improvement in terms of graph edit distance (GED) between the found architecture and the global optimum. Third, empirical results on SOTA NAS benchmarks further verify the theoretical analysis, demonstrating that the SEP approach is effective. It thus allows taking full advantage of population-based search, and serves as a theoretical foundation for further research on methods for NAS and similar problems. All source codes for reproducing the experimental results are provided at: (\href{https://github.com/cognizant-ai-labs/sepx-paper}{https://github.com/cognizant-ai-labs/sepx-paper}).

\section{Related Work}
\textbf{NAS}\hspace{5pt} NAS approaches can generally be categorized into two groups: one-shot methods and black-box methods \citep{Mehta22}. In one-shot approaches \citep{Liu19, Dong19, Chen21}, a supernet is trained to represent the entire search space. The overall training cost is reduced significantly; however, these approaches can only be run on small cell-based search spaces with a complete graph \citep{Mehta22, Zela20} and the search objectives must be differentiable. In contrast, although computationally more expensive, black-box methods have no restrictions on the search space or objectives, making them a more general solution to NAS. Thus, this paper will focus on black-box NAS.

\textbf{Black-box NAS}\hspace{5pt} Black-box NAS methods, also called zeroth-order methods, iteratively generate architectures for evaluation, and then use the outcome to update the search strategy. There are four main types of search strategies in black-box NAS approaches: random search \citep{Li20, Yu20}, RL \citep{Zoph17, Zoph18}, evolutionary search \citep{Real17, Real19}, and local search \citep{White21a, White21b}. Local search, whether used alone or together with  neural predictors (e.g., Bayesian models), is based on operations that are essentially the same as mutation (although different terminology may be used) \citep{White21a, White21b}. They can therefore be seen as equivalent to mutation-only evolutionary search with a population size of one. The one search strategy that is significantly different from evolutionary methods is RL, and will thus be included in the theoretical analysis in this paper. The theory developed in this paper thus covers most of the search strategies in Black-box NAS. 

\textbf{RL-based black-box NAS}
RL-based methods work by iteratively sampling architectures using a RL agent, then collecting the prediction accuracies as the reward for updating the policy. \citet{Zoph17} successfully generated well-performing convolutional networks and recurrent cells using a specially designed recurrent neural network as the agent. \citet{Zoph18} further showed that the approach finds architectures that transfer well between different datasets. In a recent empirical study \citep{Ying19}, a simple RL agent based on multinomial probability distributions was found to perform significantly better on NAS-bench-101 than previous RL-based NAS methods. This RL controller is analyzed in this paper as well.

\textbf{Evolutionary Black-box NAS}
Evolutionary NAS methods work by improving a population of architectures over time \citep{Liu21}. To generate new offspring architectures, two operators can be used: a random edge/node mutation applied to an existing architecture, and crossover to recombine two existing architectures. Architectures that do not perform well are removed periodically from the population, and the best-performing architecture returned in the end.  While crossover is a powerful operator, most existing methods rely on mutation only because of the permutation problem. It is this problem that this work aims to solve, in order to take full advantage of the evolutionary approch in NAS.

\textbf{The permutation problem and existing solutions}\hspace{5pt} The permutation problem has been discussed in the Neuroevolution community for many years. One simple but common solution is simply to get rid of crossover completely during evolution \citep{Angeline94, Yao98}. Indeed, almost all newly developed evolutionary NAS methods avoid using a crossover operator \citep{Real17,Fernando17,Liu18,ElskenMH19,Real19,David21,Co-reyes21,Gao22}. For instance, \citet{Real17} reported that crossover operators were included in their initial experiments, but no performance improvement was observed, and therefore only mutation was deployed in their final AmoebaNet algorithm.

A number of principled solutions have been proposed to overcome the permutation problem as well. Many of them require that the network topologies are fixed. For instance, \citet{Thierens96} proposed a non-redundant encoding for matching neurons during crossover, \citet{Uriot20} developed a safe crossover through a neural alignment mechanism, and \citet{Gangwani18} used genetic distillation to improve crossover. Further, \citet{Dragoni14} proposed a generalization where the population can include different topologies, but only parents with a similar topology can be crossed over.

Other solutions have been developed for special cases, making them non-applicable to arbitrary architectures. For instance, the unit-alignment method \citep{Sun20} utilizes a special encoding that is only for CNN-based architectures. A graph matching recombination operator \citep{Mahmood07} only applies to parents with very different qualities. It mimics the behaviors of mutating the weaker parent towards the stronger parent, so the offspring does not differ from parents greatly. A modular inheritable crossover \citep{He21} is developed for a specific cell-based structure, and the default order of the nodes is preserved when performing crossover, without any node matching or reordering. As a result, the permutation problem still remains. The historical markings in NEAT-based algorithms \citep{Stanley02, Miikkulainen19} are intended to be used together with other mechanisms in NEAT, and cannot be directly applied to any given architectures. 

In contrast to these existing solutions, the proposed SEP crossover does not have any constraints on the encoding or other algorithmic components, and can be directly applied to any arbitrary architectures.

\section{The Shortest Edit Path Crossover}
In this section, the permutation problem is first described and a solution to it proposed in the form of Shortest Edit Path Crossover.

Given two neural architectures as parents, a crossover operator generates an offspring architecture by recombining the two parents. The crossover design consists of the encoding (i.e. genotype) and the recombination strategy, with the goal of properly integrating the information in both parents. The permutation problem arises when the same architecture (i.e. phenotype) can have  multiple distinct genotypes. As a result, crossover on these genotypes has a disruptive effect on the information encoded in the parents, leading to damaged offspring \citep{Stanley02}.

In order to propose a solution, let us first define a representation of the neural network architecture and a distance metric between two architectures. A neural architecture is a computation graph that can always be represented by an attributed directed graph, defined as:
\begin{definition}[Directed graph]
	A directed graph $\gG$ consists of a set of vertices $\sV=\{v_i | i=1, 2, \ldots, n\}$, where $n$ is the number of vertices and each $v_i$ denotes a vertex (node), and a set of directed edges $\mathbb{E}=\{e_{i,j} | i, j \in {1, 2, \ldots, n}\}$, where $e_{i,j}$ denotes a directed edge from $v_i$ to $v_j$. The order of a directed graph $\gG$ equals the number of its vertices, represented by $|\gG|$. For an attributed directed graph, a function $\gamma_v$ assigns an attribute (e.g., an integer) to each vertex, and a funtion $\gamma_e$ assigns an attribute to each edge.
\end{definition}
In the context of NAS, each vertex with an attribute denotes an operation in a neural architecture, and the directed edges denote data flows. The similarity between two architectures can then be measured by the graph edit distance (GED) between their corresponding graphs, defined as:
\begin{definition}[Graph edit distance]
	A graph edit operation is defined as a function $\delta: \gG \rightarrow \gG'$ that applies an elementary graph edit to transform $\gG$ to $\gG'$. In standard neural architecture search, the set of elementary graph edits typically includes vertex deletion/insertion, edge deletion/insertion, and vertex attribute substitution. An edit path is defined as a sequence of graph edit operations $\overline{\delta}=\delta_1, \delta_2, \ldots, \delta_d$, where $d$ is the length of the edit path. Application of $\overline{\delta}$ to a graph is equivalent to applying each edit sequentially: $\overline{\delta}(\gG)=\delta_d\circ \ldots\circ \delta_2\circ \delta_1(\gG)$. Graph edit distance between two graphs $\gG_1$ and $\gG_2$ is then defined as $\mathrm{GED}(\gG_1, \gG_2)=\min_{\overline{\delta}\in\Delta(\gG_1, \gG_2)} \sum_{i=1}^{d}{c(\delta_i)}$,
	where $\Delta(\gG_1, \gG_2)$ denotes the set of all edit paths that transform $\gG_1$ to an isomorphism of $\gG_2$ (including $\gG_2$ itself), $\overline{\delta}=\delta_1, \delta_2, \ldots, \delta_d$, and $c(\delta_i)$ is the cost of edit $\delta_i$. In this work, all types of edit operations are defined to have the same cost of $1$. As a result, the edit path that minimizes the total edit cost, $\overline{\delta}_{\gG_1, \gG_2}^*$, equals the shortest edit path between $\gG_1$ and $\gG_2$. Thus, $\mathrm{GED}(\gG_1, \gG_2)=d_{\gG_1, \gG_2}^*$, where $d_{\gG_1, \gG_2}^*$ is the length of this shortest edit path. Note that $\overline{\delta}_{\gG_1, \gG_2}^*$ may not be unique, and thus there may exist multiple shortest edit paths that have the same length.
\end{definition}
The proposed SEP crossover is then defined as
\begin{definition}[Shortest edit path (SEP) crossover]
	Given two attributed directed graphs $\gG_1$ and $\gG_2$, suppose $\overline{\delta}_{\gG_1, \gG_2}^*=\delta^*_1, \delta^*_2, \ldots, \delta^*_{d_{\gG_1, \gG_2}^*}$. SEP crossover generates an offspring graph $\gG_{\mathrm{new}}$ by 
	\begin{align*}
	\gG_{\mathrm{new}} = &\delta^*_{\pi_r(\lceil\frac{d_{\gG_1, \gG_2}^*}{2}\rceil)}\circ\delta^*_{\pi_r(\lceil\frac{d_{\gG_1, \gG_2}^*}{2}\rceil-1)}\circ\delta^*_{\pi_r(\lceil\frac{d_{\gG_1, \gG_2}^*}{2}\rceil-2)}\\&\circ\ldots\circ \delta^*_{\pi_r(2)}\circ \delta^*_{\pi_r(1)}(\gG_1),
	\end{align*}
	where $\pi_r$ is a random permutation of the $d_{\gG_1, \gG_2}^*$ indices: $\pi_r: 1, 2, \ldots, d_{\gG_1, \gG_2}^* \rightarrow \pi(1), \pi(2), \ldots, \pi(d_{\gG_1, \gG_2}^*)$, and $\lceil \cdot \rceil$ denotes the ceiling function. In other words, the SEP crossover shuffles the edits randomly in the SEP between parents, then selects half of them randomly, and applies them to one of the parents to obtain the offspring.
\end{definition}
This operator is motivated by a common observation in the literature \citep[e.g.][]{Ying19,White21a,Mehta22} that the differences in predictive performance between two architectures are positively correlated with their GEDs. This observation suggests that the edits in the SEP encode fundamental differences between two architectures that matter to predictive performance. An offspring that lies in the middle of this SEP can explore the search regions where the parents have fundamental discrepancies. At the same time, the offspring can automatically preserve those common substructures between parents, avoiding unnecessary disruptive behaviors, and thus avoiding the permutation problem. A visual demo showing how the SEP crossover resolves the permutation problem is provided in Appendix~\ref{subsec:add_sep_demo}.

\section{Theoretical Analysis}

In this section, the SEP crossover, standard crossover, mutation, and RL approaches to NAS will be analyzed theoretically, showing that the SEP crossover has an advantage in improving the expected quality of generated graphs. The fundamental concepts are defined first in Section~\ref{subsec:cross_muta}, leading to new interpretations of graph edit distance, crossover and mutation based on \emph{attributed adjacency matrices}. Feasibility assumptions are then declared, and theorems derived for expected improvement for SEP, standard crossover, and mutation. Section~\ref{subsec:RL_theory} focuses on RL: It interprets RL in terms of the same fundamental concepts, defines two extreme cases whose combinations span the possible states of the RL process, and derives theorems for expected improvement for both. Section~\ref{subsec:theory_comp} then brings these theorems together, showing that the SEP crossover results in more improvement than the other methods in common NAS setups. Section~\ref{subsec:error_GED} further verifies the robustness of the SEP crossover under inaccurate GED calculations. All proofs and lemmas are included in Appendix~\ref{subsec:app_proof}. For clarify, a full list of mathematical symbols is provided in Appendix~\ref{subsec:add_notation}.

\subsection{Expected Improvement with Crossover and Mutation}\label{subsec:cross_muta}
First, let us define the basic concepts:
\begin{definition}[Attributed adjacency matrix]
	An attributed adjacency matrix (AA-matrix) $\mA_\gG$ is a representation of an attributed directed graph. It is a $n\times n$ matrix, where $n$ is the number of vertices in $\gG$. The entry in $i$th row and $j$th column is represented by $A^\gG_{i,j}$. $A^\gG_{i,j}=0$ if there is no edge from $v_i$ to $v_j$, and $A^\gG_{i,j}=\gamma_e(e_{i,j})$ if there exists an edge from $v_i$ to $v_j$, for $i, j \in {1, 2, \ldots, n}$ and $i\neq j$. $A^\gG_{i,i} = \gamma_v(v_i)$, for $i\in {1, 2, \ldots, n}$.
\end{definition}

\begin{definition}[Permutation matrix]
	Given a permutation $\pi$ of $n$ elements: $\pi: 1, 2, \ldots, n \rightarrow \pi(1), \pi(2), \ldots, \pi(n)$, a permutation matrix $\mP_{\pi}$ can be constructed by permuting the columns or rows of an $n\times n$ identity matrix $\mI_n$ according to $\pi$. In this work, a column permutation of $\mI_n$ is performed to obtain $\mP_{\pi}$, The entry in $i$th row and $j$th column is represented by $P^{\pi}_{i,j}$, and $P^{\pi}_{i,j}=1\ \mathrm{if}\ j=\pi(i), \mathrm{and}\ P^{\pi}_{i,j}=0\ \mathrm{otherwise.}$
\end{definition}

\begin{definition}[Null vertex]
	A null vertex has no connections to other existing vertices in a graph. It is assigned with a special "null" attribute, which means that it does not have any impact on the original graph. Null vertices are added to an existing graph only for convenience of theoretical analysis, and they do not affect the calculation of GEDs. 
\end{definition}

Based on the above definitions, GED, crossover and mutation can be interpreted from the AA-matrix perspective:

\begin{definition}[AA-matrix-based interpretation of GED]
	Two graphs $\gG_1$ and $\gG_2$ can both be extended to have the same order $n=\max(|\gG_1|,|\gG_2|)$ by adding null vertices. The extended $\gG_1$ and $\gG_2$ are denoted as $\hat{\gG}_1$ and $\hat{\gG}_2$. Calculating the GED between $\gG_1$ and $\gG_2$ can then be defined as
	\begin{equation*}
	\mathrm{GED}(\gG_1, \gG_2)=\min_{\pi\in S_n}d(\mA_{\hat{\gG}_1}, \mP_{\pi}\mA_{\hat{\gG}_2}\mP_{\pi}^{\top}),
	\end{equation*}
	where $d(\mA,\mB)=\sum_{i=1}^m\sum_{j=1}^n\1_{A_{i,j}\neq B_{i,j}}$, $m\times n$ is the order of both $\mA$ and $\mB$, $\1_\mathrm{condition}$ is 1 if the condition is true, 0 otherwise (i.e.,  $d(\mA,\mB)$ counts the number of different entries between two matrices with same shape), $S_n$ denotes the set of all permutations of $\{1, 2, 3, \ldots, n\}$. The permutation that minimizes $d(\mA_{\hat{\gG}_1}, \mP_{\pi}\mA_{\hat{\gG}_2}\mP_{\pi}^{\top})$ is denoted as $\pi^*_{\hat{\gG}_1,\hat{\gG}_2}$, and the permuted AA-matrix of $\hat{\gG}_2$ is denoted as $\mA_{\hat{\gG}_2\rightarrow\hat{\gG}_1}=\mP_{\pi^*_{\hat{\gG}_1,\hat{\gG}_2}}\mA_{\hat{\gG}_2}\mP_{\pi^*_{\hat{\gG}_1,\hat{\gG}_2}}^{\top}$. 
\end{definition}

\begin{remark}\label{rmk:NAS_conext}
	In the context of standard neural architecture search, assume $\gamma_e(\cdot)$ always assigns 1 to any existing edge, and $\gamma_v(\cdot)$ assigns 0 to "null" vertex and positive integers for other types of vertex attributes (each type of attribute has its own unique integer). Then the differences between $\mA_{\hat{\gG}_1}$ and $\mA_{\hat{\gG}_2\rightarrow\hat{\gG}_1}$ correspond to the shortest edit path that transforms $\gG_1$ to $\gG_2$ in the following way: $\delta:=$ (1) add a vertex with attribute $ A^{\hat{\gG}_2\rightarrow\hat{\gG}_1}_{i,i}$, if $A^{\hat{\gG}_1}_{i,i}=0\ \mathrm{and\ } A^{\hat{\gG}_2\rightarrow\hat{\gG}_1}_{i,i}>0$; (2) delete vertex $v_i$ from $\gG_1$, if $A^{\hat{\gG}_1}_{i,i}>0\ \mathrm{and\ } A^{\hat{\gG}_2\rightarrow\hat{\gG}_1}_{i,i}=0$; (3) change attribute of vertex $v_i$ to $A^{\hat{\gG}_2\rightarrow\hat{\gG}_1}_{i,i}$, if $A^{\hat{\gG}_1}_{i,i}>0$ and $A^{\hat{\gG}_2\rightarrow\hat{\gG}_1}_{i,i}>0$; (4) add an edge from $v_i$ to $v_j$, if $A^{\hat{\gG}_1}_{i,j}=0$ and $A^{\hat{\gG}_2\rightarrow\hat{\gG}_1}_{i,j}=1, i\neq j$; (5) delete the edge from $v_i$ to $v_j$, if $A^{\hat{\gG}_1}_{i,j}=1$ and $A^{\hat{\gG}_2\rightarrow\hat{\gG}_1}_{i,j}=0, i\neq j$. Note that when adding an edge, the origin $v_i$ and/or destination $v_j$ may be newly added vertices.
\end{remark}

\begin{definition}[AA-matrix-based interpretation of crossover]\label{def:AA_crossover}
	Assume two graphs $\gG_1$ and $\gG_2$ are extended to have the same order by adding null vertices, resulting $\hat{\gG}_1$ and $\hat{\gG}_2$. A crossover between $\gG_1$ and $\gG_2$ is defined as the process of generating an offspring graph $\gG_{\mathrm{new}}$ by recombining $\mA_{\hat{\gG}_1}$ and $\mA_{\hat{\gG}_2}$: $\mA_{\hat{\gG}_{\mathrm{new}}}=r(\mA_{\hat{\gG}_1}, \mP_{\pi}\mA_{\hat{\gG}_2}\mP_{\pi}^{\top})$,
	where function $r(\mA, \mB)$ returns a matrix that inherits each entry from $\mA$ or $\mB$ with probability 0.5. That is, if $\mC=r(\mA, \mB)$, then $p(C_{i,j}=A_{i,j})=p(C_{i,j}=B_{i,j})=0.5$ for any valid $i, j$. $\mP_{\pi}$ is a permutation matrix based on permutation $\pi$, which is decided by the specific crossover operator utilized. For the SEP crossover, $\pi=\pi^*_{\hat{\gG}_1,\hat{\gG}_2}$, which minimizes the GED between $\gG_1$ and $\gG_2$. For the standard crossover, since the vertices may be in any order in the original AA-matrix representation and there is no particular vertex/edge matching mechanisms during crossover, a purely random permutation $\pi_\mathrm{rand}$ is used to represent this randomness. The result, $\mA_{\hat{\gG}_{\mathrm{new}}}$, is the AA-matrix of the generated new graph with null vertices. By removing all null vertices from $\hat{\gG}_{\mathrm{new}}$, the offspring graph $\gG_{\mathrm{new}}$ is obtained.
\end{definition}

\begin{definition}[AA-matrix-based interpretation of mutation]\label{def:AA_mutation}
	Given a graph $\gG_1$, a mutation operation is defined as the process of generating an offspring graph $\gG_{\mathrm{new}}$ by mutating $\gG_1$. In standard NAS, allowed mutations to $\gG_1$ include vertex deletion/insertion, edge deletion/insertion, and vertex attribute substitution. In the AA-matrix representation, a mutation operation is then defined as $\mA_{\hat{\gG}_{\mathrm{new}}}=m(\mA_{\hat{\gG}_1})$, where function $m(\mA)$ alters each element of $\mA$ with an equal probability $p_m$. and $p_m$ is usually selected so that on average one element is altered during each mutation operation. The $\hat{\gG}_1$ is the extended graph of $\gG_1$ with null vertices, so that node additions can be performed in $\mA_{\hat{\gG}_1}$ (by changing a null vertex to a vertex with a valid attribute). An element $A^{\hat{\gG}_1}_{i,j}$ can be altered in order to randomly resample an allowed value that is different from the original $A^{\hat{\gG}_1}_{i,j}$. The result, $\mA_{\hat{\gG}_{\mathrm{new}}}$, is the AA-matrix of the generated new graph with null vertices. By removing all null vertices from $\hat{\gG}_{\mathrm{new}}$, the mutated offspring graph $\gG_{\mathrm{new}}$ is obtained.
\end{definition}

Next, in order to define a performance metric for comparing different crossover and mutation operators, a realistic assumption needs to be made about the search space:

Locality in NAS search spaces means that close architectures (in terms of GED) tend to have similar performance. Random-walk autocorrelation \citep[RWA;][]{Weinberger20} is a commonly used metric to measure such locality. Strong autocorrelation of prediction accuracies of architectures during a random walk, in which each move is a graph edit operation, has been consistently observed in many existing NAS benchmarks or studies \citep{Ying19,White21a,Mehta22}. This observation leads to the following assumption: 

\begin{assumption}[Positive correlation between GED and fitness/reward difference]\label{asm:corr_GED_fitness}
	If $\mathrm{GED}(\gG_i, \gG_j)<\mathrm{GED}(\gG_i, \gG_k)$, then $\E(|f(\gG_i)-f(\gG_j)|)<\E(|f(\gG_i)-f(\gG_k)|)$, where $f(\gG)$ returns the fitness/reward of $\gG$, i.e., the prediction accuracy.
\end{assumption}

Suppose $\gG_{\mathrm{opt}}$ is the global optimal graph (i.e.\ the target of the evolutionary search), $\gG_1$ and $\gG_2$ are the two parents to undergo crossover or mutation, and $\gG_{\mathrm{new}}$ is the generated offspring. For convenience of theoretical analysis, $\gG_{\mathrm{opt}}$, $\gG_1$, and $\gG_2$ are extended to have the same order $n=\max(|\gG_{\mathrm{opt}}|, |\gG_1|, |\gG_2|)$ by adding null vertices. The extended $\gG_{\mathrm{opt}}$ is denoted as $\hat{\gG}_{\mathrm{opt}}$, and $\hat{\gG}_1$, $\hat{\gG}_2$ and $\hat{\gG}_{\mathrm{new}}$ have the same meaning as in Definitions~\ref{def:AA_crossover} and \ref{def:AA_mutation}. 

Given assumption~\ref{asm:corr_GED_fitness}, a direct measurement of the progress of the entire search is $\mathrm{GED}(\gG_{\mathrm{opt}}, \gG_{\mathrm{new}})$, and the ultimate goal is to minimize it so that a good solution can be generated. $\mathrm{GED}(\gG_{\mathrm{opt}}, \gG_{\mathrm{new}})=d_{\gG_{\mathrm{opt}}, \gG_{\mathrm{new}}}^*=d(\mA_{\hat{\gG}_{\mathrm{opt}},}, \mA_{\hat{\gG}_{\mathrm{new}}\rightarrow\hat{\gG}_{\mathrm{opt}},})$ can be decomposed to $d_v(\mA_{\hat{\gG}_{\mathrm{opt}}}, \mA_{\hat{\gG}_{\mathrm{new}}\rightarrow\hat{\gG}_{\mathrm{opt}}})+d_e(\mA_{\hat{\gG}_{\mathrm{opt}}}, \mA_{\hat{\gG}_{\mathrm{new}}\rightarrow\hat{\gG}_{\mathrm{opt}}})$, where $d_v(\mA, \mB)=\sum_i\1_{A_{i,i}\neq B_{i,i}}$ counts only the number of different diagonal entries, i.e., the differences in vertex attributes, and $d_e(\mA, \mB)=\sum_i\sum_{j\neq i}\1_{A_{i,j}\neq B_{i,j}}$ counts the number of different non-diagonal entries, i.e., the differences in edges/connections, thereby measuring the topological similarity.

In order to derive a performance metric, let's consider two factors. First, $d_v(\cdot)$ only covers $n$ elements, whereas $d_e(\cdot)$ covers $n\cdot(n-1)$ elements. We have $n\cdot(n-1)\gg n$ when $n$ increases, so $d_e(\cdot)$ is a dominant factor in deciding $\mathrm{GED}(\gG_{\mathrm{opt}}, \gG_{\mathrm{new}})$. Second, modeling of vertex attributes varies a lot across different NAS spaces, e.g., they have different numbers of usable attributes and different constraints on vertex attribute assignments. In contrast, $\gamma_e(\cdot)=1$ can simply be used for all valid edges in most NAS spaces, leading to generality of any theoretical conclusions. These two factors suggest that $d_e(\mA_{\hat{\gG}_{\mathrm{opt}}}, \mA_{\hat{\gG}_{\mathrm{new}}\rightarrow\hat{\gG}_{\mathrm{opt}}})$ is a representative quantitative metric when comparing different crossover and mutation operators theoretically. For simplicity, we will use $d_{e, \gG_1, \gG_2}^*$ to denote $d_e(\mA_{\gG_1}, \mA_{\gG_2\rightarrow\gG_1})$.

Accordingly, the main performance metric for crossover and mutation can now be defined as follows:

\begin{definition}[Expected improvement of crossover and mutation]
	This work focuses on the expected improvement in terms of topological similarity to the global optimal graph. More specifically, expected improvement refers to $\E(\max(d_e(\mA_{\hat{\gG}_{\mathrm{opt}}}, \mA_{\hat{\gG}_1\rightarrow\hat{\gG}_{\mathrm{opt}}})-d_e(\mA_{\hat{\gG}_{\mathrm{opt}}}, \mA_{\hat{\gG}_{\mathrm{new}}\rightarrow\hat{\gG}_{\mathrm{opt}}}),0))$, which compares offspring graph $\gG_{\mathrm{new}}$ with one parent graph $\gG_1$ in terms of the expected edge/connection differences to $\gG_{\mathrm{opt}}$. The $\max(\cdot,0)$ part takes into account the selection pressure in standard EAs; that is, only the offspring that is better than its parent can survive and become the next parent.
\end{definition}

As the penultimate step, three lemmas are derived in Appendix~\ref{subsec:app_proof} to assist the proofs regarding expected improvement. According to Lemmas~\ref{lem:invar_crossover} and \ref{lem:invar_mutation}, any $\pi'$ can be chosen to analyze the behaviors of SEP crossover, standard crossover, and mutation, without affecting the result of $d_e(\mA_{\hat{\gG}_{\mathrm{opt}}}, \mA_{\hat{\gG}_{\mathrm{new}}\rightarrow\hat{\gG}_{\mathrm{opt}}})$. Lemma~\ref{lem:common_parts} further derives the lower bound for common parts in $\gG_{\mathrm{opt}}$, $\gG_1$ and $\gG_2$. Now, choose $\pi'=\pi_1=\pi^*_{\hat{\gG}_{\mathrm{opt}},\hat{\gG}_1}$ and $\pi_2=\pi^*_{\hat{\gG}'_1,\hat{\gG}_2}$ so that there are at least $n_s=\max(n^2-d_{\hat{\gG}_{\mathrm{opt}}, \hat{\gG}_1}^*-d_{\hat{\gG}_1, \hat{\gG}_2}^*, 0)$ common entries among $\mA_{\hat{\gG}_{\mathrm{opt}}}$, $\mA_{\hat{\gG}_1\rightarrow\hat{\gG}_{\mathrm{opt}}}$ and $\mA_{\hat{\gG}_2\rightarrow\hat{\gG}'_1}$, where $\mA_{\hat{\gG}_1\rightarrow\hat{\gG}_{\mathrm{opt}}}=\mA_{\hat{\gG}'_1}$. Regarding the remaining entries, the following assumption is made:
\begin{assumption}[Uniform distribution of differences]\label{asm:diff_entry}
	The entries that are different between $\mA_{\hat{\gG}'_1}$ and $\mA_{\hat{\gG}_2\rightarrow\hat{\gG}'_1}$ are assumed to be uniformly distributed on the positions other than those $n_s$ common entries.
\end{assumption} 
With these lemmas and assumption, the expected improvement of SEP crossover, standard crossover and mutation can be derived:

\begin{theorem}[Expected improvement of SEP crossover]\label{thm:EI_SEPX} Following Assumption~\ref{asm:diff_entry}, let $n_{se}=\max(n\cdot(n-1)-d_{e,\hat{\gG}_{\mathrm{opt}}, \hat{\gG}_1}^*-d_{e,\hat{\gG}_1, \hat{\gG}_2}^*, 0)$. and suppose $\mA_{\hat{\gG}_{\mathrm{new}}}=r(\mA_{\hat{\gG}'_1}, \mP_{\pi^*_{\hat{\gG}'_1,\hat{\gG}_2}}\mA_{\hat{\gG}_2}\mP_{\pi^*_{\hat{\gG}'_1,\hat{\gG}_2}}^{\top})$. Then we have
{\footnotesize \begin{align*}
\displaystyle &\E(\max(d_e(\mA_{\hat{\gG}_{\mathrm{opt}}}, \mA_{\hat{\gG}_1\rightarrow\hat{\gG}_{\mathrm{opt}}})-d_e(\mA_{\hat{\gG}_{\mathrm{opt}}}, \mA_{\hat{\gG}_{\mathrm{new}}\rightarrow\hat{\gG}_{\mathrm{opt}}}),0))\\&\geq\E(\max(\frac{d_{e,\hat{\gG}_{\mathrm{opt}}, \hat{\gG}_1}^*\cdot d_{e,\hat{\gG}_1, \hat{\gG}_2}^*}{n\cdot(n-1)-n_{se}}-\mathcal{B}(d_{e,\hat{\gG}_1, \hat{\gG}_2}^*, 0.5),0))\\&=\mathrm{LBEI}_{\mathrm{SEPX}}, 
\end{align*}}where $\mathcal{B}(d_{e,\hat{\gG}_1, \hat{\gG}_2}^*, 0.5)$ denotes the number of successful trials after sampling from a binomial distribution with $d_{e,\hat{\gG}_1, \hat{\gG}_2}^*$ trials and success probability of 0.5, and $\mathrm{LBEI}_{\mathrm{SEPX}}$ denotes the lower bound of expected improvement of the SEP crossover.
\end{theorem}

\begin{theorem}[Expected improvement of standard crossover]\label{thm:EI_STDX}
	Suppose $\mA_{\hat{\gG}_{\mathrm{new}}}=r(\mA_{\hat{\gG}'_1}, \mP_{\pi_\mathrm{rand}}\mA_{\hat{\gG}_2}\mP_{\pi_\mathrm{rand}}^{\top})$. Then we have
	{\footnotesize \begin{align*}
	&\E(\max(d_e(\mA_{\hat{\gG}_{\mathrm{opt}}}, \mA_{\hat{\gG}_1\rightarrow\hat{\gG}_{\mathrm{opt}}})-d_e(\mA_{\hat{\gG}_{\mathrm{opt}}}, \mA_{\hat{\gG}_{\mathrm{new}}\rightarrow\hat{\gG}_{\mathrm{opt}}}),0))\\&\geq\E(\max(d_{e,\hat{\gG}_{\mathrm{opt}}, \hat{\gG}_1}^*-\mathcal{B}(\frac{n_1^1\cdot n_2^0+n_1^0\cdot n_2^1}{n\cdot(n-1)},0.5)-\\&\frac{(d_{e,\hat{\gG}_{\mathrm{opt}}, \hat{\gG}_1}^*+n_1^1-n_{\mathrm{opt}}^1)\cdot n_2^1+(d_{e,\hat{\gG}_{\mathrm{opt}}, \hat{\gG}_1}^*+n_1^0-n_{\mathrm{opt}}^0)\cdot n_2^0}{2n\cdot(n-1)}\\&,0))=\mathrm{LBEI}_{\mathrm{STDX}}, 
	\end{align*}}where $n_{\mathrm{opt}}^1$, $n_1^1$ and $n_2^1$ denote the number of ones in $\mA_{\hat{\gG}_{\mathrm{opt}}}$, $\mA_{\hat{\gG}_1}$ and $\mA_{\hat{\gG}_2}$ (excluding diagonal entries), respectively, $n_{\mathrm{opt}}^0$, $n_1^0$ and $n_2^0$ denote the number of zeros in $\mA_{\hat{\gG}_{\mathrm{opt}}}$, $\mA_{\hat{\gG}_1}$ and $\mA_{\hat{\gG}_2}$ (excluding diagonal entries), respectively, and $\mathrm{LBEI}_{\mathrm{STDX}}$ denotes the lower bound of expected improvement of the standard crossover.
\end{theorem}

\begin{theorem}[Expected improvement of mutation]\label{thm:EI_MUTA}
	Suppose $\mA_{\hat{\gG}_{\mathrm{new}}}=m(\mA_{\hat{\gG}'_1})$. Then we have
	{\footnotesize \begin{align*}
	\displaystyle &\E(\max(d_e(\mA_{\hat{\gG}_{\mathrm{opt}}}, \mA_{\hat{\gG}_1\rightarrow\hat{\gG}_{\mathrm{opt}}})-d_e(\mA_{\hat{\gG}_{\mathrm{opt}}}, \mA_{\hat{\gG}_{\mathrm{new}}\rightarrow\hat{\gG}_{\mathrm{opt}}}),0))\\&\geq\E(\max( d_{e,\hat{\gG}_{\mathrm{opt}}, \hat{\gG}_1}^*-\mathcal{B}(n\cdot(n-1)-d_{e,\hat{\gG}_{\mathrm{opt}}, \hat{\gG}_1}^*,p_m)\\&-\mathcal{B}(d_{e,\hat{\gG}_{\mathrm{opt}}, \hat{\gG}_1}^*, 1-p_m),0))=\mathrm{LBEI}_{\mathrm{MUTA}}, 
	\end{align*}}where $p_m$ is the mutation rate usually chosen to be $p_m=\frac{1}{n\cdot(n-1)}$, and $\mathrm{LBEI}_{\mathrm{MUTA}}$ denotes the lower bound of expected improvement of mutation.
\end{theorem}

\subsection{Expected Improvement with RL}\label{subsec:RL_theory}
First, let us interpret the RL approach using concepts established in Section~\ref{subsec:cross_muta}. The setup follows the implementation of \citet{Ying19}, which provides good performance in NAS-bench-101 dataset.
\begin{definition}[AA-matrix-based interpretation of RL]\label{def:AA_rl}
	RL invokes an agent that generates architectures following a probability distribution $\mQ_\theta$ defined in AA-matrix space. For $\mA_\theta \sim \mQ_\theta$, each entry $A_{i,j}^{\theta}$ is sampled from a separate categorical distribution defined by $Q_{i,j}^{\theta}$. The $\theta=\{z_{i,j}^k|k\in 0,1, \cdots, k_{i,j}^{\mathrm{max}}, \mathrm{for} \ i,j\in 1,2,\cdots,n\}$ is the parameter set that contains the logits for defining the categorical distributions through softmax functions $p(A_{i,j}^{\theta}=k)=\frac{\mathrm{e}^{z_{i,j}^k}}{\Sigma_k \mathrm{e}^{z_{i,j}^k}}$, for $k\in 0,1, \cdots, k_{i,j}^{\mathrm{max}}, \mathrm{and} \ i,j\in 1,2,\cdots,n$. The learning process of $\theta$ follows the standard REINFORCE rule \citep{Williams92}. The resulting scaled policy gradient is calculated as $\E_{\mA_\theta \sim \mQ_\theta}(\Sigma_{i,j} \bigtriangledown_\theta \log p(A_{i,j}^{\theta})\cdot (R-b))$,
	where $R$ is the reward for the currently sampled architecture (usually the validation accuracy) and $b$ is a baseline to reduce the variance of gradient estimate.
\end{definition}
The expected improvement of a policy update can then be defined. It is based on Lemma~\ref{lem:lb_expected_GED} in Appendix~\ref{subsec:app_proof} that defines $\mQ^*_{\theta}$ as the optimal permutation of $\mQ_{\theta}$ and establishes an upper bound of expected GED to optimal.
\begin{definition}[Expected improvement of a policy update]\label{def:ei_policy} Suppose the RL policy parameters are updated from $\theta_t$ to $\theta_{t+1}$, where $t$ indicates the current time step. The expected improvement is defined as $\Sigma_{i,j}p(A^{\theta_t*}_{i,j}\neq A^{\gG_{\mathrm{opt}}}_{i,j}|\mA^*_{\theta_t} \sim \mQ^*_{\theta_t}) - \Sigma_{i,j}p(A^{\theta_{t+1}*}_{i,j}\neq A^{\gG_{\mathrm{opt}}}_{i,j}|\mA^*_{\theta_{t+1}} \sim \mQ^*_{\theta_{t+1}})$ for $i,j\in 1,2,\cdots,n$ and $i\neq j$. That is, it is the change in the upper bound of expected GED to optimal after policy update, considering only the edge/connection differences (similar to that of crossover and mutation).
\end{definition}

Next, expected improvement can be derived in two extreme cases:

\begin{definition}[Unbiased RL agent and oracle RL agent]\label{def:extreme_agent}
	Given a pre-defined value for $\Sigma_{i,j}p(A^{\theta*}_{i,j}\neq A^{\gG_{\mathrm{opt}}}_{i,j}|\mA^*_\theta \sim \mQ^*_\theta)$ ($i,j\in 1,2,\cdots,n$ and $i\neq j$), an unbiased agent is one that has the same $p(A^{\theta*}_{i,j}\neq A^{\gG_{\mathrm{opt}}}_{i,j}|\mA^*_\theta \sim \mQ^*_\theta)$ value for any $i,j\in 1,2,\cdots,n$ and $i\neq j$, and an oracle agent is one that has the maximum number of non-diagonal entries in $\mA^*_\theta$ satisfying $p(A^{\theta*}_{i,j}\neq A^{\gG_{\mathrm{opt}}}_{i,j}|\mA^*_\theta \sim \mQ^*_\theta)=0$, while all the remaining non-diagonal entries have the same and positive value for $p(A^{\theta*}_{i,j} = A^{\gG_{\mathrm{opt}}}_{i,j}|\mA^*_\theta \sim \mQ^*_\theta)$.
\end{definition}

In practical NAS experiments, the RL agent is usually initially unbiased, and converges towards the oracle agent during learning. Therefore, it is possible to interpolate between these two cases to span the entire RL search process. Next, expected improvement in RL is derived for the two cases:

\begin{theorem}[Expected improvement of unbiased agent and oracle agent]\label{thm:EI_extreme}
	Suppose $\Sigma_{i,j}p(A^{\theta*}_{i,j}\neq A^{\gG_{\mathrm{opt}}}_{i,j}|\mA^*_\theta \sim \mQ^*_\theta)=b^*_{e,\theta}$, and further suppose $R-b=\alpha\cdot(\Sigma_{i,j}p(A^{\theta*}_{i,j}\neq A^{\gG_{\mathrm{opt}}}_{i,j}|\mA^*_\theta \sim \mQ^*_\theta)-d^*_{e,\gG_{\mathrm{opt}},\gG_{\theta_t}})$ for $i,j\in 1,2,\cdots,n$ and $i\neq j$, where $\alpha$ is a positive scaling factor and $\gG_{\theta_t}$ is a graph sampled at time step $t$ to obtain the empirical approximation of the policy gradient. With all $z^k_{i,j}$ initialized to 0, the expected improvement after one policy update with learning rate $\eta$ is no less than
	{\footnotesize \begin{align*}\\[-18pt]
	&\mathrm{LBEI}_{\mathrm{RLU}}=b^*_{e,\theta}-(n_w\cdot \frac{1}{1+(\frac{1}{p_w}-1)\cdot \mathrm{e}^{-2\alpha \eta(b^*_{e,\theta}-n_w)(1-p_w)}}\\&+(n(n-1)-n_w)\cdot \frac{1}{1+(\frac{1}{p_w}-1)\cdot \mathrm{e}^{2\alpha \eta(b^*_{e,\theta}-n_w)\cdot p_w}})
	\end{align*}}\\[-1.5ex]for unbiased agent, where $p_w=\frac{b^*_{e,\theta}}{n(n-1)}, n_w=\mathcal{B}(n(n-1), \frac{b^*_{e,\theta}}{n(n-1)})$, and no less than
	{\footnotesize \begin{align*}\\[-20pt]
	&\mathrm{LBEI}_{\mathrm{RLO}}=b^*_{e,\theta}-(n_w\cdot \frac{1}{1+(\frac{1}{p_w}-1)\cdot \mathrm{e}^{-2\alpha \eta(b^*_{e,\theta}-n_w)(1-p_w)}}\\&+(\lfloor b^*_{e,\theta} \rfloor +1-n_w)\cdot \frac{1}{1+(\frac{1}{p_w}-1)\cdot \mathrm{e}^{2\alpha \eta(b^*_{e,\theta}-n_w)\cdot p_w}})
	\end{align*}}\\[-1.5ex]for oracle agent, where $p_w=\frac{b^*_{e,\theta}}{\lfloor b^*_{e,\theta}\rfloor +1}, n_w=\mathcal{B}(\lfloor b^*_{e,\theta}\rfloor+1, \frac{b^*_{e,\theta}}{\lfloor b^*_{e,\theta}\rfloor+1})$, and $\lfloor\cdot\rfloor$ is the floor function.
\end{theorem}

\subsection{Comparisons based on Theory}\label{subsec:theory_comp}
As Theorems~\ref{thm:EI_SEPX}--\ref{thm:EI_MUTA} and \ref{thm:EI_extreme} indicate, expected improvement with the different methods depends on several factors, making problem-agnostic comparisons in closed form infeasible. It is, however, possible to compare these theoretical constructs numerically in specific representative settings, such as the various NAS benchmark domains.

To this end, $\mathrm{LBEI}_{\mathrm{SEPX}}$, $\mathrm{LBEI}_{\mathrm{MUTA}}$, $\mathrm{LBEI}_{\mathrm{STDX}}$, $\mathrm{LBEI}_{\mathrm{RLU}}$ and $\mathrm{LBEI}_{\mathrm{RLO}}$  were compared in NAS-bench-101 benchmark \citep{Ying19}. A numerical comparison requires instantiating the methods with specific parameter values. The standard NAS-bench-101 setup was used for $n=7$, $n_{\mathrm{opt}}^1=9$, $n_1^1=9$, and $n_2^1=9$, and for $d_{e,\hat{\gG}_{\mathrm{opt}}, \hat{\gG}_1}^*$ and $d_{e,\hat{\gG}_1, \hat{\gG}_2}^*$ different combinations within a reasonable range were evaluated (the validity of these ranges will be verified in Section~\ref{subsec:applicability}). 
The expected improvement in each case was then estimated through a Monte Carlo simulation with $10^6$ trials. For RL, $b^*_{e,\theta} \equiv d_{e,\hat{\gG}_{\mathrm{opt}}, \hat{\gG}_1}^*$, and $\alpha \cdot \eta=0.1$ was used because this value provides the best tradeoff between unbiased and oracle agents (Figure~\ref{fig:rl_self}). 

Figure~\ref{fig:EI_101} shows the main results: $\mathrm{LBEI}_{\mathrm{SEPX}}$ is larger than $\mathrm{LBEI}_{\mathrm{MUTA}}$, $\mathrm{LBEI}_{\mathrm{RLU}}$, and $\mathrm{LBEI}_{\mathrm{RLO}}$ in almost all cases. In contrast, Figure~\ref{fig:EI_stdx_mut} and Figure~\ref{fig:stdx_rl} in Appendix~\ref{subsec:add_theo_comp} show that the standard crossover leads to worse $\mathrm{LBEI}$ compared to mutation and RL. This numerical analysis thus illustrates the theoretical advantage of SEP crossover compared to mutation, standard crossover, and RL. More comparisons, as well as another benchmark \citep[NAS-bench-NLP;][]{Klyuchnikov22}, are included in Appendix~\ref{subsec:add_theo_comp}, reinforcing these conclusions.
\begin{figure*}[t]
	\centering
	\includegraphics[width=0.32\linewidth]{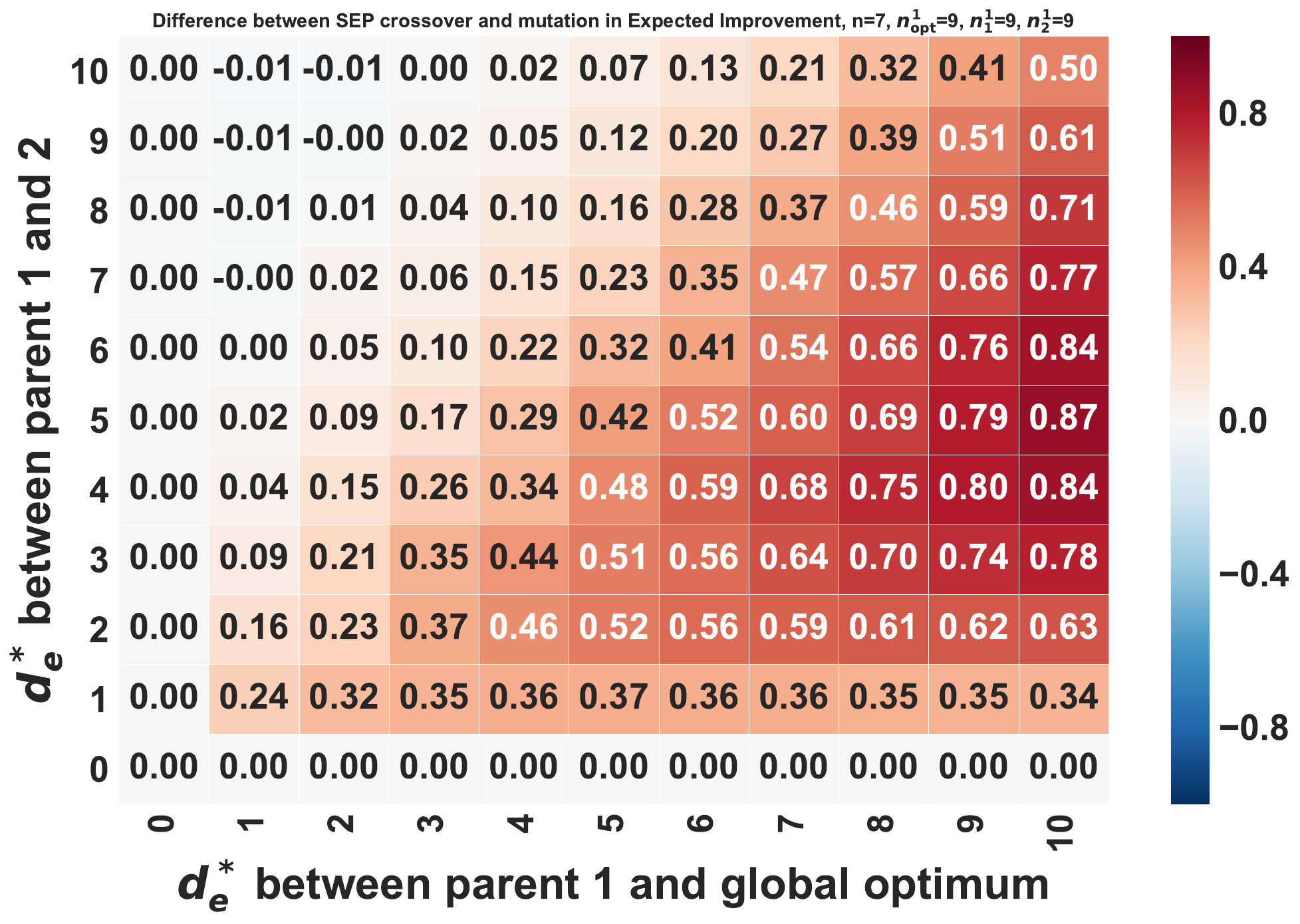}
    \includegraphics[width=0.32\linewidth]{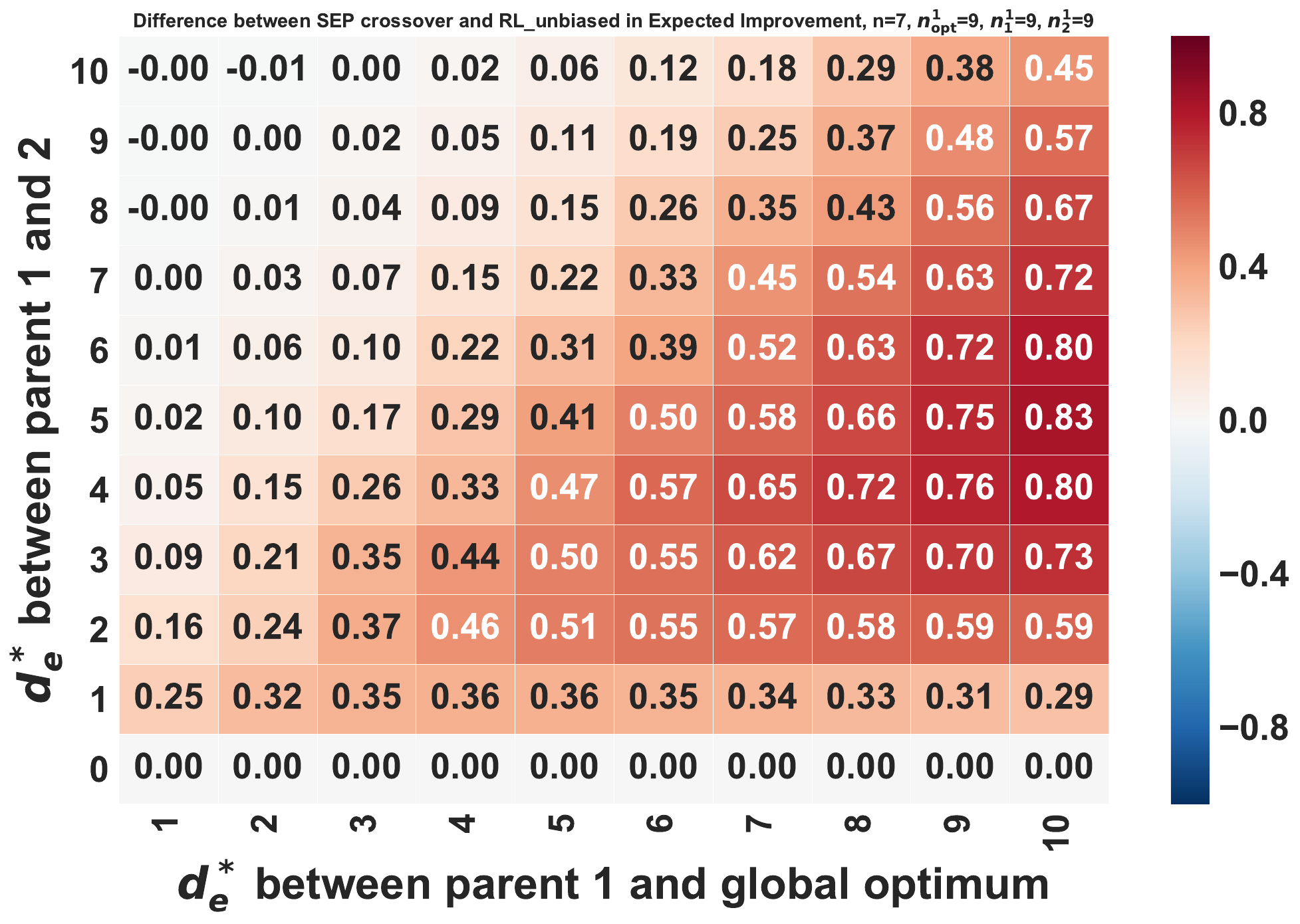}
	\includegraphics[width=0.32\linewidth]{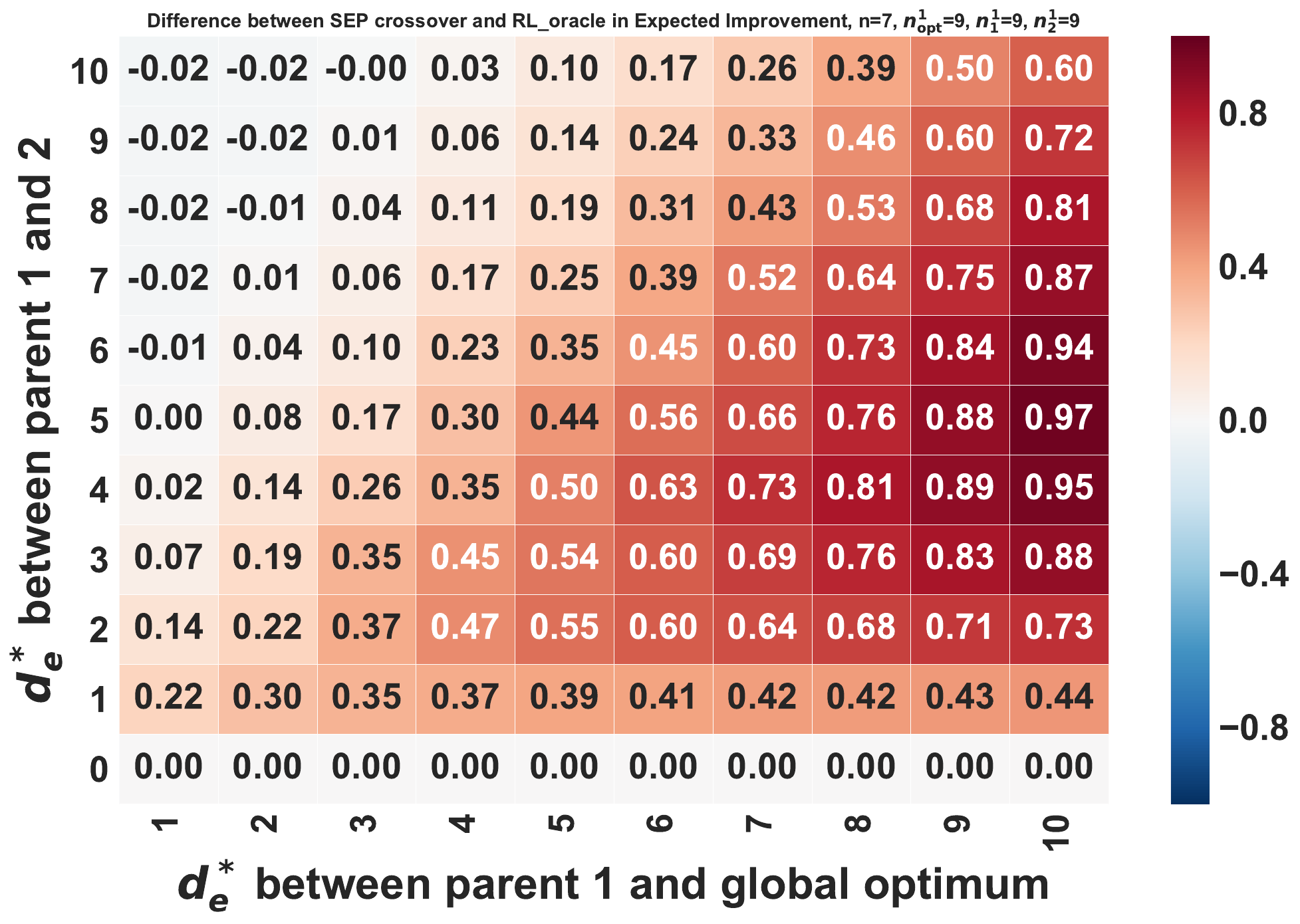}
\vspace*{-1.5ex}
	\caption{\textbf{Comparison of expected improvement between SEP crossover, mutation, and RL in NAS-bench-101.} (Left) Differences between $\mathrm{LBEI}_{\mathrm{SEPX}}$ and $\mathrm{LBEI}_{\mathrm{MUTA}}$ under different $d_{e,\hat{\gG}_1, \hat{\gG}_2}^*$ ($y$-axis) and $d_{e,\hat{\gG}_{\mathrm{opt}}, \hat{\gG}_1}^*$ ($x$-axis) combinations. (Middle) Differences between $\mathrm{LBEI}_{\mathrm{SEPX}}$ and $\mathrm{LBEI}_{\mathrm{RLU}}$. (Right) Differences between $\mathrm{LBEI}_{\mathrm{SEPX}}$ and $\mathrm{LBEI}_{\mathrm{RLO}}$. $\mathrm{LBEI}_{\mathrm{SEPX}}$ is larger (i.e.\ more red) than $\mathrm{LBEI}_{\mathrm{MUTA}}$, $\mathrm{LBEI}_{\mathrm{RLU}}$, and $\mathrm{LBEI}_{\mathrm{RLO}}$ almost everywhere. Thus, the SEP crossover has a theoretical advantage over mutation and RL.
		\label{fig:EI_101}
	}
\end{figure*}
\subsection{Effect of Errors during GED Calculation} \label{subsec:error_GED}
Finding the shortest edit path between two graphs requires calculating the GED between them, which is a NP-hard problem if an exact optimal solution is desired. Several fast approximation methods exist for GED calculation \citep{Riesen16,Serratosa15}. They can be run in polynomial time, at the cost of slightly reduced accuracy of the returned GED. To verify that SEP crossover is robust against such a loss of accuracy, a theoretical analysis was conducted. First, a corollary was derived to quantify the resulting expected improvement of SEP crossover with errors in the GED calculation. Second, a numerical analysis based on this corollary was run under three different levels of error.

Following Theorem~\ref{thm:EI_SEPX}, an error in calculating GED between two architectures $\hat{\gG}_1$ and $\hat{\gG}_2$ can be expressed as
\begin{equation*}
d_{e,\hat{\gG}_1, \hat{\gG}_2}^\epsilon=d_{e,\hat{\gG}_1, \hat{\gG}_2}^*\cdot(1+\epsilon),
\end{equation*}where $\epsilon>0$ is the error ratio and $d_{e,\hat{\gG}_1, \hat{\gG}_2}^\epsilon$ is the expectation of GED calculation result. Assuming the resulting GED is either $\lfloor d_{e,\hat{\gG}_1, \hat{\gG}_2}^\epsilon\rfloor$ or $\lfloor d_{e,\hat{\gG}_1, \hat{\gG}_2}^\epsilon\rfloor+1$ following a Bernoulli distribution, the following corollary can be obtained:\\[-1.5ex]
\begin{corollary}[Effect of GED errors on $\mathrm{LBEI}_{\mathrm{SEPX}}$]\label{coro:EI_SEPX_error} With error ratio $\epsilon$ in calculating $d_{e,\hat{\gG}_1, \hat{\gG}_2}^*$, $\mathrm{LBEI}_{\mathrm{SEPX}}$ becomes
	{\scriptsize \begin{align*}
	&\mathrm{LBEI}_{\mathrm{SEPX}}^\epsilon= \ (d_{e,\hat{\gG}_1, \hat{\gG}_2}^\epsilon-\lfloor d_{e,\hat{\gG}_1, \hat{\gG}_2}^\epsilon\rfloor)\\&\cdot\E(\max(\frac{d_{e,\hat{\gG}_{\mathrm{opt}}, \hat{\gG}_1}^*\cdot (\lfloor d_{e,\hat{\gG}_1, \hat{\gG}_2}^\epsilon\rfloor+1)}{n\cdot(n-1)-\lfloor n_{se}^\epsilon\rfloor}-\mathcal{B}(\lfloor d_{e,\hat{\gG}_1, \hat{\gG}_2}^\epsilon\rfloor+1, 0.5),0))\\&+(\lfloor d_{e,\hat{\gG}_1, \hat{\gG}_2}^\epsilon\rfloor+1-d_{e,\hat{\gG}_1, \hat{\gG}_2}^\epsilon)\\&\cdot\E(\max(\frac{d_{e,\hat{\gG}_{\mathrm{opt}}, \hat{\gG}_1}^*\cdot \lfloor d_{e,\hat{\gG}_1, \hat{\gG}_2}^\epsilon\rfloor}{n\cdot(n-1)-\lceil n_{se}^\epsilon\rceil}-\mathcal{B}(\lfloor d_{e,\hat{\gG}_1, \hat{\gG}_2}^\epsilon \rfloor, 0.5),0)),
	\end{align*}}where $n_{se}^\epsilon=\max(n\cdot(n-1)-d_{e,\hat{\gG}_{\mathrm{opt}}, \hat{\gG}_1}^*-d_{e,\hat{\gG}_1, \hat{\gG}_2}^\epsilon, 0).$
\end{corollary}

As in Section~\ref{subsec:theory_comp}, Monte Carlo simulations with $10^6$ trials each were performed to estimate the values of $\mathrm{LBEI}_{\mathrm{SEPX}}^\epsilon$ under different error ratios $\epsilon$. Figure~\ref{fig:EI_101_error} in Appendix~\ref{subsec:add_error_GED} compares $\mathrm{LBEI}_{\mathrm{SEPX}}^\epsilon$ with the $\mathrm{LBEI}$ values for other methods under error ratios $\epsilon=$ 0.1, 0.2, and 0.3.

The conclusion is that the SEP crossover has a theoretical advantage in expected improvement compared to mutation, standard crossover, and RL even with a very high error ratio of $30\%$ in the GED calculations. Thus, approximation methods can be used for GED if the computational cost of the SEP crossover needs to be reduced.
\section{Empirical Verification}\label{sec:exp}

This section first verifies that the parameter values used in the numerical analysis indeed apply to real-world problems. It then demonstrates that the SEP crossover is effective in real NAS problems under both noise-free and noisy environments. Experiment setup is provided in Appendix~\ref{subsec:Exp_Setup}.

\subsection{Applicability of the Theory}\label{subsec:applicability}

Figures~\ref{fig:EI_101} (and Figure~\ref{fig:EI_nlp} in Appendix~\ref{subsec:add_theo_comp}) demonstrate the theoretical advantage of SEP crossover numerically. However, it is important to verify that the parameter values used in the Monte Carlo simulation  indeed lie within the favorable regions in real NAS problems. In particular, the values used for $d_{e,\hat{\gG}_{\mathrm{opt}}, \hat{\gG}_1}^*$, $d_{e,\hat{\gG}_1, \hat{\gG}_2}^*$, $n_1^1$, and $n_2^1$ are critical to the expected improvement and need to be verified in standard benchmarks and with a standard NAS algorithm.

A NAS benchmark is said to be queryable if it directly returns the predictive performance of any architecture in the search space. While NAS-bench-101 has the most flexible graph search space among all queryable NAS benchmarks, NAS-bench-NLP (which is not queryable) has the largest search space among all existing NAS benchmarks \citep{Mehta22}. They were both thus used to evaluate the parameter ranges. In order to evaluate the SEP crossover with a standard NAS algorithm, it was incorporated into the state-of-the-art Regularized Evolution method \citep[RE;][]{Real19}. RE employs only a mutation operator; SEP crossover was integrated into it by alternating crossover with mutation. To measure the parameter ranges, RE was run on both benchmarks, and the relative frequency distributions of the above parameters recorded (see Appendix~\ref{subsec:add_applicability}).

The results indeed show that the parameters lie within the range of the numerical analysis in Section~\ref{subsec:theory_comp}. Moreover, they are within the subrange where the SEP crossover has a theoretical advantage (Figure~\ref{fig:EI_101}). The results thus verify that the theory applies to NAS in real-world problems.

\subsection{Performance in Noise-free Environments}\label{subsec:noise-free}
\begin{figure*}[t]
	\centering
	\includegraphics[width=0.23\linewidth]{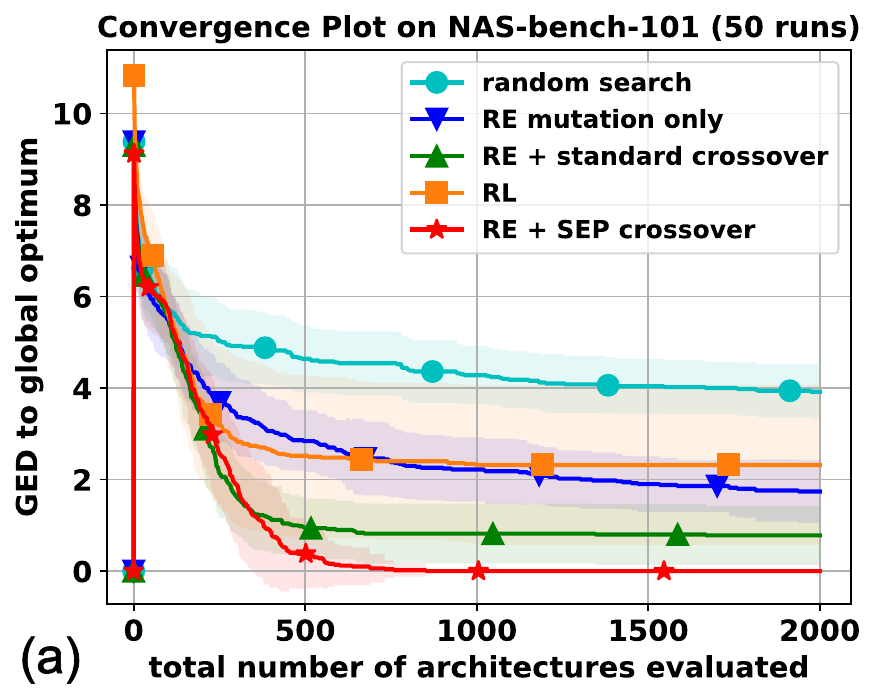}
	\includegraphics[width=0.235\linewidth]{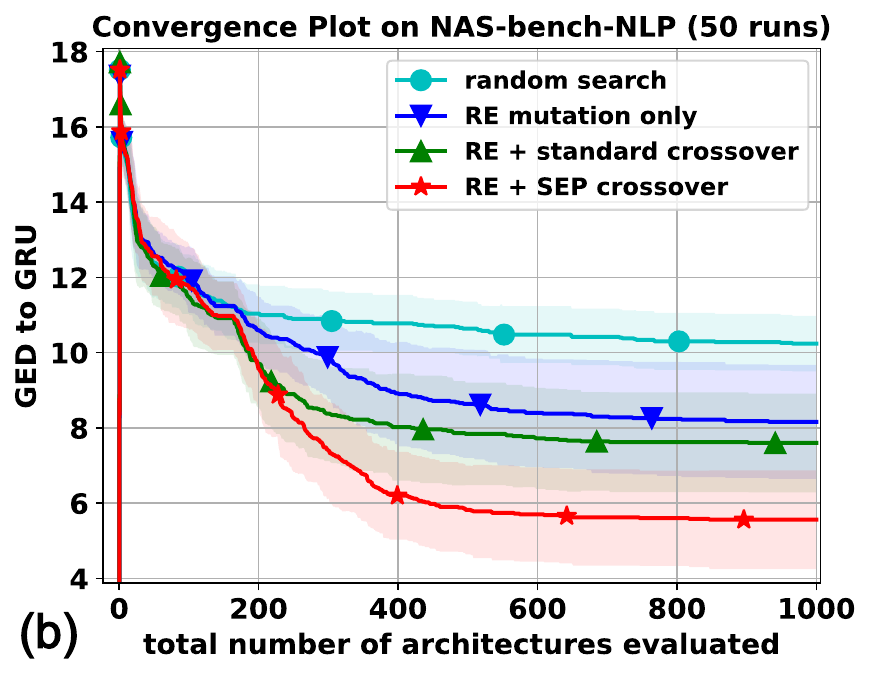}
	\includegraphics[width=0.2525\linewidth]{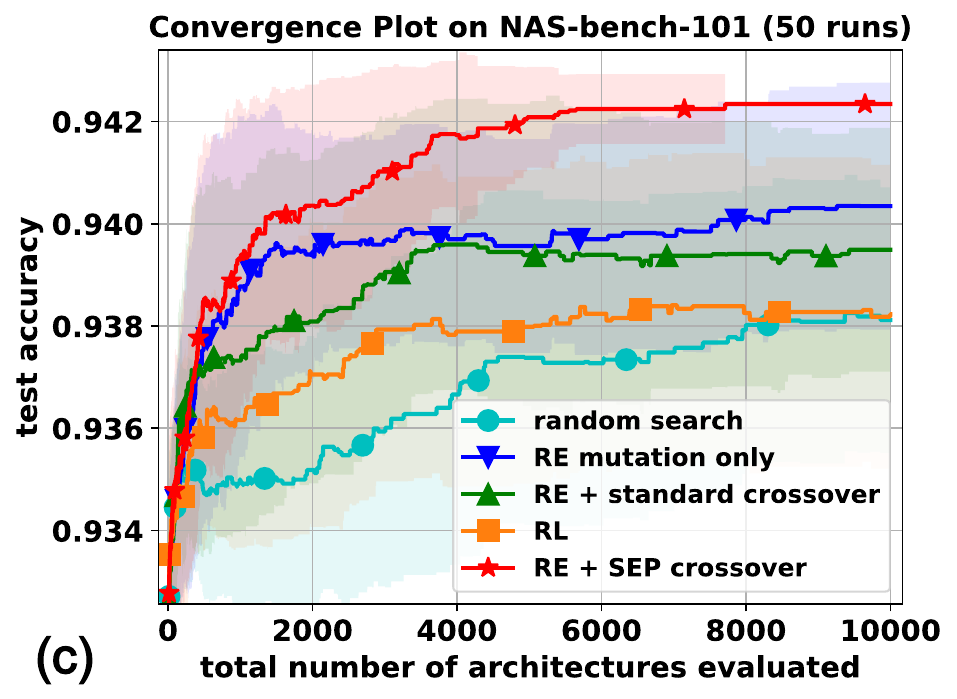}
	\includegraphics[width=0.26\linewidth]{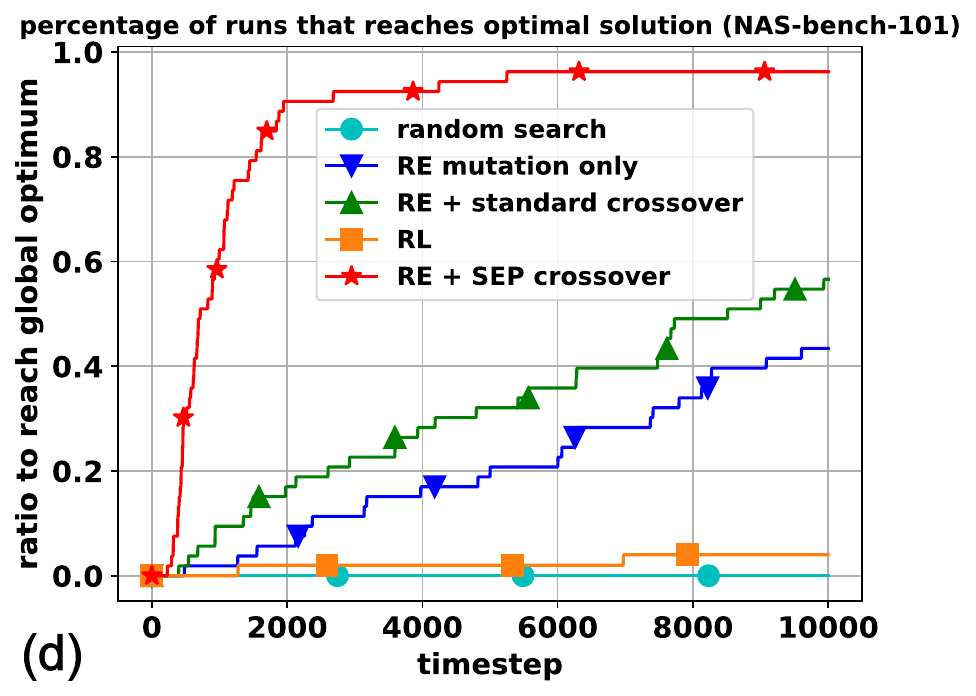}
\vspace*{-1.5ex}
	\caption{\textbf{Convergence in noise-free ((a) and (b)) and noisy environments ((c) and (d)).} (a) GED to global optimum in NAS-bench-101. (b) GED to GRU in NAS-bench-NLP. (c) Average testing accuracy in NAS-bench-101. (d) Percentage of runs that reach the global optimal architecture in NAS-bench-101. In all experiments, the SEP crossover performs consistently better than the other methods in both noise-free and noisy environments. The SEP crossover also reaches the global optimum significantly more efficiently than the other methods in NAS-bench-101. Together the experiments show that SEP consistently improves evolutionary NAS in practice. \label{fig:convergence}
	}
\end{figure*}
The evaluation step in NAS, i.e.\ the training and testing of an architecture, can be very noisy \citep{White21a}. To evaluate the search efficiency of the SEP crossover without the confounding effects of such noise, a noise-free evaluation function was first employed as the GED between the candidate architecture and the target architecture. The global optimum was selected as the target in NAS-bench-101, while the GRU \citep{Cho14} and LSTM \citep{Hochreiter97} models were used as targets in NAS-bench-NLP. Because the NAS-bench-NLP is not queryable, the global optimum is unknown. However, GRU and LSTM are two known top-performing models in this search space, and can therefore be used as a proxy for the global optimum. Since the RL method discussed in this work is only applicable to NAS-bench-101 space, it is not included in experiments on other benchmarks.

Plots (a) and (b) in Figure~\ref{fig:convergence} compare the performance of random search, the original RE with mutation only, a modified RE augmented with standard crossover, RL \citep{Ying19}, and a modified RE augmented with the SEP crossover. The SEP crossover performs significantly better than the other methods, demonstrating its value in practical NAS in noise-free environments. The experiments using LSTM as the target on NAS-bench-NLP is shown in Figure~\ref{fig:empirical} in Appendix~\ref{subsec:add_noise}, and a similar advantage of the SEP crossover can be observed. Note that the standard crossover also performs better than mutation; the population is not very diverse in these experiments and thus most parent graphs are already well aligned for crossover.

\subsection{Performance in Noisy Environments}\label{subsec:noisy}

In the third experiment, the robustness of the SEP crossover was evaluated by applying it to NAS problems with noisy evaluations. Noise arises from two sources: (1) the direct fitness/reward, e.g., the validation accuracy, used for search strategy is noisy; (2) The mapping between the final objective, e.g., the test accuracy, and direct fitness/reward is noisy. The validation accuracy in NAS-bench-101, which consists of random sampling of three real-world training trials was used as the direct fitness/reward. The average test accuracy in NAS-bench-101 was used as the final objectives.

Plots (c) and (d) of Figure~\ref{fig:convergence} again compare the performance of random search, RE with mutation-only, RE augmented with standard crossover, RL \citep{Ying19}, and RE augmented with the SEP crossover. The SEP crossover consistently outperforms other variants in this setup as well. Its performance is superior to the others in reaching the global optimal architecture (in terms of direct fitness/reward) in NAS-bench-101. Appendix~\ref{subsec:add_bo} and ~\ref{subsec:add_path} show more comparisons to two Bayesian optimization (BO) methods, namely BOHB \citep{Falkner18} and SMAC \citep{Hutter2011}, and crossover based on path encoding \citep{White21b}. The SEP crossover significantly outperforms all these approaches. The empirical results thus demonstrate that the proposed SEP crossover is robust and effective in realistic noisy environments as well. Note that the standard crossover performs worse than mutation on NAS-bench-101 in both test accuracy and validation accuracy (see Figure~\ref{fig:empirical} in Appendix~\ref{subsec:add_noise}). The population converges slower in these noisy environments, and the parent graphs are not as well aligned. Supplementary experiments using the surrogate predictions on NAS-bench-301 \citep{Zela22} are included in Figure~\ref{fig:empirical} of Appendix~\ref{subsec:add_noise}; the SEP crossover shows consistently better search ability.

\vspace*{-0.5ex}
\section{Discussion and Future Work}
To the best of our knowledge, this paper presents the first theoretical analysis on evolutionary NAS. In addition to the SEP crossover operator itself, the definitions, assumptions, lemmas and theorems can form a foundation for future theory development. The work thus deepens our understanding of the behaviors of EAs and provides useful insights toward developing better evolutionary NAS methods.

Although the advantage of the SEP crossover over mutation is demonstrated both theoretically and empirically, it does not mean that mutation should be avoided. To make any crossover operators work, diversity in the population is important. Mutation is critical in introducing new architectures into the population, thereby increasing and maintaining diversity. Search that takes advantage of both a proper crossover and mutation, such as RE augmented with SEP, is likely to be the most effective.

The theoretical results show that the standard crossover is not as good as mutation in terms of expected improvement. This conclusion is consistent with observations in prior literature: Applying crossover without resolving the permutation problem may simply make search less efficient. On the other hand, the advantage of the SEP crossover demonstrates that crossover can indeed help evolutionary search in NAS problems if the permutation problem can be avoided.

The computational cost of the SEP crossover depends on the calculation of GED between two parent graphs. Appendix~\ref{subsec:add_ged_time} reports the computation time for exact GED calculation in the NAS experiments. This cost is still negligible compared to the training and evaluation of an architecture, which may take several GPU hours or even days. GED calculation is therefore not the computation bottleneck for existing NAS problems. Moreover, analysis in Section~\ref{subsec:error_GED} suggests that the SEP crossover is robust to inaccurate GED calculation, and that if needed, approximate methods can be used to further reduce its computational costs.

Future directions include: (1) Developing a generative model that can output the offspring architecture for SEP crossover given two parents directly without a GED calculation; (2) applying the SEP crossover to more evolutionary NAS approaches and large-scale real-world NAS problems; and (3) applying the SEP crossover to other types of graph search/optimization problems, thus evaluating it as a general solution to optimization problems that involve graph search.

\section{Conclusion}

The SEP crossover is proposed as a solution the permutation problem in evolutionary NAS. Its advantage over standard crossover, mutation and RL was first shown theoretically, with a focus on the expected improvement of GED to global optimal. Empirical studies were then performed to verify the applicability of the theoretical results, and demonstrate the superior performance of the SEP crossover in both noise-free and noisy environments. The SEP crossover therefore allows taking full advantage of evolution in NAS, and potentially other similar design problems as well.


\bibliography{example_paper}
\bibliographystyle{icml2023}

\newpage
\appendix
\counterwithin{figure}{section}
\counterwithin{table}{section}
\onecolumn
\section{Appendix}\label{sec:app}
\subsection{Lemmas and Theorems with Proof Details}\label{subsec:app_proof}
\begin{lemma}[Invariance of SEP and standard crossover to parent permutation]\label{lem:invar_crossover}
	For any permutation $\pi'$, suppose graph $\hat{\gG}'_1$ has the corresponding AA-matrix $\mA_{\hat{\gG}'_1}=\mP_{\pi'}\mA_{\hat{\gG}_1}\mP_{\pi'}^{\top}$, $\mA_{\hat{\gG}'_{\mathrm{new}}}=r(\mA_{\hat{\gG}'_1}, \mP_{\pi_a}\mA_{\hat{\gG}_2}\mP_{\pi_a}^{\top})$, $\mA_{\hat{\gG}_{\mathrm{new}}}=r(\mA_{\hat{\gG}_1}, \mP_{\pi_b}\mA_{\hat{\gG}_2}\mP_{\pi_b}^{\top})$, and $\gG'_{\mathrm{new}}$,  $\gG_{\mathrm{new}}$ are the graphs after removing all null vertices from $\hat{\gG}'_{\mathrm{new}}$ and $\hat{\gG}_{\mathrm{new}}$, respectively. If $\pi_a=\pi^*_{\hat{\gG}'_1,\hat{\gG}_2}$, $\pi_b=\pi^*_{\hat{\gG}_1,\hat{\gG}_2}$, or $\pi_a=\pi_\mathrm{rand}$, $\pi_b=\pi_\mathrm{rand}$ ($\pi_a$ and $\pi_b$ are sampled independently), then $\mathrm{GED}(\gG_{\mathrm{opt}}, \gG'_{\mathrm{new}})=\mathrm{GED}(\gG_{\mathrm{opt}}, \gG_{\mathrm{new}})$, $d_v(\mA_{\hat{\gG}_{\mathrm{opt}}}, \mA_{\hat{\gG}'_{\mathrm{new}}\rightarrow\hat{\gG}_{\mathrm{opt}}})=d_v(\mA_{\hat{\gG}_{\mathrm{opt}}}, \mA_{\hat{\gG}_{\mathrm{new}}\rightarrow\hat{\gG}_{\mathrm{opt}}})$, and $d_e(\mA_{\hat{\gG}_{\mathrm{opt}}}, \mA_{\hat{\gG}'_{\mathrm{new}}\rightarrow\hat{\gG}_{\mathrm{opt}}})=d_e(\mA_{\hat{\gG}_{\mathrm{opt}}}, \mA_{\hat{\gG}_{\mathrm{new}}\rightarrow\hat{\gG}_{\mathrm{opt}}})$.
\end{lemma}
\begin{proof}
	Since a permutation of nodes (without changing their attributes and connections) simply generates an isomorphism of the original graph, $\hat{\gG}'_1$ is an isomorphism of $\hat{\gG}_1$. Calculations of the graph edit distance between two graphs are invariant to isomorphisms of either graph, so we have $\mA_{\hat{\gG}_2\rightarrow\hat{\gG}'_1}=\mP_{\pi'}\mA_{\hat{\gG}_2\rightarrow\hat{\gG}_1}\mP_{\pi'}^{\top}\Rightarrow \mP_{\pi^*_{\hat{\gG}'_1,\hat{\gG}_2}}\mA_{\hat{\gG}_2}\mP_{\pi^*_{\hat{\gG}'_1,\hat{\gG}_2}}^{\top}=\mP_{\pi'}\mP_{\pi^*_{\hat{\gG}_1,\hat{\gG}_2}}\mA_{\hat{\gG}_2}\mP_{\pi^*_{\hat{\gG}_1,\hat{\gG}_2}}^{\top}\mP_{\pi'}^{\top}\Rightarrow\mP_{\pi^*_{\hat{\gG}'_1,\hat{\gG}_2}}=\mP_{\pi'}\mP_{\pi^*_{\hat{\gG}_1,\hat{\gG}_2}}$. Because $r(\mA,\mB)$ is an element-wise operation that randomly chooses each entry either from $\mA$ or $\mB$, we have $r(\mP\mA\mP^{\top},\mP\mB\mP^{\top})=\mP r(\mA,\mB)\mP^{\top}$ for any $\mP$. Given $\pi_a=\pi^*_{\hat{\gG}'_1,\hat{\gG}_2}$, $\pi_b=\pi^*_{\hat{\gG}_1,\hat{\gG}_2}$, we have $\mA_{\hat{\gG}'_{\mathrm{new}}}=r(\mA_{\hat{\gG}'_1}, \mP_{\pi^*_{\hat{\gG}'_1,\hat{\gG}_2}}\mA_{\hat{\gG}_2}\mP_{\pi^*_{\hat{\gG}'_1,\hat{\gG}_2}}^{\top})=r(\mP_{\pi'}\mA_{\hat{\gG}_1}\mP_{\pi'}^{\top}, \mP_{\pi'}\mP_{\pi^*_{\hat{\gG}_1,\hat{\gG}_2}}\mA_{\hat{\gG}_2}\mP_{\pi^*_{\hat{\gG}_1,\hat{\gG}_2}}^{\top}\mP_{\pi'}^{\top})=\mP_{\pi'}r(\mA_{\hat{\gG}_1}, \mP_{\pi^*_{\hat{\gG}_1,\hat{\gG}_2}}\mA_{\hat{\gG}_2}\mP_{\pi^*_{\hat{\gG}_1,\hat{\gG}_2}}^{\top})\mP_{\pi'}^{\top}=\mP_{\pi'}\mA_{\hat{\gG}_{\mathrm{new}}}\mP_{\pi'}^{\top}$, which shows $\gG'_{\mathrm{new}}$ is an isomorphism of $\gG_{\mathrm{new}}$. Therefore, calculating $\mathrm{GED}(\gG_{\mathrm{opt}}, \gG'_{\mathrm{new}})$ is equivalent to calculating $\mathrm{GED}(\gG_{\mathrm{opt}}, \gG_{\mathrm{new}})$, and $d_v(\mA_{\hat{\gG}_{\mathrm{opt}}}, \mA_{\hat{\gG}'_{\mathrm{new}}\rightarrow\hat{\gG}_{\mathrm{opt}}})=d_v(\mA_{\hat{\gG}_{\mathrm{opt}}}, \mA_{\hat{\gG}_{\mathrm{new}}\rightarrow\hat{\gG}_{\mathrm{opt}}})$, $d_e(\mA_{\hat{\gG}_{\mathrm{opt}}}, \mA_{\hat{\gG}'_{\mathrm{new}}\rightarrow\hat{\gG}_{\mathrm{opt}}})=d_e(\mA_{\hat{\gG}_{\mathrm{opt}}}, \mA_{\hat{\gG}_{\mathrm{new}}\rightarrow\hat{\gG}_{\mathrm{opt}}})$.
	
	For the situation where $\pi_a=\pi_\mathrm{rand}$, $\pi_b=\pi_\mathrm{rand}$ ($\pi_a$ and $\pi_b$ are sampled independently), since any permutation of a randomly generated sequence is equivalent to directly generating a random sequence, we have $\mP_{\pi'}\mP_{\pi_\mathrm{rand}}=\mP_{\pi_\mathrm{rand}}$ for any $\pi'$. We can then derive the same conclusion as we did with $\mP_{\pi'}\mP_{\pi^*_{\hat{\gG}_1,\hat{\gG}_2}}=\mP_{\pi^*_{\hat{\gG}'_1,\hat{\gG}_2}}$.
\end{proof}
\begin{lemma}[Invariance of mutation to parent permutation]\label{lem:invar_mutation}
	For any permutation $\pi'$, suppose graph $\hat{\gG}'_1$ has the corresponding AA-matrix $\mA_{\hat{\gG}'_1}=\mP_{\pi'}\mA_{\hat{\gG}_1}\mP_{\pi'}^{\top}$, $\mA_{\hat{\gG}'_{\mathrm{new}}}=m(\mA_{\hat{\gG}'_1})$, $\mA_{\hat{\gG}_{\mathrm{new}}}=m(\mA_{\hat{\gG}_1})$, and $\gG'_{\mathrm{new}}$,  $\gG_{\mathrm{new}}$ are the graphs after removing all null vertices from $\hat{\gG}'_{\mathrm{new}}$ and $\hat{\gG}_{\mathrm{new}}$, respectively, then $\mathrm{GED}(\gG_{\mathrm{opt}}, \gG'_{\mathrm{new}})=\mathrm{GED}(\gG_{\mathrm{opt}}, \gG_{\mathrm{new}})$, $d_v(\mA_{\hat{\gG}_{\mathrm{opt}}}, \mA_{\hat{\gG}'_{\mathrm{new}}\rightarrow\hat{\gG}_{\mathrm{opt}}})=d_v(\mA_{\hat{\gG}_{\mathrm{opt}}}, \mA_{\hat{\gG}_{\mathrm{new}}\rightarrow\hat{\gG}_{\mathrm{opt}}})$, and $d_e(\mA_{\hat{\gG}_{\mathrm{opt}}}, \mA_{\hat{\gG}'_{\mathrm{new}}\rightarrow\hat{\gG}_{\mathrm{opt}}})=d_e(\mA_{\hat{\gG}_{\mathrm{opt}}}, \mA_{\hat{\gG}_{\mathrm{new}}\rightarrow\hat{\gG}_{\mathrm{opt}}})$.
\end{lemma}
\begin{proof}
	Since $m(\mA)$ is an element-wise operation, we have $m(\mP\mA\mP^{\top})=\mP m(\mA)\mP^{\top}$ for any $\mP$. We then have $\mA_{\hat{\gG}'_{\mathrm{new}}}=m(\mA_{\hat{\gG}'_1})=m(\mP_{\pi'}\mA_{\hat{\gG}_1}\mP_{\pi'}^{\top})=\mP_{\pi'}m(\mA_{\hat{\gG}_1})\mP_{\pi'}^{\top}=\mP_{\pi'}\mA_{\hat{\gG}_{\mathrm{new}}}\mP_{\pi'}^{\top}$, so $\gG'_{\mathrm{new}}$ is an isomorphism of $\gG_{\mathrm{new}}$. Therefore, we have $GED(\gG_{\mathrm{opt}}, \gG'_{\mathrm{new}})=GED(\gG_{\mathrm{opt}}, \gG_{\mathrm{new}})$, $d_v(\mA_{\hat{\gG}_{\mathrm{opt}}}, \mA_{\hat{\gG}'_{\mathrm{new}}\rightarrow\hat{\gG}_{\mathrm{opt}}})=d_v(\mA_{\hat{\gG}_{\mathrm{opt}}}, \mA_{\hat{\gG}_{\mathrm{new}}\rightarrow\hat{\gG}_{\mathrm{opt}}})$, and $d_e(\mA_{\hat{\gG}_{\mathrm{opt}}}, \mA_{\hat{\gG}'_{\mathrm{new}}\rightarrow\hat{\gG}_{\mathrm{opt}}})=d_e(\mA_{\hat{\gG}_{\mathrm{opt}}}, \mA_{\hat{\gG}_{\mathrm{new}}\rightarrow\hat{\gG}_{\mathrm{opt}}})$.
\end{proof}

\begin{lemma}[Lower bound for common parts in $\gG_{\mathrm{opt}}$, $\gG_1$ and $\gG_2$]\label{lem:common_parts}
	Suppose $\mathrm{GED}(\gG_{\mathrm{opt}}, \gG_1)=d_v(\mA_{\hat{\gG}_{\mathrm{opt}}}, \mA_{\hat{\gG}_1\rightarrow\hat{\gG}_{\mathrm{opt}}})+d_e(\mA_{\hat{\gG}_{\mathrm{opt}}}, \mA_{\hat{\gG}_1\rightarrow\hat{\gG}_{\mathrm{opt}}})=d_{v,\hat{\gG}_{\mathrm{opt}}, \hat{\gG}_1}^*+d_{e,\hat{\gG}_{\mathrm{opt}}, \hat{\gG}_1}^*=d_{\hat{\gG}_{\mathrm{opt}}, \hat{\gG}_1}^*$, $\mathrm{GED}(\gG_1, \gG_2)=d_v(\mA_{\hat{\gG}_1}, \mA_{\hat{\gG}_2\rightarrow\hat{\gG}_1})+d_e(\mA_{\hat{\gG}_1}, \mA_{\hat{\gG}_2\rightarrow\hat{\gG}_1})=d_{v,\hat{\gG}_1, \hat{\gG}_2}^*+d_{e,\hat{\gG}_1, \hat{\gG}_2}^*=d_{\hat{\gG}_1, \hat{\gG}_2}^*$, there exist $\pi_1$ and $\pi_2$ so that $s(\mA_{\hat{\gG}_{\mathrm{opt}}}, \mP_{\pi_1}\mA_{\hat{\gG}_1}\mP_{\pi_1}^{\top}, \mP_{\pi_2}\mA_{\hat{\gG}_2}\mP_{\pi_2}^{\top})>=\max(n^2-d_{\hat{\gG}_{\mathrm{opt}}, \hat{\gG}_1}^*-d_{\hat{\gG}_1, \hat{\gG}_2}^*, 0)$, $s_v(\mA_{\hat{\gG}_{\mathrm{opt}}}, \mP_{\pi_1}\mA_{\hat{\gG}_1}\mP_{\pi_1}^{\top}, \mP_{\pi_2}\mA_{\hat{\gG}_2}\mP_{\pi_2}^{\top})>=\max(n-d_{v,\hat{\gG}_{\mathrm{opt}}, \hat{\gG}_1}^*-d_{v,\hat{\gG}_1, \hat{\gG}_2}^*, 0)$, $s_e(\mA_{\hat{\gG}_{\mathrm{opt}}}, \mP_{\pi_1}\mA_{\hat{\gG}_1}\mP_{\pi_1}^{\top}, \mP_{\pi_2}\mA_{\hat{\gG}_2}\mP_{\pi_2}^{\top})>=\max(n\cdot(n-1)-d_{e,\hat{\gG}_{\mathrm{opt}}, \hat{\gG}_1}^*-d_{e,\hat{\gG}_1, \hat{\gG}_2}^*, 0)$, where $s(\mA, \mB, \mC)=\sum_i\sum_j\1_{A_{i,j}= B_{i,j}=C_{i,j}}$, $s_v(\mA, \mB, \mC)=\sum_i\1_{A_{i,i}= B_{i,i}=C_{i,i}}$, $s_e(\mA, \mB, \mC)=\sum_i\sum_{j\neq i}\1_{A_{i,j}= B_{i,j}=C_{i,j}}$.
\end{lemma}
\begin{proof}
	Let's choose $\pi_1=\pi^*_{\hat{\gG}_\mathrm{opt},\hat{\gG}_1}$ and $\pi_2=\pi^*_{\hat{\gG}'_1,\hat{\gG}_2}$, where $\hat{\gG}'_1$ has the corresponding AA-matrix $\mA_{\hat{\gG}'_1}=\mP_{\pi_1}\mA_{\hat{\gG}_1}\mP_{\pi_1}^{\top}$, then we will have $d(\mA_{\hat{\gG}_{\mathrm{opt}}}, \mP_{\pi_1}\mA_{\hat{\gG}_1}\mP_{\pi_1}^{\top})=d_{\hat{\gG}_{\mathrm{opt}}, \hat{\gG}_1}^*$ and $d(\mP_{\pi_1}\mA_{\hat{\gG}_1}\mP_{\pi_1}^{\top}, \mP_{\pi_2}\mA_{\hat{\gG}_2}\mP_{\pi_2}^{\top})=d_{\hat{\gG}'_1, \hat{\gG}_2}^*=d_{\hat{\gG}_1, \hat{\gG}_2}^*$. In the worst case that the $d_{\hat{\gG}_{\mathrm{opt}}, \hat{\gG}_1}^*$ entries and $d_{\hat{\gG}_1, \hat{\gG}_2}^*$ entries have the least overlaps in positions, the number of same entries in $\mA_{\hat{\gG}_{\mathrm{opt}}}$, $\mP_{\pi_1}\mA_{\hat{\gG}_1}\mP_{\pi_1}^{\top}$ and $\mP_{\pi_2}\mA_{\hat{\gG}_2}\mP_{\pi_2}^{\top}$ will be no less than $n^2-d_{\hat{\gG}_{\mathrm{opt}}, \hat{\gG}_1}^*-d_{\hat{\gG}_1, \hat{\gG}_2}^*$ (if it is not negative). As a result, we have $s(\mA_{\hat{\gG}_{\mathrm{opt}}}, \mP_{\pi_1}\mA_{\hat{\gG}_1}\mP_{\pi_1}^{\top}, \mP_{\pi_2}\mA_{\hat{\gG}_2}\mP_{\pi_2}^{\top})>=\max(n^2-d_{\hat{\gG}_{\mathrm{opt}}, \hat{\gG}_1}^*-d_{\hat{\gG}_1, \hat{\gG}_2}^*, 0)$. When we decompose $s(\mA_{\hat{\gG}_{\mathrm{opt}}}, \mP_{\pi_1}\mA_{\hat{\gG}_1}\mP_{\pi_1}^{\top}, \mP_{\pi_2}\mA_{\hat{\gG}_2}\mP_{\pi_2}^{\top})$ into $s_v(\mA_{\hat{\gG}_{\mathrm{opt}}}, \mP_{\pi_1}\mA_{\hat{\gG}_1}\mP_{\pi_1}^{\top}, \mP_{\pi_2}\mA_{\hat{\gG}_2}\mP_{\pi_2}^{\top})$ and $s_e(\mA_{\hat{\gG}_{\mathrm{opt}}}, \mP_{\pi_1}\mA_{\hat{\gG}_1}\mP_{\pi_1}^{\top}, \mP_{\pi_2}\mA_{\hat{\gG}_2}\mP_{\pi_2}^{\top})$, we can easily obtain $s_v(\mA_{\hat{\gG}_{\mathrm{opt}}}, \mP_{\pi_1}\mA_{\hat{\gG}_1}\mP_{\pi_1}^{\top}, \mP_{\pi_2}\mA_{\hat{\gG}_2}\mP_{\pi_2}^{\top})>=\max(n-d_{v,\hat{\gG}_{\mathrm{opt}}, \hat{\gG}_1}^*-d_{v,\hat{\gG}_1, \hat{\gG}_2}^*, 0)$ and $s_e(\mA_{\hat{\gG}_{\mathrm{opt}}}, \mP_{\pi_1}\mA_{\hat{\gG}_1}\mP_{\pi_1}^{\top}, \mP_{\pi_2}\mA_{\hat{\gG}_2}\mP_{\pi_2}^{\top})>=\max(n\cdot(n-1)-d_{e,\hat{\gG}_{\mathrm{opt}}, \hat{\gG}_1}^*-d_{e,\hat{\gG}_1, \hat{\gG}_2}^*, 0)$.
\end{proof}

\begin{lemma}[Upper bound of expected GED to optimal]\label{lem:lb_expected_GED} Given an RL agent as defined in Definition~\ref{def:AA_rl}, its expected GED to optimal is defined as $\E_{\mA_\theta \sim \mQ_\theta}(\mathrm{GED}(\gG_{\mathrm{opt}},\gG_\theta))$, where $\gG_\theta$ is the corresponding graph of $\mA_\theta$. Suppose $\gG_{\mathrm{opt}}$ is within the sample space of the RL agent, and $\mQ_\theta$ is permuted to be $\mQ^*_\theta=\mP_{\pi^*_{\gG_{\mathrm{opt}},\theta}}\mQ_\theta\mP_{\pi^*_{\gG_{\mathrm{opt}},\theta}}^{\top}$ such that for any permutation $\pi^\prime$, $\Sigma_{i,j}p(A^{\theta*}_{i,j}\neq A^{\gG_{\mathrm{opt}}}_{i,j}|\mA^*_\theta \sim \mQ^*_\theta)\leq \Sigma_{i,j}p(A^{\theta \prime}_{i,j}\neq A^{\gG_{\mathrm{opt}}}_{i,j}|\mA^\prime_\theta \sim \mQ^\prime_\theta)$, where $\mQ^\prime_\theta=\mP_{\pi^\prime}\mQ_\theta\mP_{\pi^\prime}^{\top}$, we have $\E_{\mA_\theta \sim \mQ_\theta}(\mathrm{GED}(\gG_{\mathrm{opt}},\gG_\theta))\leq\Sigma_{i,j}p(A^{\theta*}_{i,j}\neq A^{\gG_{\mathrm{opt}}}_{i,j}|\mA^*_\theta \sim \mQ^*_\theta)$, for $i,j\in 1,2,\cdots,n$.
\end{lemma}
\begin{proof}
	$\mQ^*_\theta$ is one of the permutations of $\mQ_\theta$ that minimizes the expected number of different entries between $\mA^*_\theta \sim \mQ^*_\theta$ and $\mA_{\gG_\mathrm{opt}}$, i.e., $\Sigma_{i,j}p(A^{\theta*}_{i,j}\neq A^{\gG_{\mathrm{opt}}}_{i,j}|\mA^*_\theta \sim \mQ^*_\theta)=\E_{\mA^*_\theta \sim \mQ^*_\theta}d(\mA^*_\theta, \mA_{\gG_{\mathrm{opt}}})$. Since for every sampled $\mA^*_\theta$, we have $\mathrm{GED}(\gG_{\mathrm{opt}},\gG^*_\theta)\leq d(\mA^*_\theta, \mA_{\gG_{\mathrm{opt}}})$, which leads to $\E_{\mA^*_\theta \sim \mQ^*_\theta}(\mathrm{GED}(\gG_{\mathrm{opt}},\gG^*_\theta))\leq \E_{\mA^*_\theta \sim \mQ^*_\theta}d(\mA^*_\theta, \mA_{\gG_{\mathrm{opt}}})$. Because $\mQ^*_\theta$ is a permutation of $\mQ_\theta$, we have $\E_{\mA_\theta \sim \mQ_\theta}(\mathrm{GED}(\gG_{\mathrm{opt}},\gG_\theta))=\E_{\mA^*_\theta \sim \mQ^*_\theta}(\mathrm{GED}(\gG_{\mathrm{opt}},\gG^*_\theta))\leq \E_{\mA^*_\theta \sim \mQ^*_\theta}d(\mA^*_\theta, \mA_{\gG_{\mathrm{opt}}})=\Sigma_{i,j}p(A^{\theta*}_{i,j}\neq A^{\gG_{\mathrm{opt}}}_{i,j}|\mA^*_\theta \sim \mQ^*_\theta)$, for $i,j\in 1,2,\cdots,n$.
\end{proof}

\begin{repeatthm}{\ref{thm:EI_SEPX}}[Expected improvement of SEP crossover] Following Assumption~\ref{asm:diff_entry}, let $n_{se}=\max(n\cdot(n-1)-d_{e,\hat{\gG}_{\mathrm{opt}}, \hat{\gG}_1}^*-d_{e,\hat{\gG}_1, \hat{\gG}_2}^*, 0)$. and suppose $\mA_{\hat{\gG}_{\mathrm{new}}}=r(\mA_{\hat{\gG}'_1}, \mP_{\pi^*_{\hat{\gG}'_1,\hat{\gG}_2}}\mA_{\hat{\gG}_2}\mP_{\pi^*_{\hat{\gG}'_1,\hat{\gG}_2}}^{\top})$. Then we have 
\begin{align*}
\displaystyle &\E(\max(d_e(\mA_{\hat{\gG}_{\mathrm{opt}}}, \mA_{\hat{\gG}_1\rightarrow\hat{\gG}_{\mathrm{opt}}})-d_e(\mA_{\hat{\gG}_{\mathrm{opt}}}, \mA_{\hat{\gG}_{\mathrm{new}}\rightarrow\hat{\gG}_{\mathrm{opt}}}),0))\\&\geq\E(\max(\frac{d_{e,\hat{\gG}_{\mathrm{opt}}, \hat{\gG}_1}^*\cdot d_{e,\hat{\gG}_1, \hat{\gG}_2}^*}{n\cdot(n-1)-n_{se}}-\mathcal{B}(d_{e,\hat{\gG}_1, \hat{\gG}_2}^*, 0.5),0))=\mathrm{LBEI}_{\mathrm{SEPX}}, 
\end{align*}
where $\mathcal{B}(d_{e,\hat{\gG}_1, \hat{\gG}_2}^*, 0.5)$ denotes a binomial distribution with $d_{e,\hat{\gG}_1, \hat{\gG}_2}^*$ trials and success probability of 0.5, and $\mathrm{LBEI}_{\mathrm{SEPX}}$ denotes the lower bound of expected improvement of the SEP crossover.
\end{repeatthm}
\begin{proof}
	Following Assumption~\ref{asm:diff_entry}, since $n_{se}$ elements are shared by $\mA_{\hat{\gG}'_1}$ and $\mP_{\pi^*_{\hat{\gG}'_1,\hat{\gG}_2}}\mA_{\hat{\gG}_2}\mP_{\pi^*_{\hat{\gG}'_1,\hat{\gG}_2}}^{\top}$, the $d_{e,\hat{\gG}_1, \hat{\gG}_2}^*$ different elements among them are uniformly distributed within the remaining $n\cdot(n-1)-n_{se}$ entries. As a result, the chance for any one of these $n\cdot(n-1)-n_{se}$ entries to have the same values in both parents equals $1-\frac{d_{e,\hat{\gG}_1, \hat{\gG}_2}^*}{n\cdot(n-1)-n_{se}}$, then the number of entries in $\mA_{\hat{\gG}'_1}$ that is originally different from $\mA_{\hat{\gG}_{\mathrm{opt}}}$ and stay intact after crossover is $(1-\frac{d_{e,\hat{\gG}_1, \hat{\gG}_2}^*}{n\cdot(n-1)-n_{se}})\cdot d_{e,\hat{\gG}_{\mathrm{opt}}, \hat{\gG}_1}^*$. Since all the non-diagonal elements in $\mA_{\hat{\gG}'_1}$ and $\mP_{\pi^*_{\hat{\gG}'_1,\hat{\gG}_2}}\mA_{\hat{\gG}_2}\mP_{\pi^*_{\hat{\gG}'_1,\hat{\gG}_2}}^{\top}$ are either 0 or 1 (indicating whether there is an edge between two nodes), the number of remaining entries that one of the parents is correct while the other is incorrect equals $d_{e,\hat{\gG}_1, \hat{\gG}_2}^*$. Therefore, $d_e(\mA_{\hat{\gG}_{\mathrm{opt}}}, \mA_{\hat{\gG}_{\mathrm{new}}})=(1-\frac{d_{e,\hat{\gG}_1, \hat{\gG}_2}^*}{n\cdot(n-1)-n_{se}})\cdot d_{e,\hat{\gG}_{\mathrm{opt}}, \hat{\gG}_1}^*+\mathcal{B}(d_{e,\hat{\gG}_1, \hat{\gG}_2}^*, 0.5)$.
	Considering the fact that $d_e(\mA_{\hat{\gG}_{\mathrm{opt}}}, \mA_{\hat{\gG}_{\mathrm{new}}\rightarrow\hat{\gG}_{\mathrm{opt}}})\leq d_e(\mA_{\hat{\gG}_{\mathrm{opt}}}, \mA_{\hat{\gG}_{\mathrm{new}}})$, we have
	\begin{align*}
	&\E(\max(d_e(\mA_{\hat{\gG}_{\mathrm{opt}}}, \mA_{\hat{\gG}_1\rightarrow\hat{\gG}_{\mathrm{opt}}})-d_e(\mA_{\hat{\gG}_{\mathrm{opt}}}, \mA_{\hat{\gG}_{\mathrm{new}}\rightarrow\hat{\gG}_{\mathrm{opt}}}),0))\geq\E(\max(d_e(\mA_{\hat{\gG}_{\mathrm{opt}}}, \mA_{\hat{\gG}_1\rightarrow\hat{\gG}_{\mathrm{opt}}})-d_e(\mA_{\hat{\gG}_{\mathrm{opt}}}, \mA_{\hat{\gG}_{\mathrm{new}}}),0))\\&=\E(\max(d_{e,\hat{\gG}_{\mathrm{opt}}, \hat{\gG}_1}^*-((1-\frac{d_{e,\hat{\gG}_1, \hat{\gG}_2}^*}{n\cdot(n-1)-n_{se}})\cdot d_{e,\hat{\gG}_{\mathrm{opt}}, \hat{\gG}_1}^*+\mathcal{B}(d_{e,\hat{\gG}_1, \hat{\gG}_2}^*, 0.5))\\&=\E(\max(\frac{d_{e,\hat{\gG}_{\mathrm{opt}}, \hat{\gG}_1}^*\cdot d_{e,\hat{\gG}_1, \hat{\gG}_2}^*}{n\cdot(n-1)-n_{se}}-\mathcal{B}(d_{e,\hat{\gG}_1, \hat{\gG}_2}^*, 0.5),0)).
	\end{align*}
\end{proof}

\begin{repeatthm}{\ref{thm:EI_STDX}}[Expected improvement of standard crossover]
	Suppose $\mA_{\hat{\gG}_{\mathrm{new}}}=r(\mA_{\hat{\gG}'_1}, \mP_{\pi_\mathrm{rand}}\mA_{\hat{\gG}_2}\mP_{\pi_\mathrm{rand}}^{\top})$. Then we have
	\begin{align*}
	\displaystyle &\E(\max(d_e(\mA_{\hat{\gG}_{\mathrm{opt}}}, \mA_{\hat{\gG}_1\rightarrow\hat{\gG}_{\mathrm{opt}}})-d_e(\mA_{\hat{\gG}_{\mathrm{opt}}}, \mA_{\hat{\gG}_{\mathrm{new}}\rightarrow\hat{\gG}_{\mathrm{opt}}}),0))\\&\geq\E(\max(d_{e,\hat{\gG}_{\mathrm{opt}}, \hat{\gG}_1}^*-\frac{(d_{e,\hat{\gG}_{\mathrm{opt}}, \hat{\gG}_1}^*+n_1^1-n_{\mathrm{opt}}^1)\cdot n_2^1+(d_{e,\hat{\gG}_{\mathrm{opt}}, \hat{\gG}_1}^*+n_1^0-n_{\mathrm{opt}}^0)\cdot n_2^0}{2n\cdot(n-1)}-\mathcal{B}(\frac{n_1^1\cdot n_2^0+n_1^0\cdot n_2^1}{n\cdot(n-1)},0.5),0))\\&=\mathrm{LBEI}_{\mathrm{STDX}}, 
	\end{align*}
	where $n_{\mathrm{opt}}^1$, $n_1^1$ and $n_2^1$ denote the number of ones in $\mA_{\hat{\gG}_{\mathrm{opt}}}$, $\mA_{\hat{\gG}_1}$ and $\mA_{\hat{\gG}_2}$ (excluding diagonal entries), respectively, $n_{\mathrm{opt}}^0$, $n_1^0$ and $n_2^0$ denote the number of zeros in $\mA_{\hat{\gG}_{\mathrm{opt}}}$, $\mA_{\hat{\gG}_1}$ and $\mA_{\hat{\gG}_2}$ (excluding diagonal entries), respectively, and $\mathrm{LBEI}_{\mathrm{STDX}}$ denotes the lower bound of expected improvement of the standard crossover.
\end{repeatthm}
\begin{proof}
	The resulting corresponding graph of $\mP_{\pi_\mathrm{rand}}\mA_{\hat{\gG}_2}\mP_{\pi_\mathrm{rand}}^{\top}$ is equivalent to an isomorphism that randomly shuffles the order of vertices of $\hat{\gG}_2$, therefore any non-diagonal entries in $\mA_{\hat{\gG}_2}$, which represents the connection status between two vertices, has the same chance to be moved to any non-diagonal positions in $\mP_{\pi_\mathrm{rand}}\mA_{\hat{\gG}_2}\mP_{\pi_\mathrm{rand}}^{\top}$ after the vertices shuffling. The number of different non-diagonal entries between $\mA_{\hat{\gG}'_1}$ and $ \mP_{\pi_\mathrm{rand}}\mA_{\hat{\gG}_2}\mP_{\pi_\mathrm{rand}}^{\top}$ then equals to $n_1^1\cdot \frac{n_2^0}{n\cdot(n-1)}+n_1^0\cdot \frac{n_2^1}{n\cdot(n-1)}=\frac{n_1^1\cdot n_2^0+n_1^0\cdot n_2^1}{n\cdot(n-1)}$. The number of non-diagonal entries that are same in $\mA_{\hat{\gG}'_1}$ and $ \mP_{\pi_\mathrm{rand}}\mA_{\hat{\gG}_2}\mP_{\pi_\mathrm{rand}}^{\top}$ but are different from $\mA_{\hat{\gG}_{\mathrm{opt}}}$ equals $n_\mathrm{w}^1\cdot \frac{n_2^1}{n\cdot(n-1)}+n_\mathrm{w}^0\cdot \frac{n_2^0}{n\cdot(n-1)}$, where $n_\mathrm{w}^1$ and $n_\mathrm{w}^0$ denotes the number of 1s and 0s in the non-diagonal entries where $\mA_{\hat{\gG}'_1}$ and $\mA_{\hat{\gG}_{\mathrm{opt}}}$ are different (we treat these entries as "wrong" entries, so we use the subscript "w"), respectively. To calculate $n_\mathrm{w}^1$, we need to consider two cases: (1) if $n_1^1\geq n_{\mathrm{opt}}^1$, then $n_\mathrm{w}^1$ consists of two parts, namely $n_1^1-n_{\mathrm{opt}}^1$, which represents the number of extra 1s in $\mA_{\hat{\gG}'_1}$ that $\mA_{\hat{\gG}_{\mathrm{opt}}}$ can never match, and $\frac{d_{e,\hat{\gG}_{\mathrm{opt}}, \hat{\gG}_1}^*-(n_1^1-n_{\mathrm{opt}}^1)}{2}$, which is derived from the fact that in the remaining entries where $\mA_{\hat{\gG}'_1}$ and $\mA_{\hat{\gG}_{\mathrm{opt}}}$ have the same number of 1s, one misplace (compared to $\mA_{\hat{\gG}_{\mathrm{opt}}}$) of 1 in $\mA_{\hat{\gG}'_1}$ also leads to one misplace of 0 in $\mA_{\hat{\gG}'_1}$ (otherwise the number of 1s will be unequal in $\mA_{\hat{\gG}'_1}$ and $\mA_{\hat{\gG}_{\mathrm{opt}}}$), so exactly half of these $d_{e,\hat{\gG}_{\mathrm{opt}}, \hat{\gG}_1}^*-(n_1^1-n_{\mathrm{opt}}^1)$ mismatched entries will be 1 in $\mA_{\hat{\gG}'_1}$. After summing these two parts up, we obtain $n_1^1-n_{\mathrm{opt}}^1+\frac{d_{e,\hat{\gG}_{\mathrm{opt}}, \hat{\gG}_1}^*-(n_1^1-n_{\mathrm{opt}}^1)}{2}=\frac{d_{e,\hat{\gG}_{\mathrm{opt}}, \hat{\gG}_1}^*-(n_{\mathrm{opt}}^1-n_1^1)}{2}$. (2) if $n_1^1< n_{\mathrm{opt}}^1$, we only need to consider the entries that excluding those extra 1s in $\mA_{\hat{\gG}_{\mathrm{opt}}}$ that cannot be matched by $\mA_{\hat{\gG}'_1}$, that is, half of the remaining $d_{e,\hat{\gG}_{\mathrm{opt}}, \hat{\gG}_1}^*-(n_{\mathrm{opt}}^1-n_1^1)$ mismatched entries. We then have $n_\mathrm{w}^1=\frac{d_{e,\hat{\gG}_{\mathrm{opt}}, \hat{\gG}_1}^*-(n_{\mathrm{opt}}^1-n_1^1)}{2}$, which also equals to the result of the first case. Similarly, we can get $n_\mathrm{w}^0=\frac{d_{e,\hat{\gG}_{\mathrm{opt}}, \hat{\gG}_1}^*-(n_{\mathrm{opt}}^0-n_1^0)}{2}$.
	
	Given the above intermediate results, we can obtain $d_e(\mA_{\hat{\gG}_{\mathrm{opt}}}, \mA_{\hat{\gG}_{\mathrm{new}}})=n_\mathrm{w}^1\cdot \frac{n_2^1}{n\cdot(n-1)}+n_\mathrm{w}^0\cdot \frac{n_2^0}{n\cdot(n-1)}+\mathcal{B}(\frac{n_1^1\cdot n_2^0+n_1^0\cdot n_2^1}{n\cdot(n-1)},0.5)=\frac{(d_{e,\hat{\gG}_{\mathrm{opt}}, \hat{\gG}_1}^*+n_1^1-n_{\mathrm{opt}}^1)\cdot n_2^1+(d_{e,\hat{\gG}_{\mathrm{opt}}, \hat{\gG}_1}^*+n_1^0-n_{\mathrm{opt}}^0)\cdot n_2^0}{2n\cdot(n-1)}+\mathcal{B}(\frac{n_1^1\cdot n_2^0+n_1^0\cdot n_2^1}{n\cdot(n-1)},0.5)$. Since $d_e(\mA_{\hat{\gG}_{\mathrm{opt}}}, \mA_{\hat{\gG}_{\mathrm{new}}\rightarrow\hat{\gG}_{\mathrm{opt}}})\leq d_e(\mA_{\hat{\gG}_{\mathrm{opt}}}, \mA_{\hat{\gG}_{\mathrm{new}}})$, we have
	\begin{align*}
	&\E(\max(d_e(\mA_{\hat{\gG}_{\mathrm{opt}}}, \mA_{\hat{\gG}_1\rightarrow\hat{\gG}_{\mathrm{opt}}})-d_e(\mA_{\hat{\gG}_{\mathrm{opt}}}, \mA_{\hat{\gG}_{\mathrm{new}}\rightarrow\hat{\gG}_{\mathrm{opt}}}),0))\geq\E(\max(d_e(\mA_{\hat{\gG}_{\mathrm{opt}}}, \mA_{\hat{\gG}_1\rightarrow\hat{\gG}_{\mathrm{opt}}})-d_e(\mA_{\hat{\gG}_{\mathrm{opt}}}, \mA_{\hat{\gG}_{\mathrm{new}}}),0))\\&=\E(\max(d_{e,\hat{\gG}_{\mathrm{opt}}, \hat{\gG}_1}^*-\frac{(d_{e,\hat{\gG}_{\mathrm{opt}}, \hat{\gG}_1}^*+n_1^1-n_{\mathrm{opt}}^1)\cdot n_2^1+(d_{e,\hat{\gG}_{\mathrm{opt}}, \hat{\gG}_1}^*+n_1^0-n_{\mathrm{opt}}^0)\cdot n_2^0}{2n\cdot(n-1)}-\mathcal{B}(\frac{n_1^1\cdot n_2^0+n_1^0\cdot n_2^1}{n\cdot(n-1)},0.5),0)).
	\end{align*}
\end{proof}

\begin{repeatthm}{\ref{thm:EI_MUTA}}[Expected improvement of mutation]
	Suppose $\mA_{\hat{\gG}_{\mathrm{new}}}=m(\mA_{\hat{\gG}'_1})$. Then we have
	\begin{align*}
	\displaystyle &\E(\max(d_e(\mA_{\hat{\gG}_{\mathrm{opt}}}, \mA_{\hat{\gG}_1\rightarrow\hat{\gG}_{\mathrm{opt}}})-d_e(\mA_{\hat{\gG}_{\mathrm{opt}}}, \mA_{\hat{\gG}_{\mathrm{new}}\rightarrow\hat{\gG}_{\mathrm{opt}}}),0))\\&\geq\E(\max( d_{e,\hat{\gG}_{\mathrm{opt}}, \hat{\gG}_1}^*-\mathcal{B}(n\cdot(n-1)-d_{e,\hat{\gG}_{\mathrm{opt}}, \hat{\gG}_1}^*,p_m)-\mathcal{B}(d_{e,\hat{\gG}_{\mathrm{opt}}, \hat{\gG}_1}^*, 1-p_m),0))=\mathrm{LBEI}_{\mathrm{MUTA}}, 
	\end{align*}
	where $p_m$ is the mutation rate usually chosen to be $p_m=\frac{1}{n\cdot(n-1)}$, and $\mathrm{LBEI}_{\mathrm{MUTA}}$ denotes the lower bound of expected improvement of mutation.
\end{repeatthm}
\begin{proof}
	Since $\mA_{\hat{\gG}'_1}=\mP_{\pi^*_{\hat{\gG}_{\mathrm{opt}},\hat{\gG}_1}}\mA_{\hat{\gG}_1}\mP_{\pi^*_{\hat{\gG}_{\mathrm{opt}},\hat{\gG}_1}}^{\top}$, there are $d_{e,\hat{\gG}_{\mathrm{opt}}, \hat{\gG}_1}^*$ non-diagonal elements in $\mA_{\hat{\gG}'_1}$ that are different from $\mA_{\hat{\gG}_{\mathrm{opt}}}$. Because all the non-diagonal elements in $\mA_{\hat{\gG}'_1}$ are either 0 or 1, and $m(\mA_{\hat{\gG}'_1})$ has $p_m$ probability to flip each non-diagonal element of $\mA_{\hat{\gG}'_1}$, we have $d_e(\mA_{\hat{\gG}_{\mathrm{opt}}}, \mA_{\hat{\gG}_{\mathrm{new}}})=\mathcal{B}(n\cdot(n-1)-d_{e,\hat{\gG}_{\mathrm{opt}}, \hat{\gG}_1}^*,p_m)+\mathcal{B}(d_{e,\hat{\gG}_{\mathrm{opt}}, \hat{\gG}_1}^*, 1-p_m)$. Since $d_e(\mA_{\hat{\gG}_{\mathrm{opt}}}, \mA_{\hat{\gG}_{\mathrm{new}}\rightarrow\hat{\gG}_{\mathrm{opt}}})\leq d_e(\mA_{\hat{\gG}_{\mathrm{opt}}}, \mA_{\hat{\gG}_{\mathrm{new}}})$, we have
	\begin{align*}
	\displaystyle &\E(\max(d_e(\mA_{\hat{\gG}_{\mathrm{opt}}}, \mA_{\hat{\gG}_1\rightarrow\hat{\gG}_{\mathrm{opt}}})-d_e(\mA_{\hat{\gG}_{\mathrm{opt}}}, \mA_{\hat{\gG}_{\mathrm{new}}\rightarrow\hat{\gG}_{\mathrm{opt}}}),0))\geq\E(\max(d_e(\mA_{\hat{\gG}_{\mathrm{opt}}}, \mA_{\hat{\gG}_1\rightarrow\hat{\gG}_{\mathrm{opt}}})-d_e(\mA_{\hat{\gG}_{\mathrm{opt}}}, \mA_{\hat{\gG}_{\mathrm{new}}}),0))\\&=\E(\max( d_{e,\hat{\gG}_{\mathrm{opt}}, \hat{\gG}_1}^*-\mathcal{B}(n\cdot(n-1)-d_{e,\hat{\gG}_{\mathrm{opt}}, \hat{\gG}_1}^*,p_m)-\mathcal{B}(d_{e,\hat{\gG}_{\mathrm{opt}}, \hat{\gG}_1}^*, 1-p_m),0)).
	\end{align*}
\end{proof}
\begin{repeatthm}{\ref{thm:EI_extreme}}[Expected improvement of unbiased agent and oracle agent]
	Suppose $\Sigma_{i,j}p(A^{\theta*}_{i,j}\neq A^{\gG_{\mathrm{opt}}}_{i,j}|\mA^*_\theta \sim \mQ^*_\theta)=b^*_{e,\theta}$ and assume $R-b=\alpha\cdot(\Sigma_{i,j}p(A^{\theta*}_{i,j}\neq A^{\gG_{\mathrm{opt}}}_{i,j}|\mA^*_\theta \sim \mQ^*_\theta)-d^*_{e,\gG_{\mathrm{opt}},\gG_{\theta_t}})$ for $i,j\in 1,2,\cdots,n$ and $i\neq j$, where $\alpha$ is a positive scaling factor and $\gG_{\theta_t}$ is a graph sampled at time step $t$ for obtaining the empirical approximation of policy gradient. With all $z^k_{i,j}$ initialized to 0, the expected improvement after one policy update with learning rate $\eta$ is no less than 
	\begin{equation*}
	\mathrm{LBEI}_{\mathrm{RLU}}=b^*_{e,\theta}-(n_w\cdot \frac{1}{1+(\frac{1}{p_w}-1)\cdot \mathrm{e}^{-2\alpha \eta(b^*_{e,\theta}-n_w)(1-p_w)}}+(n(n-1)-n_w)\cdot \frac{1}{1+(\frac{1}{p_w}-1)\cdot \mathrm{e}^{2\alpha \eta(b^*_{e,\theta}-n_w)\cdot p_w}})
	\end{equation*}
	for unbiased agent, where $p_w=\frac{b^*_{e,\theta}}{n(n-1)}, n_w=\mathcal{B}(n(n-1), \frac{b^*_{e,\theta}}{n(n-1)})$, and no less than
	\begin{equation*}
	\mathrm{LBEI}_{\mathrm{RLO}}=b^*_{e,\theta}-(n_w\cdot \frac{1}{1+(\frac{1}{p_w}-1)\cdot \mathrm{e}^{-2\alpha \eta(b^*_{e,\theta}-n_w)(1-p_w)}}+(b^*_{e,\theta}+1-n_w)\cdot \frac{1}{1+(\frac{1}{p_w}-1)\cdot \mathrm{e}^{2\alpha \eta(b^*_{e,\theta}-n_w)\cdot p_w}})
	\end{equation*}
	for oracle agent, where $p_w=\frac{b^*_{e,\theta}}{b^*_{e,\theta}+1}, n_w=\mathcal{B}(b^*_{e,\theta}+1, \frac{b^*_{e,\theta}}{b^*_{e,\theta}+1})$.
\end{repeatthm}
\begin{proof}
Under the REINFORCE rule, the policy gradient based on one sample is $\Sigma_{i,j} \bigtriangledown_\theta \log p(A_{i,j}^{\theta})\cdot (R-b)$ for $i,j\in 1,2,\cdots,n$, and the constraint $i\neq j$ can be added to only consider edges/connections. Since only two values 0 and 1 are allowed for each entry that denotes an edge connection, they can be mapped to ``correct'' and ``wrong'' by comparing entries between $\mA_\theta^*$ and $\mA_{\gG_{\mathrm{opt}}}$: ``correct'' means $A^{\theta*}_{i,j}= A^{\gG_{\mathrm{opt}}}_{i,j}$ and ``wrong'' means $A^{\theta*}_{i,j}\neq A^{\gG_{\mathrm{opt}}}_{i,j}$. For an entry in $\mA_\theta^*\sim \mQ^*_\theta$, let $p_c$ be the probability for it to be correct and $p_w$ the probability for it to be wrong, and let $z_c$ and $z_w$ be the logits for generating $p_c$ and $p_w$, respectively. Then
	\begin{align*}
	\frac{\partial \log p_c}{\partial z_c}=&\frac{\partial}{\partial z_c}\log(\frac{\mathrm{e}^{z_c}}{\mathrm{e}^{z_c}+\mathrm{e}^{z_w}})=\frac{\partial}{\partial z_c}(z_c-\log(\mathrm{e}^{z_c}+\mathrm{e}^{z_w}))\\=&1-\frac{1}{\mathrm{e}^{z_c}+\mathrm{e}^{z_w}}\cdot(\frac{\partial}{\partial z_c}(\mathrm{e}^{z_c}+\mathrm{e}^{z_w}))= 1-\frac{\mathrm{e}^{z_c}}{\mathrm{e}^{z_c}+\mathrm{e}^{z_w}}=1-p_c.
	\end{align*}
	Similarly, $\frac{\partial \log p_c}{\partial z_w}=-p_w$, $\frac{\partial \log p_w}{\partial z_w}=1-p_w$ and $\frac{\partial \log p_w}{\partial z_c}=-p_c$. For the entries that sample correctly, the policy gradient for updating $z_c$ and $z_w$ is $\frac{\partial\log p_c}{\partial z_c}\cdot \alpha(b^*_{e,\theta}-d^*_{e,\gG_{\mathrm{opt}},\gG_{\theta_t}})=(1-p_c)\cdot \alpha(b^*_{e,\theta}-d^*_{e,\gG_{\mathrm{opt}},\gG_{\theta_t}})$ and $-p_w\cdot \alpha(b^*_{e,\theta}-d^*_{e,\gG_{\mathrm{opt}},\gG_{\theta_t}})$, respectively. For the entries that sample wrong, the policy gradient for updating $z_c$ and $z_w$ is $\frac{\partial\log p_w}{\partial z_c}\cdot \alpha(b^*_{e,\theta}-d^*_{e,\gG_{\mathrm{opt}},\gG_{\theta_t}})=-p_c\cdot \alpha(b^*_{e,\theta}-d^*_{e,\gG_{\mathrm{opt}},\gG_{\theta_t}})$ and $(1-p_w)\cdot \alpha(b^*_{e,\theta}-d^*_{e,\gG_{\mathrm{opt}},\gG_{\theta_t}})$, respectively. Since $p_c = 1-p_w$, the policy gradients are always opposite but the same magnitude for $z_c$ and $z_w$. Because all the $z^k_{i,j}$ are initialized to 0, thus $z_c=-z_w$ for each entry. Given $p_w=\frac{\mathrm{e}^{z_w}}{\mathrm{e}^{z_c}+\mathrm{e}^{z_w}}$ and $z_c=-z_w$, then $\mathrm{e}^{z_c}=\sqrt{\frac{1-p_w}{p_w}}$ and $\mathrm{e}^{z_w}=\sqrt{\frac{p_w}{1-p_w}}$. Therefore, for the entries that sample correctly, $p_c$ is updated as 
	\begin{equation*}
	p_c^\prime = \frac{\mathrm{e}^{z_c}\cdot \mathrm{e}^{(1-p_c)\cdot\alpha(b^*_{e,\theta}-d^*_{e,\gG_{\mathrm{opt}},\gG_{\theta_t}})\cdot \eta}}{\mathrm{e}^{z_c}\cdot \mathrm{e}^{(1-p_c)\cdot\alpha(b^*_{e,\theta}-d^*_{e,\gG_{\mathrm{opt}},\gG_{\theta_t}})\cdot \eta}+\mathrm{e}^{z_w}\cdot \mathrm{e}^{-p_w\cdot\alpha(b^*_{e,\theta}-d^*_{e,\gG_{\mathrm{opt}},\gG_{\theta_t}})\cdot \eta}}.
	\end{equation*}
	For the entries that sample wrong, $p_w$ is updated as 
	\begin{equation*}
	p_w^\prime = \frac{\mathrm{e}^{z_w}\cdot \mathrm{e}^{(1-p_w)\cdot\alpha(b^*_{e,\theta}-d^*_{e,\gG_{\mathrm{opt}},\gG_{\theta_t}})\cdot \eta}}{\mathrm{e}^{z_w}\cdot \mathrm{e}^{(1-p_w)\cdot\alpha(b^*_{e,\theta}-d^*_{e,\gG_{\mathrm{opt}},\gG_{\theta_t}})\cdot \eta}+\mathrm{e}^{z_c}\cdot \mathrm{e}^{-p_c\cdot\alpha(b^*_{e,\theta}-d^*_{e,\gG_{\mathrm{opt}},\gG_{\theta_t}})\cdot \eta}}.
	\end{equation*}
	
	For the unbiased agent, in order to have $\Sigma_{i,j}p(A^{\theta*}_{i,j}\neq A^{\gG_{\mathrm{opt}}}_{i,j}|\mA^*_\theta \sim \mQ^*_\theta)=b^*_{e,\theta}$, every entry should have the same $p_w=\frac{b^*_{e,\theta}}{n(n-1)}$. The number of different non-diagonal entries between $\mA_\theta^*$ and $\mA_{\gG_{\mathrm{opt}}}$ is then $n_w = d^*_{e,\gG_{\mathrm{opt}},\gG_{\theta_t}}=\mathcal{B}(n(n-1), \frac{b^*_{e,\theta}}{n(n-1)})$. Since there are $n_w$ entries sampled wrong and $n(n-1)-n_w$ sampled correctly, the expected number of different non-diagonal entries between $\mA^*_\theta \sim \mQ^*_\theta$ and $\mA_{\gG_{\mathrm{opt}}}$ after updating every $p_c$ and $p_w$ becomes $n_w\cdot p_w^\prime+(n(n-1)-n_w)\cdot (1-p_c^\prime)$. Considering the possibility of further permuting the updated entries in $\mQ^*_{\theta_t}$ to obtain $\mQ^*_{\theta_{t+1}}$, and supposing that $\mQ^*_{\theta_t}=\mQ^*_\theta$, then $n_w\cdot p_w^\prime+(n(n-1)-n_w)\cdot (1-p_c^\prime)\geq  \Sigma_{i,j}p(A^{\theta_{t+1}*}_{i,j}\neq A^{\gG_{\mathrm{opt}}}_{i,j}|\mA^*_{\theta_{t+1}} \sim \mQ^*_{\theta_{t+1}})$. As a result, the expected improvement is 
	\begin{align*}
	&\Sigma_{i,j}p(A^{\theta_t*}_{i,j}\neq A^{\gG_{\mathrm{opt}}}_{i,j}|\mA^*_{\theta_t} \sim \mQ^*_{\theta_t}) - \Sigma_{i,j}p(A^{\theta_{t+1}*}_{i,j}\neq A^{\gG_{\mathrm{opt}}}_{i,j}|\mA^*_{\theta_{t+1}} \sim \mQ^*_{\theta_{t+1}})\\&\geq b^*_{e,\theta}-(n_w\cdot p_w^\prime+(n(n-1)-n_w)\cdot (1-p_c^\prime))\\&=b^*_{e,\theta}-(n_w\cdot \frac{1}{1+(\frac{1}{p_w}-1)\cdot \mathrm{e}^{-2\alpha \eta(b^*_{e,\theta}-n_w)(1-p_w)}}+(n(n-1)-n_w)\cdot \frac{1}{1+(\frac{1}{p_w}-1)\cdot \mathrm{e}^{2\alpha \eta(b^*_{e,\theta}-n_w)\cdot p_w}}), 
	\end{align*}
	where $p_w=\frac{b^*_{e,\theta}}{n(n-1)}, n_w=\mathcal{B}(n(n-1), \frac{b^*_{e,\theta}}{n(n-1)})$.
	
	For the oracle agent, if $b^*_{e,\theta}$ is an integer and there are exactly $n(n-1)-b^*_{e,\theta}$ entries that have $p_c=1.0$, the remaining $b^*_{e,\theta}$ entries can only have $p_c=0$, due to the pre-condition that $\Sigma_{i,j}p(A^{\theta*}_{i,j}\neq A^{\gG_{\mathrm{opt}}}_{i,j}|\mA^*_\theta \sim \mQ^*_\theta)=b^*_{e,\theta}$. This setup results in a stuck agent that can no longer explore and update itself, and it does not satisfy Definition~\ref{def:extreme_agent}. Therefore, the maximum number of entries with $p_c=1.0$ can only be $n(n-1)-b^*_{e,\theta}-1$ when $b^*_{e,\theta}$ is an integer, and $n(n-1)-\lceil b^*_{e,\theta} \rceil$ otherwise. Since $n(n-1)-\lceil b^*_{e,\theta} \rceil = n(n-1)-\lfloor b^*_{e,\theta}\rfloor-1$ always holds true and $n(n-1)-b^*_{e,\theta}-1=n(n-1)-\lfloor b^*_{e,\theta}\rfloor-1$ is true when $b^*_{e,\theta}$ is an integer, $n(n-1)-\lfloor b^*_{e,\theta}\rfloor-1$ can always be used to describe the maximum number of entries with $p_c=1.0$. The remaining $\lfloor b^*_{e,\theta} \rfloor+1$ entries would have $p_w=\frac{b^*_{e,\theta}}{\lfloor b^*_{e,\theta}\rfloor+1}$, within which $n_w$ entries are sampled wrong, and $\lfloor b^*_{e,\theta}\rfloor+1-n_w$ entries are sampled correctly, with $n_w=\mathcal{B}(\lfloor b^*_{e,\theta}\rfloor+1,\frac{b^*_{e,\theta}}{\lfloor b^*_{e,\theta}\rfloor+1})$. The expected number of different non-diagonal entries between $\mA^*_\theta \sim \mQ^*_\theta$ and $\mA_{\gG_{\mathrm{opt}}}$ after updating every $p_c$ and $p_w$ thus becomes $n_w\cdot p_w^\prime+(\lfloor b^*_{e,\theta}\rfloor+1-n_w)\cdot (1-p_c^\prime)$. Similar to the analysis of the unbiased agent, we have
	\begin{align*}
	&\Sigma_{i,j}p(A^{\theta_t*}_{i,j}\neq A^{\gG_{\mathrm{opt}}}_{i,j}|\mA^*_{\theta_t} \sim \mQ^*_{\theta_t}) - \Sigma_{i,j}p(A^{\theta_{t+1}*}_{i,j}\neq A^{\gG_{\mathrm{opt}}}_{i,j}|\mA^*_{\theta_{t+1}} \sim \mQ^*_{\theta_{t+1}})\\&\geq b^*_{e,\theta}-(n_w\cdot \frac{1}{1+(\frac{1}{p_w}-1)\cdot \mathrm{e}^{-2\alpha \eta(b^*_{e,\theta}-n_w)(1-p_w)}}+(\lfloor b^*_{e,\theta}\rfloor+1-n_w)\cdot \frac{1}{1+(\frac{1}{p_w}-1)\cdot \mathrm{e}^{2\alpha \eta(b^*_{e,\theta}-n_w)\cdot p_w}})
	\end{align*}
	for the oracle agent, where $p_w=\frac{b^*_{e,\theta}}{\lfloor b^*_{e,\theta}\rfloor+1}, n_w=\mathcal{B}(\lfloor b^*_{e,\theta}\rfloor+1, \frac{b^*_{e,\theta}}{\lfloor b^*_{e,\theta}\rfloor+1})$.
\end{proof}

\begin{repeatcoro}{\ref{coro:EI_SEPX_error}}[Effect of GED errors on $\mathrm{LBEI}_{\mathrm{SEPX}}$]
With error ratio $\epsilon$ in calculating $d_{e,\hat{\gG}_1, \hat{\gG}_2}^*$, $\mathrm{LBEI}_{\mathrm{SEPX}}$ becomes
\begin{align*}
\mathrm{LBEI}_{\mathrm{SEPX}}^\epsilon=& \ (d_{e,\hat{\gG}_1, \hat{\gG}_2}^\epsilon-\lfloor d_{e,\hat{\gG}_1, \hat{\gG}_2}^\epsilon\rfloor)\cdot\E(\max(\frac{d_{e,\hat{\gG}_{\mathrm{opt}}, \hat{\gG}_1}^*\cdot (\lfloor d_{e,\hat{\gG}_1, \hat{\gG}_2}^\epsilon\rfloor+1)}{n\cdot(n-1)-\lfloor n_{se}^\epsilon\rfloor}-\mathcal{B}(\lfloor d_{e,\hat{\gG}_1, \hat{\gG}_2}^\epsilon\rfloor+1, 0.5),0))\\
&+(\lfloor d_{e,\hat{\gG}_1, \hat{\gG}_2}^\epsilon\rfloor+1-d_{e,\hat{\gG}_1, \hat{\gG}_2}^\epsilon)\cdot\E(\max(\frac{d_{e,\hat{\gG}_{\mathrm{opt}}, \hat{\gG}_1}^*\cdot \lfloor d_{e,\hat{\gG}_1, \hat{\gG}_2}^\epsilon\rfloor}{n\cdot(n-1)-\lceil n_{se}^\epsilon\rceil}-\mathcal{B}(\lfloor d_{e,\hat{\gG}_1, \hat{\gG}_2}^\epsilon \rfloor, 0.5),0)),
\end{align*}
where $n_{se}^\epsilon=\max(n\cdot(n-1)-d_{e,\hat{\gG}_{\mathrm{opt}}, \hat{\gG}_1}^*-d_{e,\hat{\gG}_1, \hat{\gG}_2}^\epsilon, 0).$
\end{repeatcoro}
\begin{proof}
Given the assumption that the resulting GED can only be either $\lfloor d_{e,\hat{\gG}_1, \hat{\gG}_2}^\epsilon\rfloor$ or $\lfloor d_{e,\hat{\gG}_1, \hat{\gG}_2}^\epsilon\rfloor+1$ following a Bernoulli distribution, and the expectation is $d_{e,\hat{\gG}_1, \hat{\gG}_2}^\epsilon=d_{e,\hat{\gG}_1, \hat{\gG}_2}^*\cdot(1+\epsilon)$, the probabilities for getting the two results can be derived as $p(\lfloor d_{e,\hat{\gG}_1, \hat{\gG}_2}^\epsilon\rfloor)=\lfloor d_{e,\hat{\gG}_1, \hat{\gG}_2}^\epsilon\rfloor+1-d_{e,\hat{\gG}_1, \hat{\gG}_2}^\epsilon$ and $p(\lfloor d_{e,\hat{\gG}_1, \hat{\gG}_2}^\epsilon\rfloor+1)=d_{e,\hat{\gG}_1, \hat{\gG}_2}^\epsilon-\lfloor d_{e,\hat{\gG}_1, \hat{\gG}_2}^\epsilon\rfloor$. In case of getting $\lfloor d_{e,\hat{\gG}_1, \hat{\gG}_2}^\epsilon\rfloor$ as the GED calculation result, $\mathrm{LBEI}_{\mathrm{SEPX}}$ becomes
\begin{equation*}
\mathrm{LBEI}_{\mathrm{SEPX}}|\lfloor d_{e,\hat{\gG}_1, \hat{\gG}_2}^\epsilon\rfloor=\E(\max(\frac{d_{e,\hat{\gG}_{\mathrm{opt}}, \hat{\gG}_1}^*\cdot \lfloor d_{e,\hat{\gG}_1, \hat{\gG}_2}^\epsilon\rfloor}{n\cdot(n-1)-\lceil n_{se}^\epsilon\rceil}-\mathcal{B}(\lfloor d_{e,\hat{\gG}_1, \hat{\gG}_2}^\epsilon \rfloor, 0.5),0)),
\end{equation*} where $\lceil n_{se}^\epsilon\rceil=\max(n\cdot(n-1)-d_{e,\hat{\gG}_{\mathrm{opt}}, \hat{\gG}_1}^*-\lfloor d_{e,\hat{\gG}_1, \hat{\gG}_2}^\epsilon\rfloor, 0).$ Similarly,
\begin{equation*}
\mathrm{LBEI}_{\mathrm{SEPX}}|(\lfloor d_{e,\hat{\gG}_1, \hat{\gG}_2}^\epsilon\rfloor+1)=\E(\max(\frac{d_{e,\hat{\gG}_{\mathrm{opt}}, \hat{\gG}_1}^*\cdot (\lfloor d_{e,\hat{\gG}_1, \hat{\gG}_2}^\epsilon\rfloor+1)}{n\cdot(n-1)-\lfloor n_{se}^\epsilon\rfloor}-\mathcal{B}(\lfloor d_{e,\hat{\gG}_1, \hat{\gG}_2}^\epsilon\rfloor+1, 0.5),0)),
\end{equation*} if $\lfloor d_{e,\hat{\gG}_1, \hat{\gG}_2}^\epsilon\rfloor+1$ is not an integer. If $\lfloor d_{e,\hat{\gG}_1, \hat{\gG}_2}^\epsilon\rfloor+1$ is an integer, then $p(\lfloor d_{e,\hat{\gG}_1, \hat{\gG}_2}^\epsilon\rfloor+1)=d_{e,\hat{\gG}_1, \hat{\gG}_2}^\epsilon-\lfloor d_{e,\hat{\gG}_1, \hat{\gG}_2}^\epsilon\rfloor=0$, so this case does not need to be considered.
By combining the two cases,
\begin{align*}
\mathrm{LBEI}_{\mathrm{SEPX}}^\epsilon=& \ p(\lfloor d_{e,\hat{\gG}_1, \hat{\gG}_2}^\epsilon\rfloor)*\mathrm{LBEI}_{\mathrm{SEPX}}|\lfloor d_{e,\hat{\gG}_1, \hat{\gG}_2}^\epsilon\rfloor+p(\lfloor d_{e,\hat{\gG}_1, \hat{\gG}_2}^\epsilon\rfloor+1)\cdot\mathrm{LBEI}_{\mathrm{SEPX}}|(\lfloor d_{e,\hat{\gG}_1, \hat{\gG}_2}^\epsilon\rfloor+1)\\
=& \ (d_{e,\hat{\gG}_1, \hat{\gG}_2}^\epsilon-\lfloor d_{e,\hat{\gG}_1, \hat{\gG}_2}^\epsilon\rfloor)\cdot\E(\max(\frac{d_{e,\hat{\gG}_{\mathrm{opt}}, \hat{\gG}_1}^*\cdot (\lfloor d_{e,\hat{\gG}_1, \hat{\gG}_2}^\epsilon\rfloor+1)}{n\cdot(n-1)-\lfloor n_{se}^\epsilon\rfloor}-\mathcal{B}(\lfloor d_{e,\hat{\gG}_1, \hat{\gG}_2}^\epsilon\rfloor+1, 0.5),0))\\
&+(\lfloor d_{e,\hat{\gG}_1, \hat{\gG}_2}^\epsilon\rfloor+1-d_{e,\hat{\gG}_1, \hat{\gG}_2}^\epsilon)\cdot\E(\max(\frac{d_{e,\hat{\gG}_{\mathrm{opt}}, \hat{\gG}_1}^*\cdot \lfloor d_{e,\hat{\gG}_1, \hat{\gG}_2}^\epsilon\rfloor}{n\cdot(n-1)-\lceil n_{se}^\epsilon\rceil}-\mathcal{B}(\lfloor d_{e,\hat{\gG}_1, \hat{\gG}_2}^\epsilon \rfloor, 0.5),0)).
\end{align*}
\end{proof}

\subsection{List of Mathematical Symbols}\label{subsec:add_notation}
This section provides a list of all mathematical symbols used in this paper.

\bgroup
\def\arraystretch{1.5}
\begin{longtable}{p{1in}p{3.5in}}
	$\displaystyle \sV$ & A set of vertices\\
	$\displaystyle v_i$ & Vertex(node) with index $i$\\
	$\displaystyle \mathbb{E}$ & A set of directed edges\\
	$\displaystyle e_{i,j}$ & A directed edge from vertex $i$ to vertex $j$\\
	$\displaystyle \gG$ & A directed graph\\
	$\displaystyle |\gG|$ & The order of a directed graph $\gG$, which equals the number of its vertices\\
	$\displaystyle \gamma_v$ & A function that assigns an attribute (e.g., an integer) to each vertex of a directed graph\\
	$\displaystyle \gamma_e$ & A function that assigns an attribute (e.g., an integer) to each edge of a directed graph\\
	$\displaystyle \delta: \gG \rightarrow \gG'$ & A function that applies an elementary graph edit to transform $\gG$ to $\gG'$\\
	$\displaystyle \overline{\delta}=\delta_1, \delta_2, \ldots, \delta_d$ & A sequence of graph edit operations, and $d$ is the length of the resulting edit path\\
	$\displaystyle \mathrm{GED}(\gG_1, \gG_2)$ & The graph edit distance (GED) between $\gG_1$ and $\gG_2$\\
	$\displaystyle \Delta(\gG_1, \gG_2)$ & The set of all edit paths that transform $\gG_1$ to an isomorphism of $\gG_2$ (including $\gG_2$ itself)\\
	$\displaystyle c(\delta_i)$ & The cost of edit $\delta_i$ (in this work, all types of edit operations are defined to have the same cost of $1$)\\
	$\displaystyle \overline{\delta}_{\gG_1, \gG_2}^*$ & The edit path that minimizes the total edit cost to transform $\gG_1$ to an isomorphism of $\gG_2$ (including $\gG_2$ itself)\\
	$\displaystyle d_{\gG_1, \gG_2}^*$ & The length of the shortest edit path that transforms $\gG_1$ to an isomorphism of $\gG_2$ (including $\gG_2$ itself)\\
	$\displaystyle \pi$ & A permutation of multiple elements/indices\\
	$\displaystyle \lceil \cdot \rceil$ & The ceiling function\\
	$\displaystyle \lfloor\cdot\rfloor$ & The floor function\\
	$\displaystyle \mA_\gG$ & The attributed adjacency matrix (AA-matrix) for graph $\gG$\\
	$\displaystyle A^\gG_{i,j}$ & The entry in $i$th row and $j$th column of matrix $\mA_\gG$\\
	$\displaystyle \mI_n$ & Identity matrix with $n$ rows and $n$ columns\\
	$\displaystyle \mP_{\pi}$ & A permutation matrix based on permutation $\pi$\\
	$\displaystyle P^{\pi}_{i,j}$ & The entry in $i$th row and $j$th column of matrix $\mP_{\pi}$\\
	$\displaystyle d(\mA,\mB)$ & A function that returns the number of different entries between two matrices ($\mA$ and $\mB$ here) with same shape\\
	$\displaystyle \hat{\gG}$ & The extended graph of $\gG$ after adding null vertices\\
	$\displaystyle S_n$ & The set of all permutations of $\{1, 2, 3, \ldots, n\}$\\
	$\displaystyle \1_\mathrm{condition}$ & A function that returns 1 if the condition is true, 0 otherwise\\
	$\displaystyle \pi^*_{\hat{\gG}_1,\hat{\gG}_2}$ & The permutation that minimizes $d(\mA_{\hat{\gG}_1}, \mP_{\pi}\mA_{\hat{\gG}_2}\mP_{\pi}^{\top})$\\
	$\displaystyle \mA_{\hat{\gG}_2\rightarrow\hat{\gG}_1}$ & The permuted AA-matrix of $\hat{\gG}_2$ using permutation matrix $\mP_{\pi^*_{\hat{\gG}_1,\hat{\gG}_2}}$\\
	$\displaystyle r(\mA, \mB)$ & A function that returns a matrix inheriting each entry from $\mA$ or $\mB$ with probability 0.5 (that is, if $\mC=r(\mA, \mB)$, then $p(C_{i,j}=A_{i,j})=p(C_{i,j}=B_{i,j})=0.5$ for any valid $i, j$)\\
	$\displaystyle m(\mA)$ & A function that alters each element of $\mA$ with an equal probability\\
	$\displaystyle p_m$ & Mutation probability\\
	$\displaystyle f(\gG)$ & The fitness/reward of $\gG$\\
	$\displaystyle \gG_{\mathrm{opt}}$ & The global optimal graph\\
	$\displaystyle \max(\cdot,\cdot,\ldots,\cdot)$ & A function that returns the maximum value among all inputs\\
	$\displaystyle d_v(\mA, \mB)$ & A function that returns the number of different diagonal entries between two matrices ($\mA$ and $\mB$ here) with same shape\\
	$\displaystyle d_e(\mA, \mB)$ & A function that returns the number of different non-diagonal entries between two matrices ($\mA$ and $\mB$ here) with same shape\\
	$\displaystyle d_{e, \gG_1, \gG_2}^*$ & A simplified symbol to denote $d_e(\mA_{\gG_1}, \mA_{\gG_2\rightarrow\gG_1})$\\
	$\displaystyle n_s$ & Number of common entries among multiple matrices\\
	$\displaystyle n_{se}$ & Number of common non-diagonal entries among multiple matrices\\
	$\displaystyle \mathcal{B}(n, p)$ & The number of successful trials after sampling from a binomial distribution with $n$ trials and success probability of $p$\\
	$\displaystyle \mathrm{LBEI}_{\mathrm{SEPX}}$ & The lower bound of expected improvement of the SEP crossover\\
	$\displaystyle \mathrm{LBEI}_{\mathrm{STDX}}$ & The lower bound of expected improvement of the standard crossover\\
	$\displaystyle \mathrm{LBEI}_{\mathrm{MUTA}}$ & The lower bound of expected improvement of mutation\\
	$\displaystyle \mathrm{LBEI}_{\mathrm{RLU}}$ & The lower bound of expected improvement of the unbiased agent\\
	$\displaystyle \mathrm{LBEI}_{\mathrm{RLO}}$ & The lower bound of expected improvement of the oracle agent\\
	$\displaystyle n_{\mathrm{opt}}^1, n_1^1\ \mathrm{and}\ n_2^1$ & The number of ones in $\mA_{\hat{\gG}_{\mathrm{opt}}}$, $\mA_{\hat{\gG}_1}$ and $\mA_{\hat{\gG}_2}$, excluding diagonal entries\\
	$\displaystyle n_{\mathrm{opt}}^0, n_1^0\ \mathrm{and}\ n_2^0$ & The number of zeros in $\mA_{\hat{\gG}_{\mathrm{opt}}}$, $\mA_{\hat{\gG}_1}$ and $\mA_{\hat{\gG}_2}$, excluding diagonal entries\\
	$\displaystyle \mQ_\theta$ & A matrix in which each entry $Q_{i,j}^{\theta}$ defines a separate categorical distribution\\
	$\displaystyle \theta$ & The parameter set that contains the logits for defining the categorical distributions in $\mQ_\theta$\\
	$\displaystyle z_{i,j}^k$ & The logits used to defining a categorical distribution through softmax functions, where $k$ denotes the class label\\
	$\displaystyle R$ & The reward for the currently sampled architecture in a RL run\\
	$\displaystyle b$ & A baseline reward to reduce the variance of gradient estimate\\
	$\displaystyle \mQ^*_{\theta}$ & The optimal permutation of $\mQ_{\theta}$ as defined in Lemma~\ref{lem:lb_expected_GED} in Appendix~\ref{subsec:app_proof}\\
	$\displaystyle t$ & The current time step\\
	$\displaystyle \theta_t$ & Policy parameter at time step $t$\\
	$\displaystyle \alpha$ & A positive scaling factor\\
	$\displaystyle \eta$ & Learning rate\\
	$\displaystyle \epsilon$ & Error ratio\\
	$\displaystyle d_{e,\gG_1, \gG_2}^\epsilon$ & The resulting $d_{e, \gG_1, \gG_2}^*$ with error ratio $\epsilon$ in GED calculation\\
	$\displaystyle \mathrm{LBEI}_{\mathrm{SEPX}}^\epsilon$ & The resulting $\mathrm{LBEI}_{\mathrm{SEPX}}$ with error ratio $\epsilon$ in GED calculation\\
	$\displaystyle n_{se}^\epsilon$ & The resulting $n_{se}$ with error ratio $\epsilon$ in GED calculation\\
\end{longtable}
\egroup
\vspace{0.25cm}

\subsection{Example for Demonstrating the Permutation Problem and the SEP Crossover Solution}\label{subsec:add_sep_demo}
Figure~\ref{fig:demo_permutation} provides a visual example of the permutation problem and how the SEP crossover solves it.

\begin{figure*}[h]
	\centering
	\includegraphics[width=0.95\linewidth]{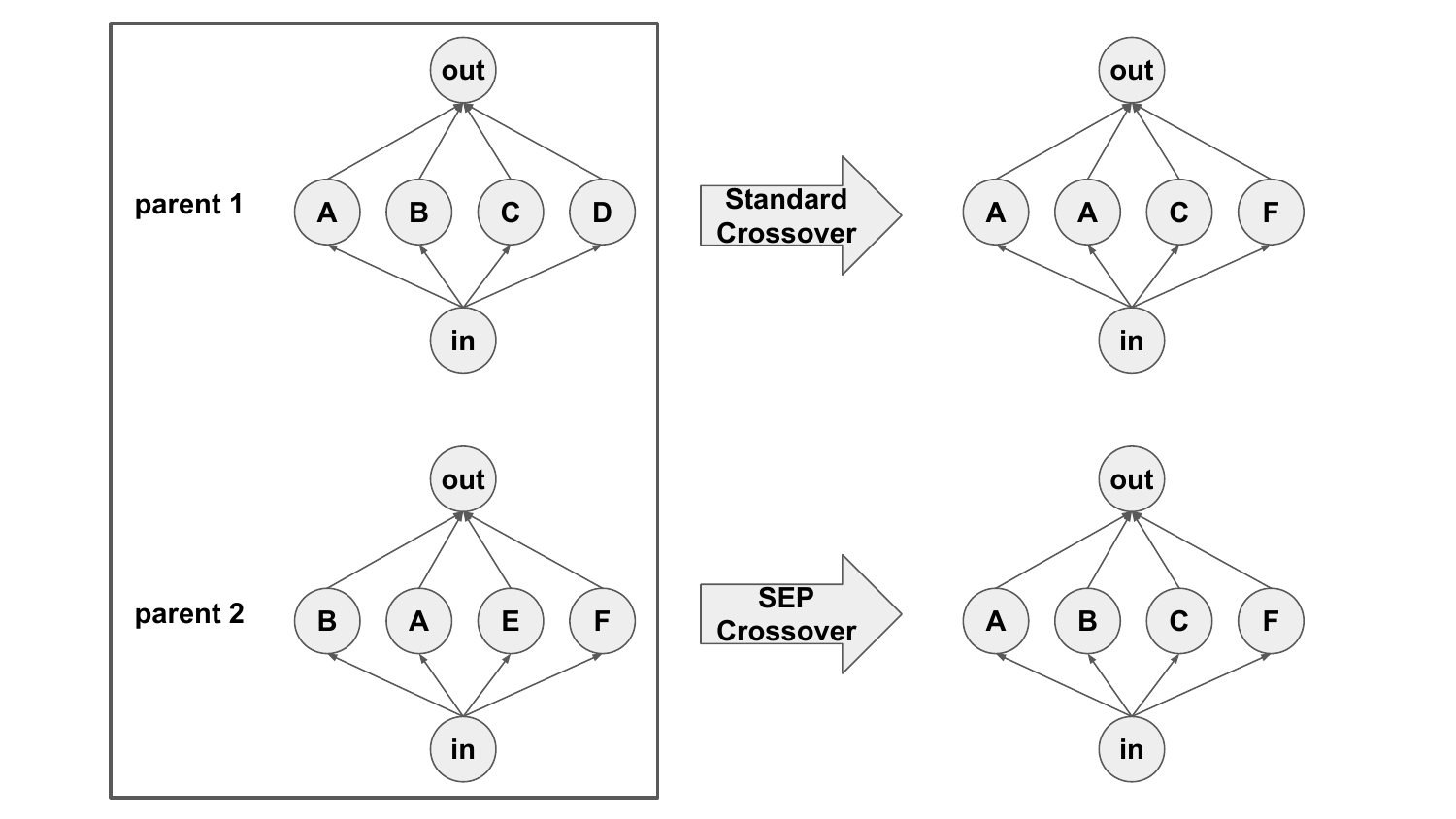}
	\caption{\textbf{The permutation problem and the SEP crossover solution.} The two parent architectures share vertices A and B. Although these two vertices appear in a different order, together they implement the same function, and this function should not be disrupted during crossover. However, standard crossover cannot identify the subgraph isomorphism, and it loses this substructure. In contrast, the shortest edit path calculation recognizes the isomorphism, and as a result, the SEP crossover preserves this substructure. Thus, the SEP crossover only explores the parts that are functionally inconsistent between the two parents.
		\label{fig:demo_permutation}
	}
\end{figure*}

\subsection{Experimental Setup Details}\label{subsec:Exp_Setup}
For experiments in Section~\ref{subsec:applicability}, all the RE-based variants used a population size of 100 and tournament selection with size 10. For NAS-bench-101, GED to the global optimal architecture was used as the fitness, and 50 independent runs were performed, each with a maximum number of evaluations of $10^3$. The allowed mutation operations were the same as in the original NAS-bench-101 example code (\href{https://github.com/google-research/nasbench}{https://github.com/google-research/nasbench}). For RL, the same implementation as in \href{https://github.com/automl/nas_benchmarks}{https://github.com/automl/nas\_benchmarks} was used. For experiments in Section~\ref{subsec:noisy}, the learning rate was 0.5, as recommended by \citet{Ying19}.

For NAS-bench-NLP, GED to the GRU architecture was used as fitness, and 50 independent runs were performed, each with a maximum number of evaluations of $10^3$. The mutation operation was the same as in \href{https://github.com/automl/NASLib}{https://github.com/automl/NASLib}. In both benchmarks, for each crossover operation, the offspring was evaluated only if it was a valid architecture in the benchmark space and different from both parents. The maximum number of trials was 50, i.e.\ the current crossover was skipped after reaching this limit.

For experiments on NAS-bench-101 in Section~\ref{subsec:noise-free}, the experimental setups was the same as in Section~\ref{subsec:applicability} except the maximum number of evaluations was $2\times10^3$.

For experiments in Section~\ref{subsec:noisy}, the setup for the NAS-bench-101 experiments was the same as in Section~\ref{subsec:applicability}, except validation accuracy was used as the fitness during evolution and test accuracy as the final performance of each architecture. The maximum number of evaluations was $10^4$. In the experiments on NAS-bench-301, all the RE-based variants had a population size of 100 and tournament size of 10, and 30 independent runs were performed, each with a maximum number of evaluations of $2\times10^3$. The mutation operation followed the standard strategy in \href{https://github.com/automl/NASLib}{https://github.com/automl/NASLib}.

For experiments regarding BO methods (Appendix~\ref{subsec:add_bo}), the same setup as in \href{https://github.com/automl/nas_benchmarks}{https://github.com/automl/nas\_benchmarks} is used.

For experiments regarding path encoding (Appendix~\ref{subsec:add_path}), the default setup without cutoff as in \href{https://github.com/naszilla/naszilla}{https://github.com/naszilla/naszilla} is used.

\subsection{Additional Figures for Section~\ref{subsec:theory_comp}}\label{subsec:add_theo_comp}

This section includes the rest of the comparisons between the SEP crossover, standard crossover, mutation, and the two RL variants. Note that the color scales differ between figures to make the conclusions more clear.

Figure~\ref{fig:EI_nlp} compares $\mathrm{LBEI}_{\mathrm{SEPX}}$ vs.\ $\mathrm{LBEI}_{\mathrm{MUTA}}$ and $\mathrm{LBEI}_{\mathrm{STDX}}$ vs.\ $\mathrm{LBEI}_{\mathrm{MUTA}}$ for different combinations of $d_{e,\hat{\gG}_{\mathrm{opt}}, \hat{\gG}_1}^*$ and $d_{e,\hat{\gG}_1, \hat{\gG}_2}^*$ in NAS-bench-NLP \citep{Klyuchnikov22}. The standard setup was used: $n=12$, $n_{\mathrm{opt}}^1=14$, $n_1^1=11$, and $n_2^1=11$. Although the standard crossover is slightly worse than mutation in most cases, the SEP crossover has a considerable theoretical advantage.
\begin{figure}
	\centering
	\includegraphics[width=0.32\linewidth]{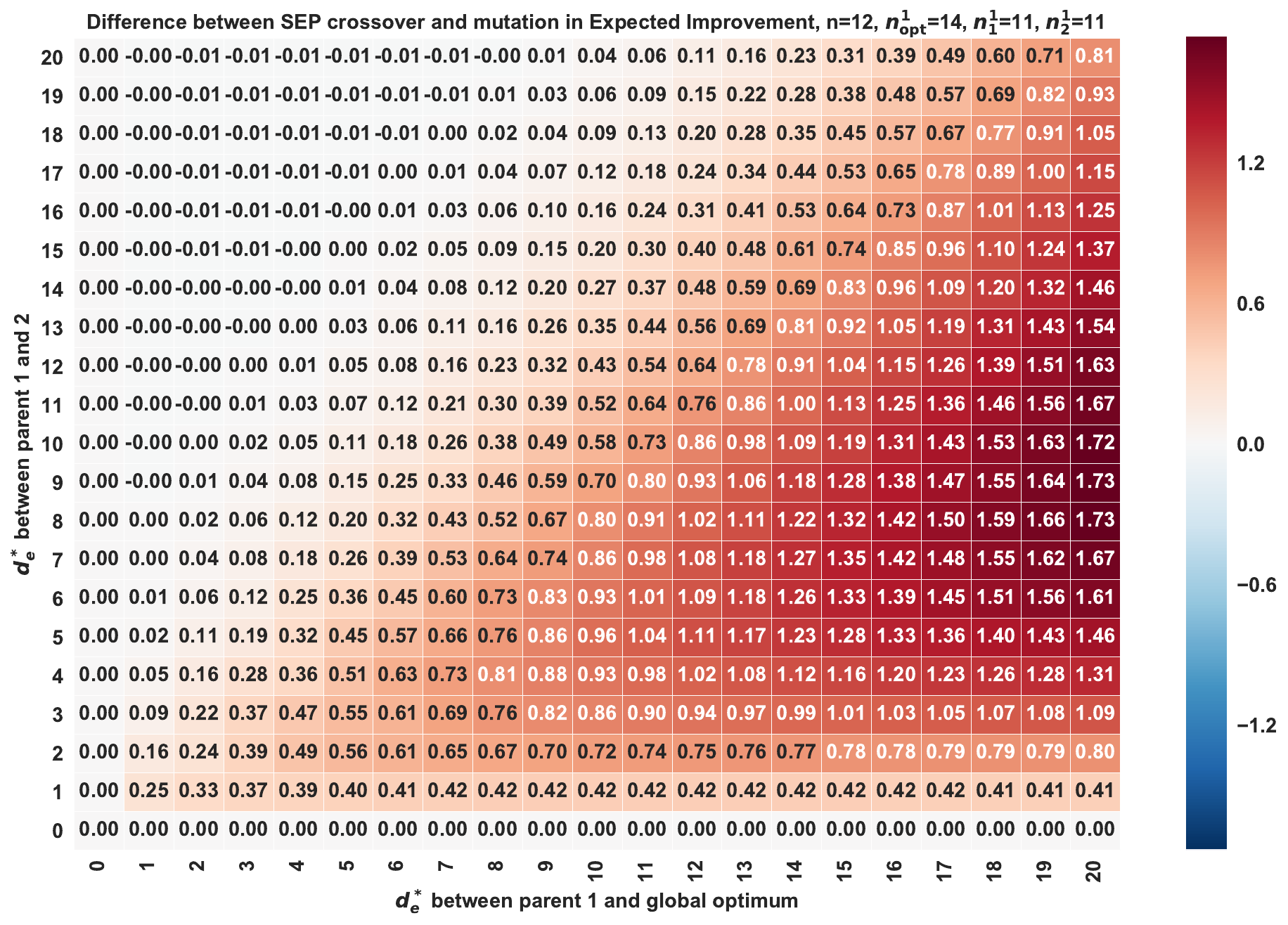}
	\includegraphics[width=0.32\linewidth]{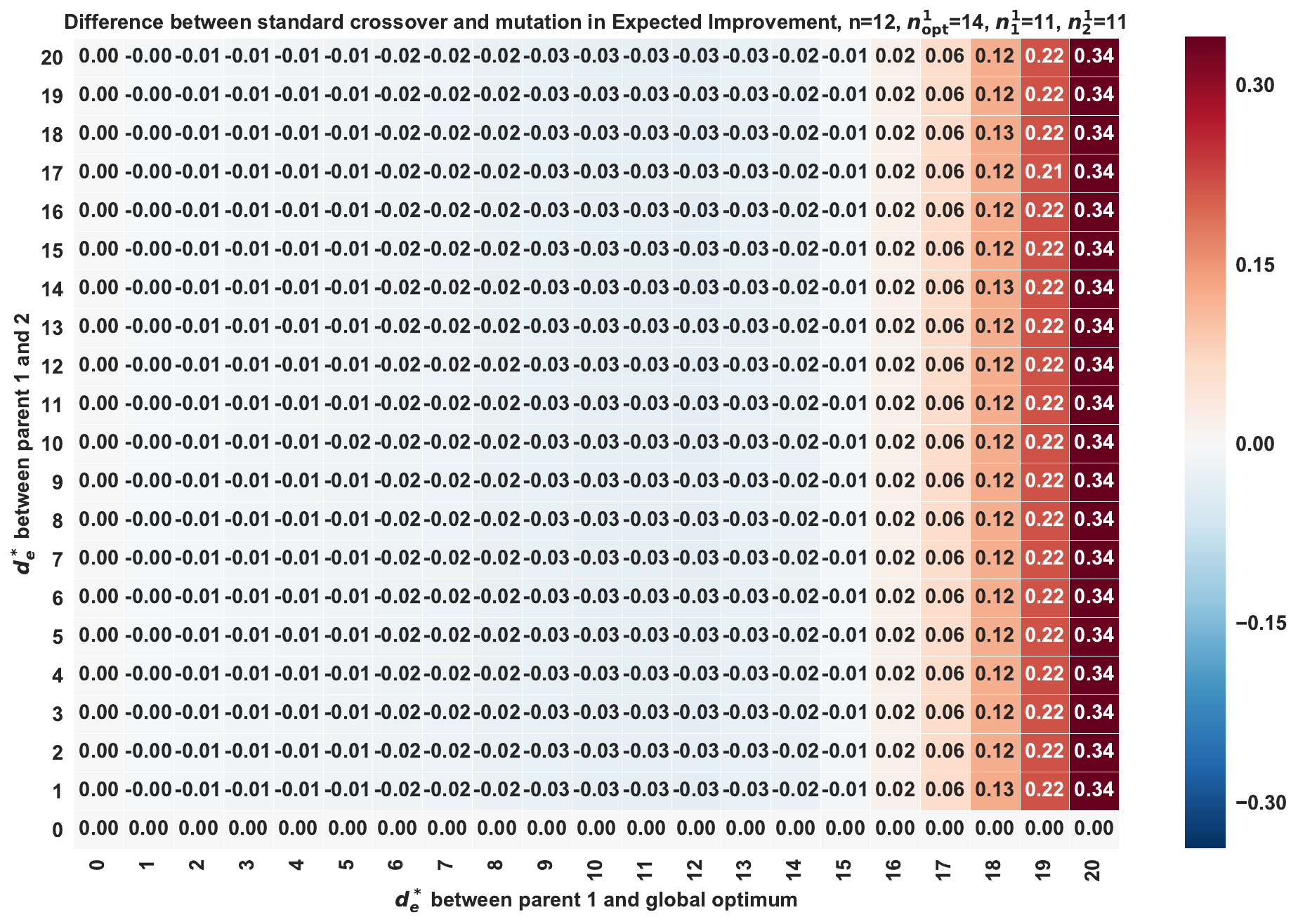}
	\caption{\textbf{Comparison of expected improvement in NAS-bench-NLP.} (Left) Differences between $\mathrm{LBEI}_{\mathrm{SEPX}}$ and $\mathrm{LBEI}_{\mathrm{MUTA}}$ under different $d_{e,\hat{\gG}_1, \hat{\gG}_2}^*$ ($y$-axis) and $d_{e,\hat{\gG}_{\mathrm{opt}}, \hat{\gG}_1}^*$ ($x$-axis) combinations. $\mathrm{LBEI}_{\mathrm{SEPX}}$ is larger than $\mathrm{LBEI}_{\mathrm{MUTA}}$ in most situations. (Right) Differences between $\mathrm{LBEI}_{\mathrm{STDX}}$ and $\mathrm{LBEI}_{\mathrm{MUTA}}$ under different $d_{e,\hat{\gG}_1, \hat{\gG}_2}^*$ ($y$-axis) and $d_{e,\hat{\gG}_{\mathrm{opt}}, \hat{\gG}_1}^*$ ($x$-axis) combinations. $\mathrm{LBEI}_{\mathrm{STDX}}$ is slightly smaller than $\mathrm{LBEI}_{\mathrm{MUTA}}$ in most situations. These two observations lead to the same conclusion for NAS-bench-NLP as for NAS-bench-101 in Figure~\ref{fig:EI_101}: Although the standard crossover has a slightly worse expected improvement than mutation under most circumstances, the SEP crossover has a considerable theoretical advantage.
		\label{fig:EI_nlp}
	}
\end{figure}

Figure~\ref{fig:EI_stdx_mut} compares $\mathrm{LBEI}_{\mathrm{STDX}}$ vs. $\mathrm{LBEI}_{\mathrm{MUTA}}$ under different $d_{e,\hat{\gG}_1, \hat{\gG}_2}^*$ and $d_{e,\hat{\gG}_{\mathrm{opt}}, \hat{\gG}_1}^*$ combinations. $\mathrm{LBEI}_{\mathrm{STDX}}$ is smaller than $\mathrm{LBEI}_{\mathrm{MUTA}}$ in most cases.
\begin{figure}
	\centering
	\includegraphics[width=0.32\linewidth]{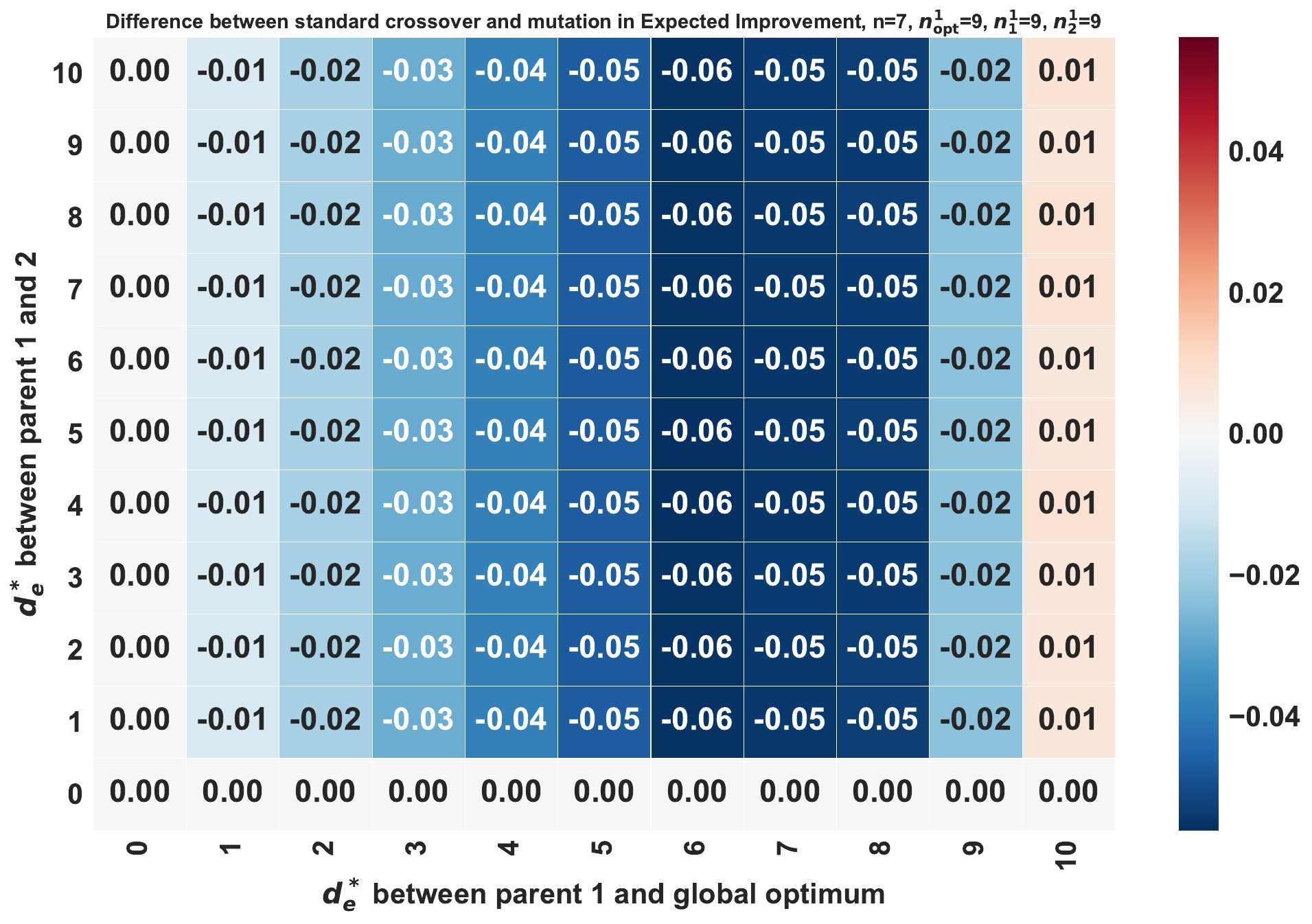}
	\caption{\textbf{Comparison of expected improvement between standard crossover and mutation in NAS-bench-101.} Differences between $\mathrm{LBEI}_{\mathrm{STDX}}$ and $\mathrm{LBEI}_{\mathrm{MUTA}}$ under different $d_{e,\hat{\gG}_1, \hat{\gG}_2}^*$ ($y$-axis) and $d_{e,\hat{\gG}_{\mathrm{opt}}, \hat{\gG}_1}^*$ ($x$-axis) combinations. $\mathrm{LBEI}_{\mathrm{STDX}}$ is smaller than $\mathrm{LBEI}_{\mathrm{MUTA}}$ in most cases.
		\label{fig:EI_stdx_mut}
	}
\end{figure}

Figure~\ref{fig:rl_self} shows $\mathrm{LBEI}_{\mathrm{RLU}}$, $\mathrm{LBEI}_{\mathrm{RLO}}$ and $\mathrm{LBEI}_{\mathrm{RLU}}$ vs. $\mathrm{LBEI}_{\mathrm{RLO}}$ under different $\alpha \cdot \eta$ values. A $\alpha \cdot \eta$ value of 0.1 provides the best tradeoff between unbiased agent and oracle agent.
\begin{figure}
	\centering
	\includegraphics[width=0.32\linewidth]{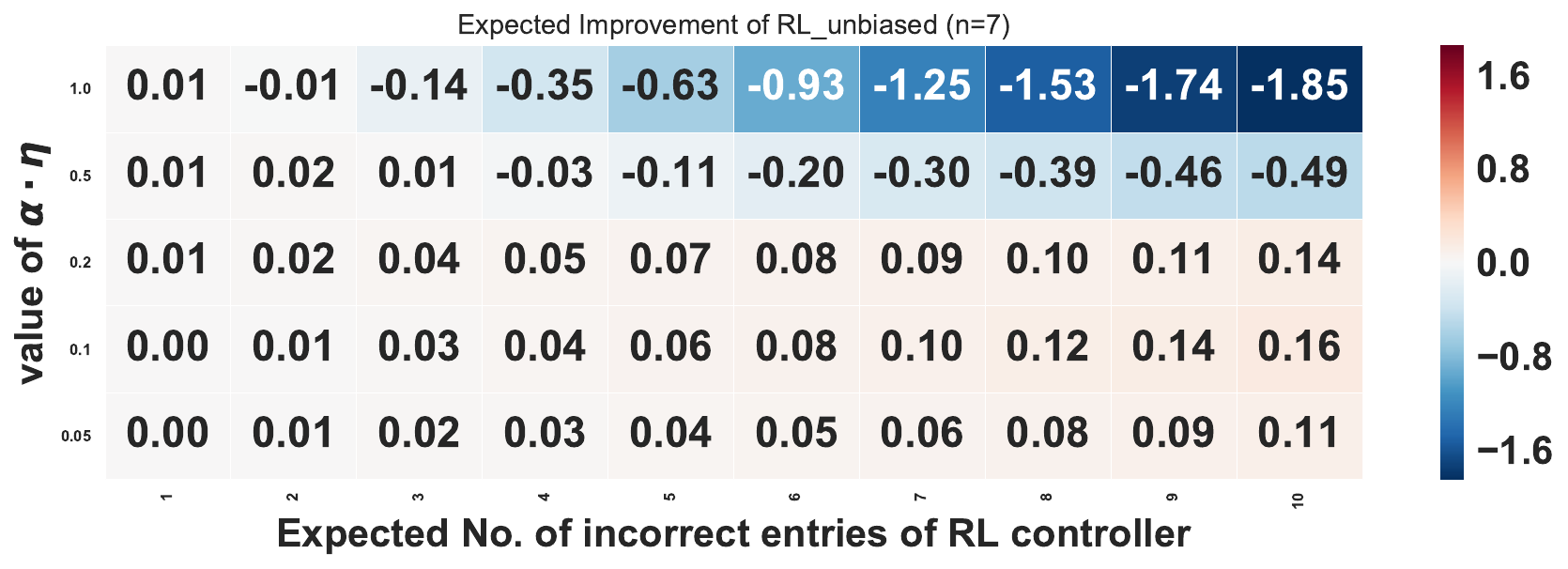}
	\includegraphics[width=0.32\linewidth]{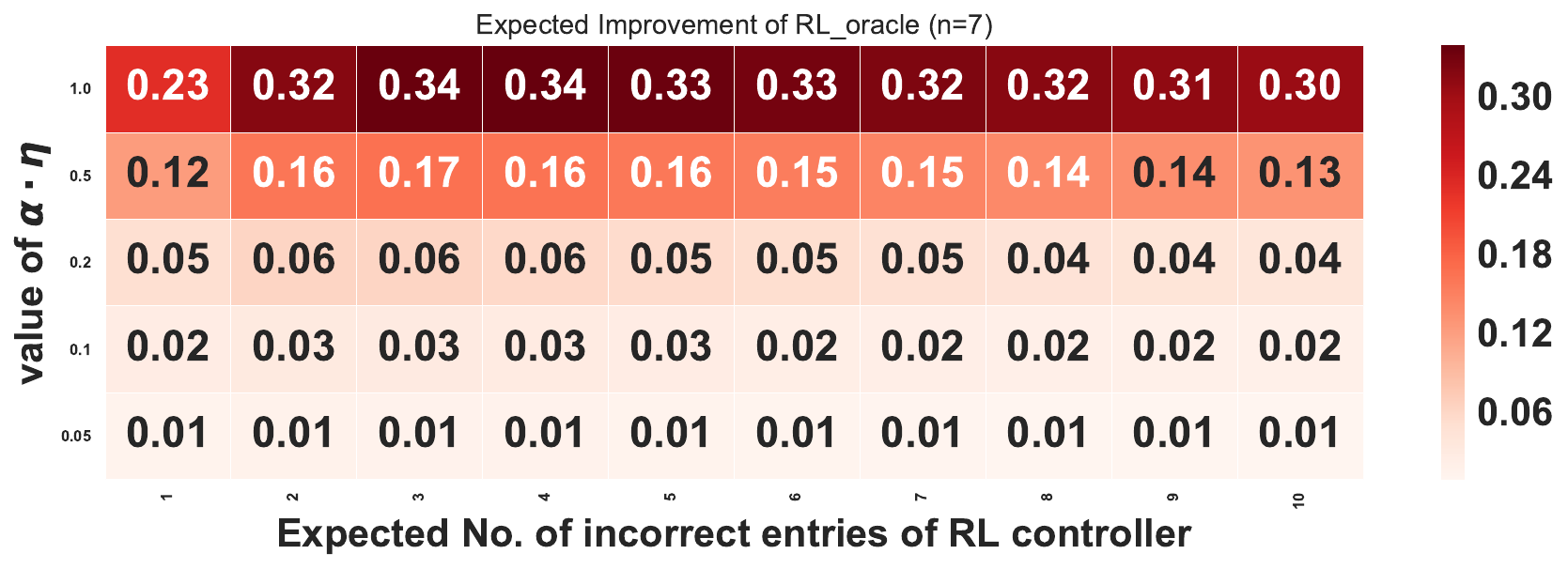}
        \includegraphics[width=0.32\linewidth]{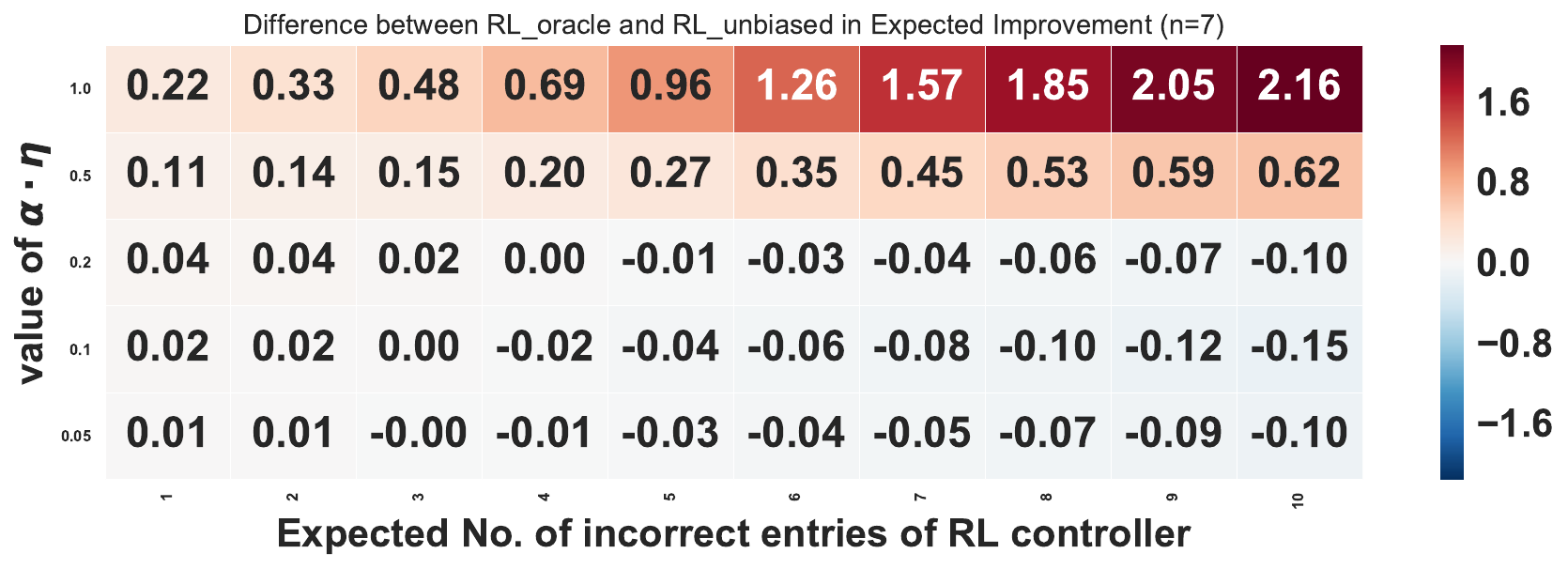}
	\caption{\textbf{Expected improvement of RL} (Left) $\mathrm{LBEI}_{\mathrm{RLU}}$ under different $\alpha \cdot \eta$ values. (middle) $\mathrm{LBEI}_{\mathrm{RLO}}$ under different $\alpha \cdot \eta$ values. (right) $\mathrm{LBEI}_{\mathrm{RLO}}-\mathrm{LBEI}_{\mathrm{RLU}}$ under different $\alpha \cdot \eta$ values. A $\alpha \cdot \eta$ value of 0.1 provides the best tradeoff between unbiased agent and oracle agent.
		\label{fig:rl_self}
	}
\end{figure}

Figure~\ref{fig:rl_muta} compares $\mathrm{LBEI}_{\mathrm{RLU}}$ vs. $\mathrm{LBEI}_{\mathrm{MUTA}}$ and $\mathrm{LBEI}_{\mathrm{RLO}}$ vs. $\mathrm{LBEI}_{\mathrm{MUTA}}$ under different $\alpha \cdot \eta$ values. Whereas unbiased agent is generally worse than mutation, the oracle agent is better in some cases and worse in others.
\begin{figure}
	\centering
	\includegraphics[width=0.32\linewidth]{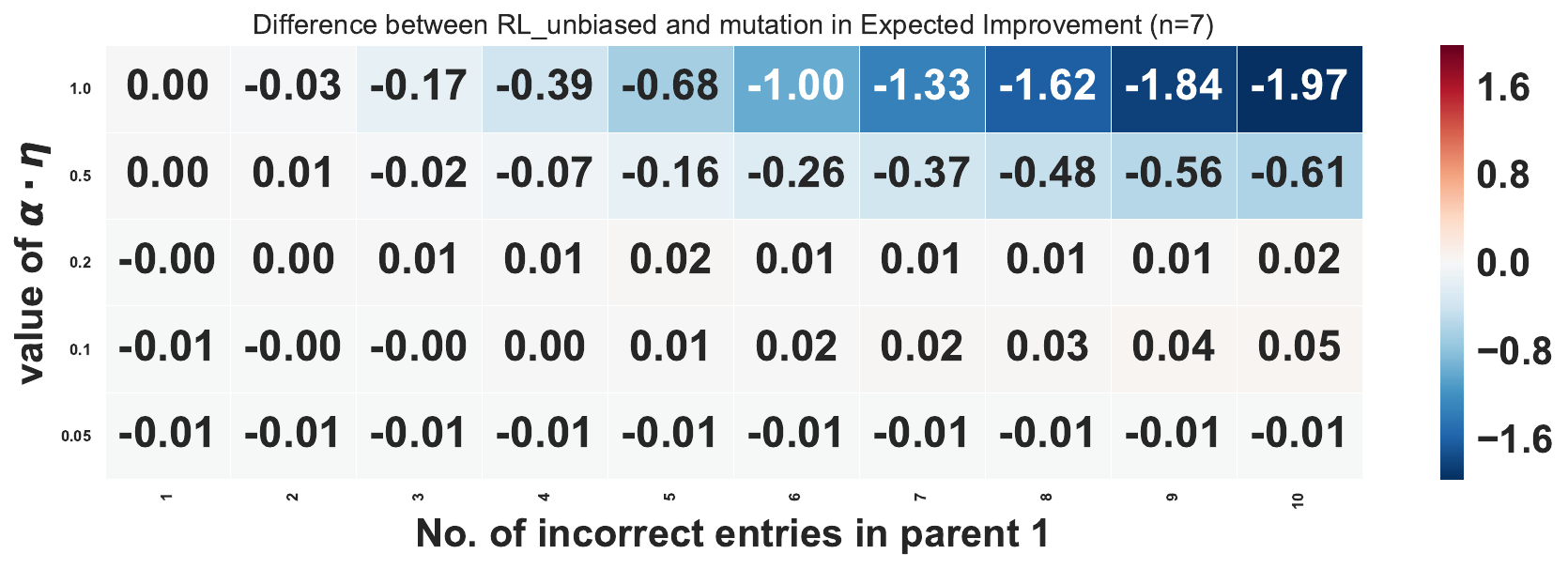}
	\includegraphics[width=0.32\linewidth]{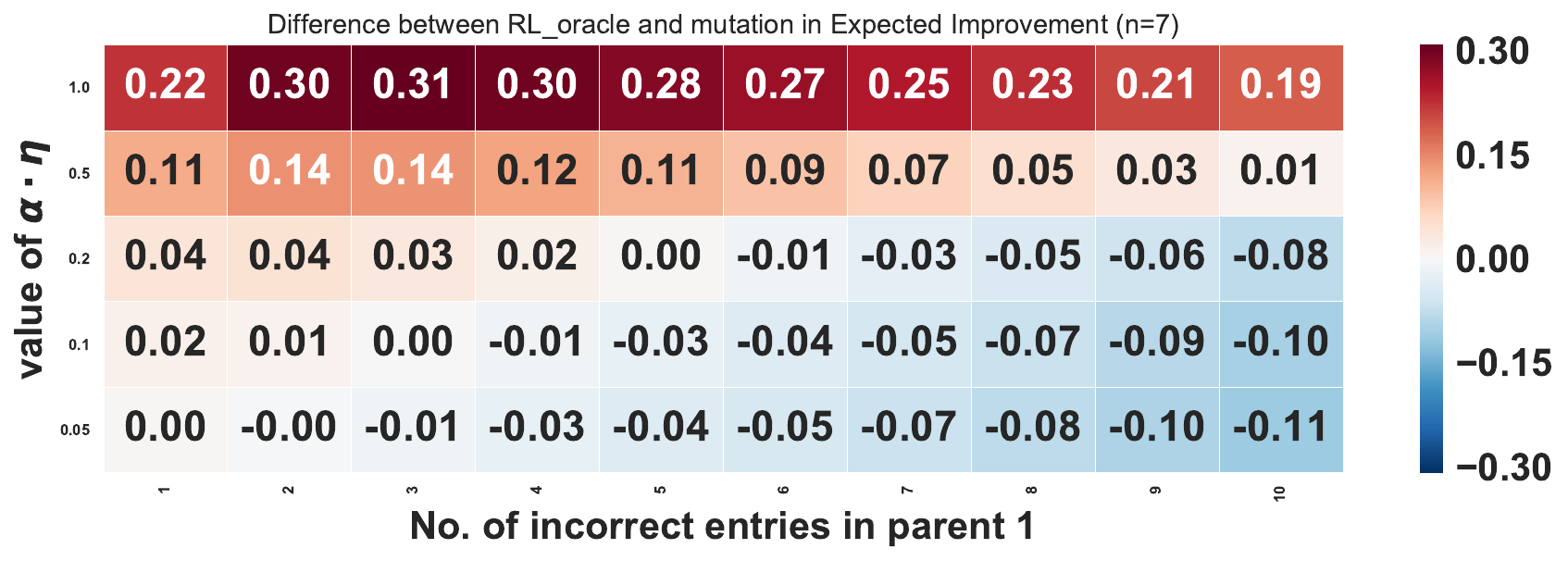}
	\caption{\textbf{Comparison of expected improvement between RL and mutation} (Left) $\mathrm{LBEI}_{\mathrm{RLU}}-\mathrm{LBEI}_{\mathrm{MUTA}}$ under different $\alpha \cdot \eta$ values. (right) $\mathrm{LBEI}_{\mathrm{RLO}}-\mathrm{LBEI}_{\mathrm{MUTA}}$ under different $\alpha \cdot \eta$ values. Whereas unbiased agent is generally worse than mutation, the oracle agent is better in some cases and worse in others.
		\label{fig:rl_muta}
	}
\end{figure}

Figure~\ref{fig:stdx_rl} compares $\mathrm{LBEI}_{\mathrm{STDX}}$ vs. $\mathrm{LBEI}_{\mathrm{RLU}}$ and $\mathrm{LBEI}_{\mathrm{STDX}}$ vs. $\mathrm{LBEI}_{\mathrm{RLO}}$ under different $d_{e,\hat{\gG}_1, \hat{\gG}_2}^*$ and $d_{e,\hat{\gG}_{\mathrm{opt}}, \hat{\gG}_1}^*$ combinations. The standard crossover is slightly worse than both RL agents in most cases.
\begin{figure}
	\centering
	\includegraphics[width=0.32\linewidth]{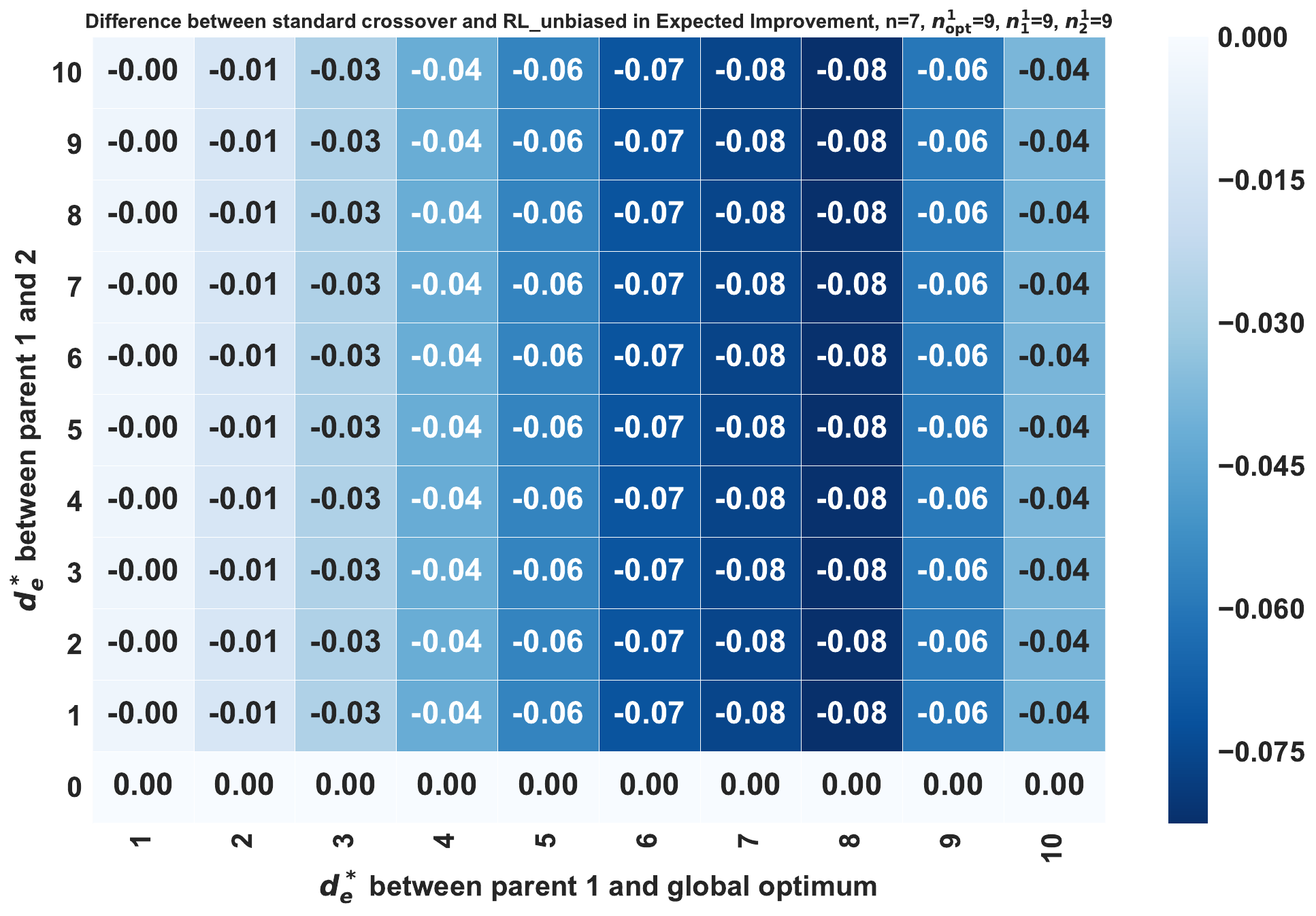}
	\includegraphics[width=0.32\linewidth]{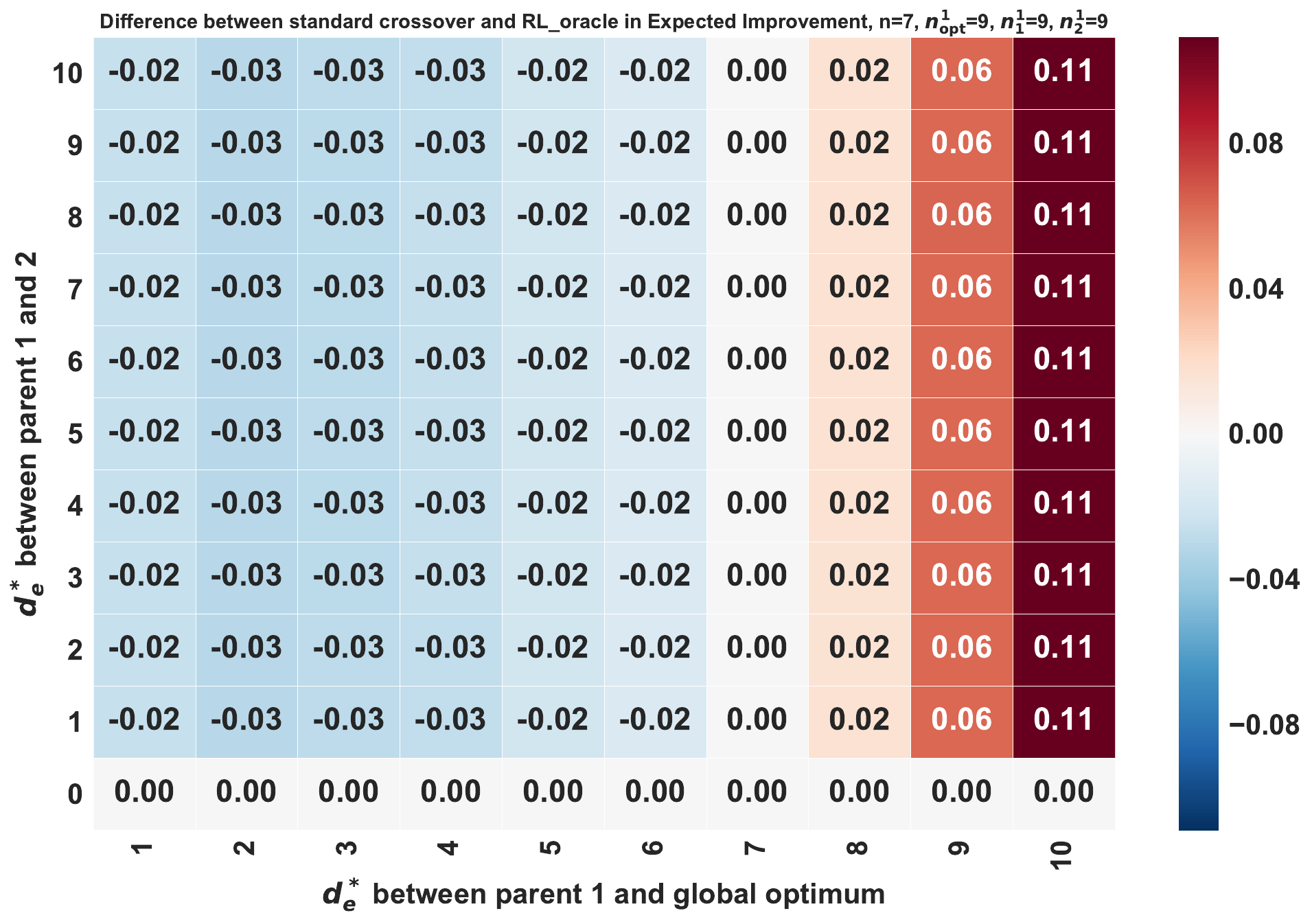}
	\caption{\textbf{Comparison of expected improvement between standard crossover and RL} (Left) $\mathrm{LBEI}_{\mathrm{STDX}}-\mathrm{LBEI}_{\mathrm{RLU}}$ under different $d_{e,\hat{\gG}_1, \hat{\gG}_2}^*$ ($y$-axis) and $d_{e,\hat{\gG}_{\mathrm{opt}}, \hat{\gG}_1}^*$ ($x$-axis) combinations. (right) $\mathrm{LBEI}_{\mathrm{STDX}}-\mathrm{LBEI}_{\mathrm{RLO}}$ under different $d_{e,\hat{\gG}_1, \hat{\gG}_2}^*$ ($y$-axis) and $d_{e,\hat{\gG}_{\mathrm{opt}}, \hat{\gG}_1}^*$ ($x$-axis) combinations. A $\alpha \cdot \eta$ value of 0.1 is used for RL. The standard crossover is slightly worse than both RL agents in most cases.
		\label{fig:stdx_rl}
	}
\end{figure}

\subsection{Additional Figures for Section~\ref{subsec:error_GED}}\label{subsec:add_error_GED}
As in Section~\ref{subsec:theory_comp}, Monte Carlo simulations with $10^6$ trials each were performed to estimate the values of $\mathrm{LBEI}_{\mathrm{SEPX}}^\epsilon$ under different error ratios $\epsilon$. Figure~\ref{fig:EI_101_error} compares $\mathrm{LBEI}_{\mathrm{SEPX}}^\epsilon$ with $\mathrm{LBEI}_{\mathrm{MUTA}}$, $\mathrm{LBEI}_{\mathrm{RLU}}$ and $\mathrm{LBEI}_{\mathrm{RLO}}$ under error ratios $\epsilon=$ 0.1, 0.2, and 0.3. Because Figure~\ref{fig:EI_stdx_mut} and ~\ref{fig:stdx_rl} already show that $\mathrm{LBEI}_{\mathrm{STDX}}$ is worse than $\mathrm{LBEI}_{\mathrm{MUTA}}$, $\mathrm{LBEI}_{\mathrm{RLU}}$ and $\mathrm{LBEI}_{\mathrm{RLO}}$ in most cases, $\mathrm{LBEI}_{\mathrm{STDX}}$ is not included in these comparisons.

The conclusion is that the SEP crossover has a theoretical advantage in expected improvement compared to mutation, standard crossover, and RL even with a very high error ratio of $30\%$ in the GED calculations. Thus, if the computational cost of the SEP crossover needs to be reduced, approximation methods can be used to calculate GED.

\begin{figure*}[ht]
	\centering
	\includegraphics[width=0.32\linewidth]{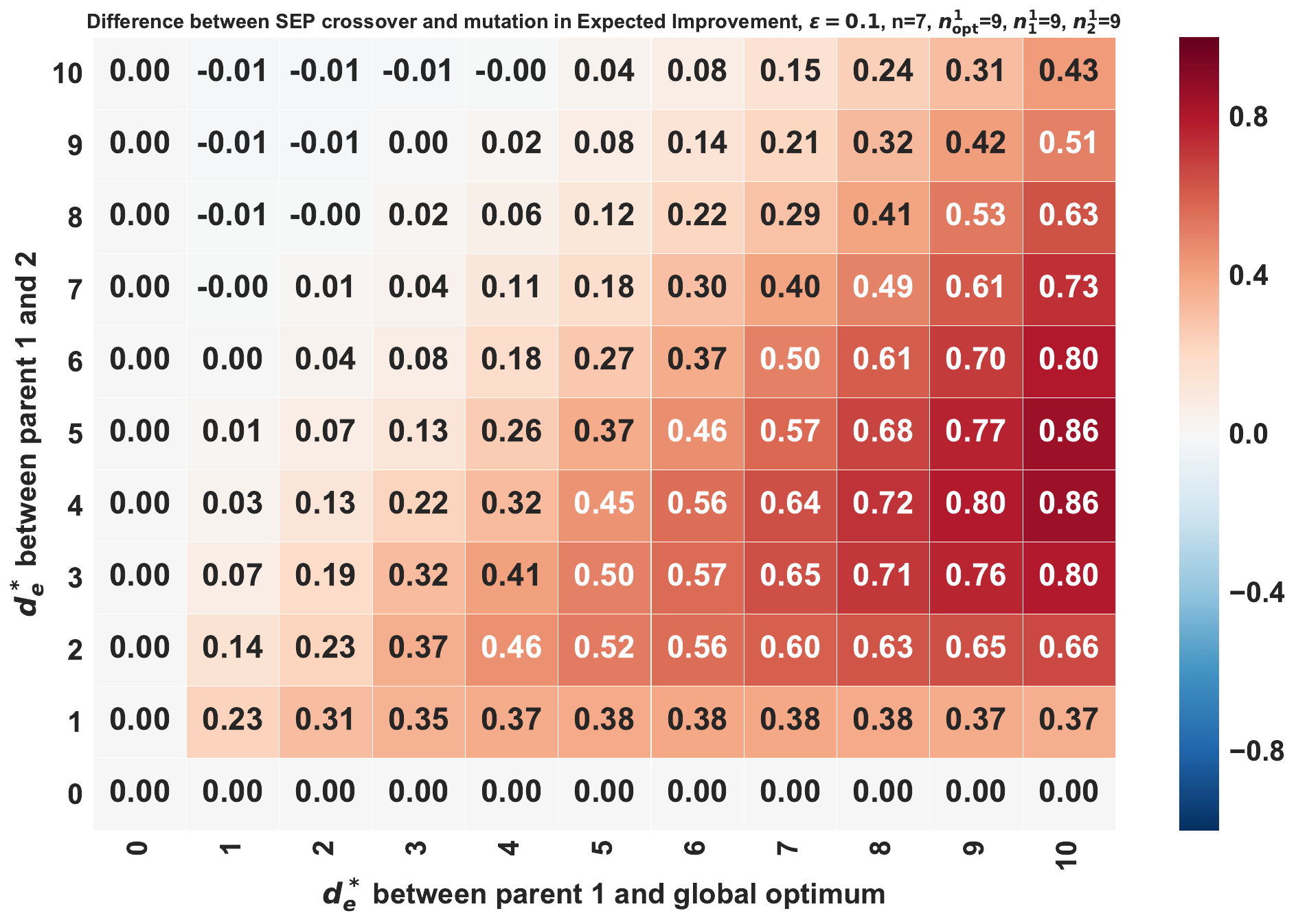}
	\includegraphics[width=0.32\linewidth]{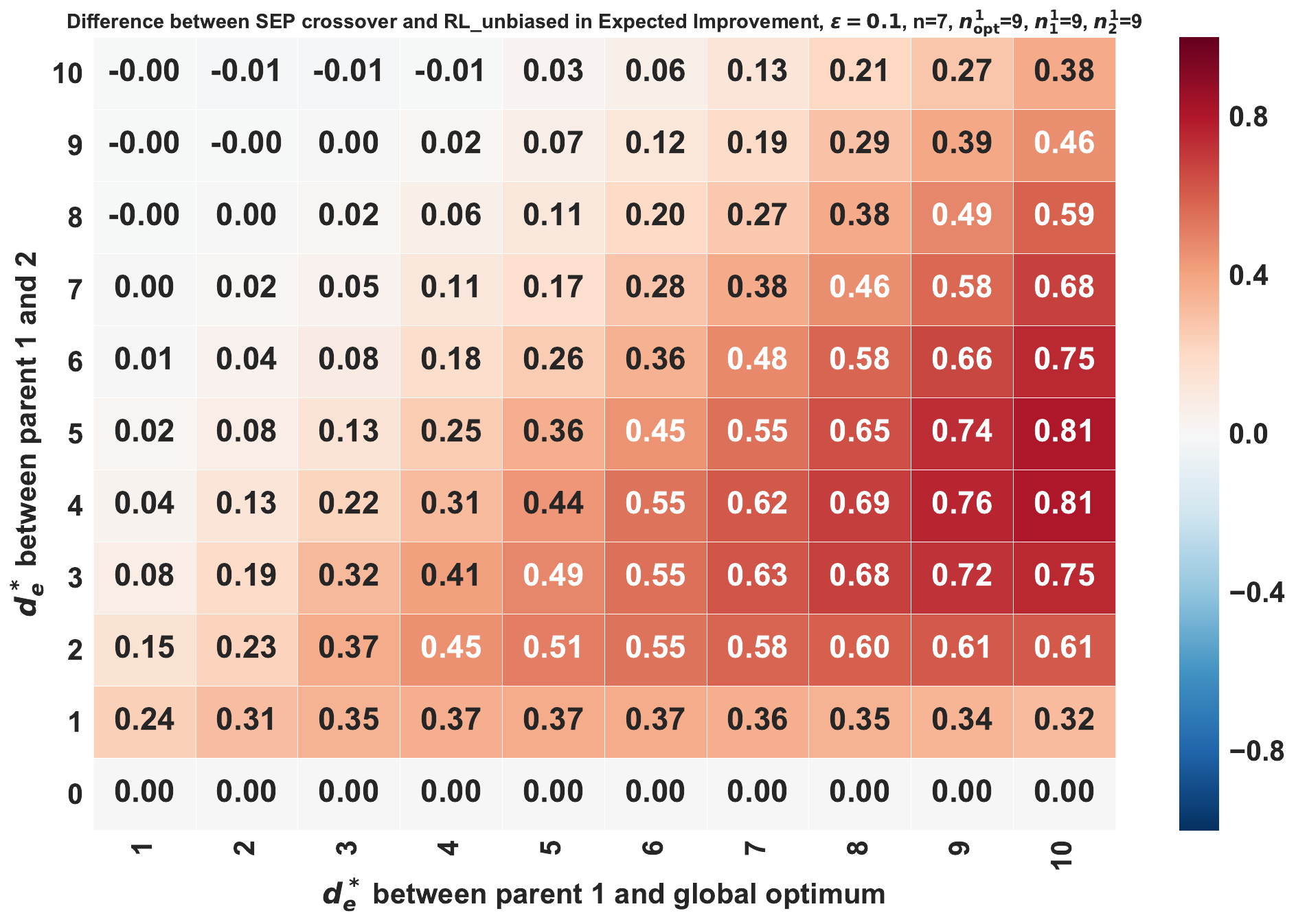}
	\includegraphics[width=0.32\linewidth]{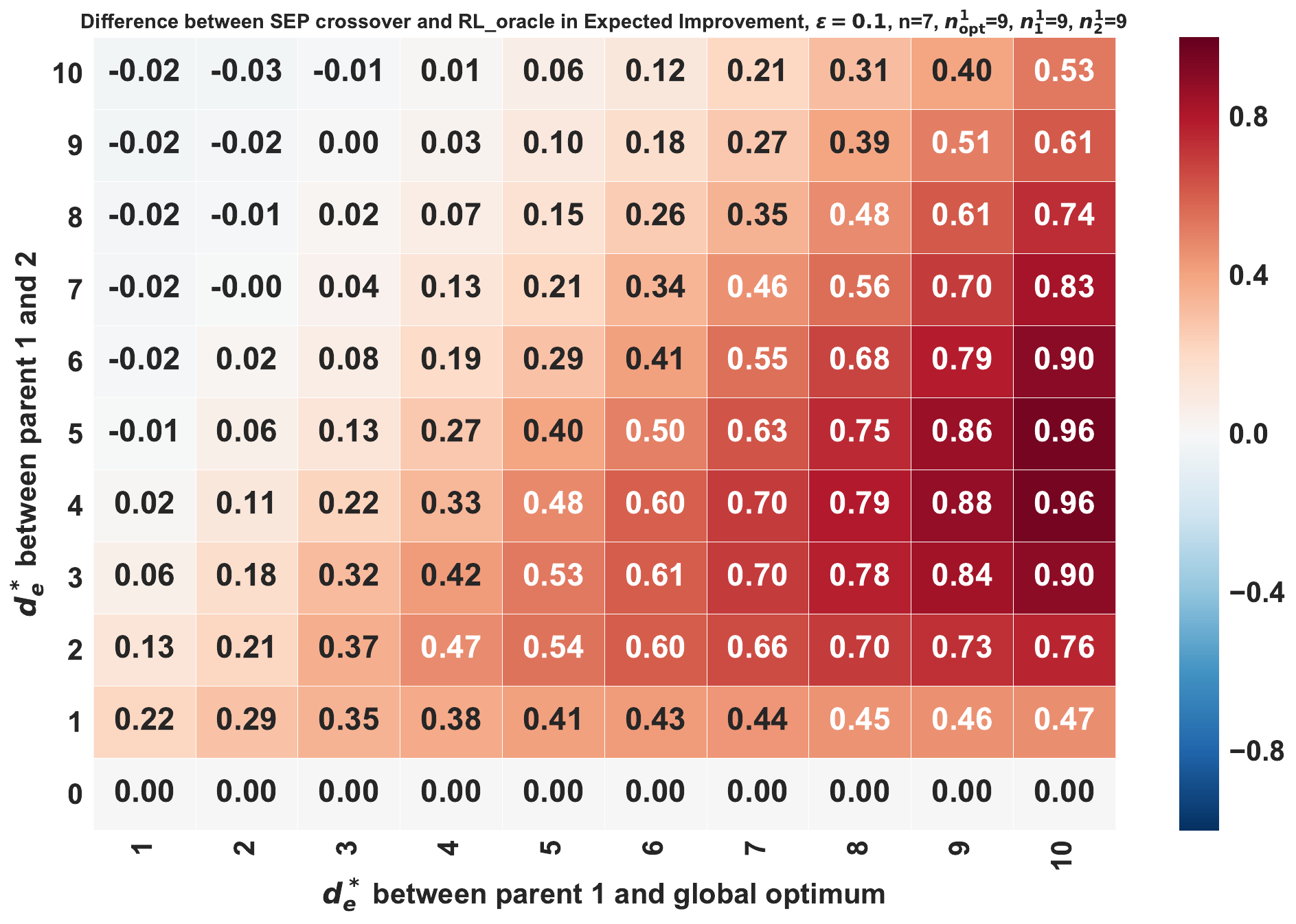}\\[-1ex]
	($a$) Error ratio \bm{$\epsilon=0.1$}\\
	\includegraphics[width=0.32\linewidth]{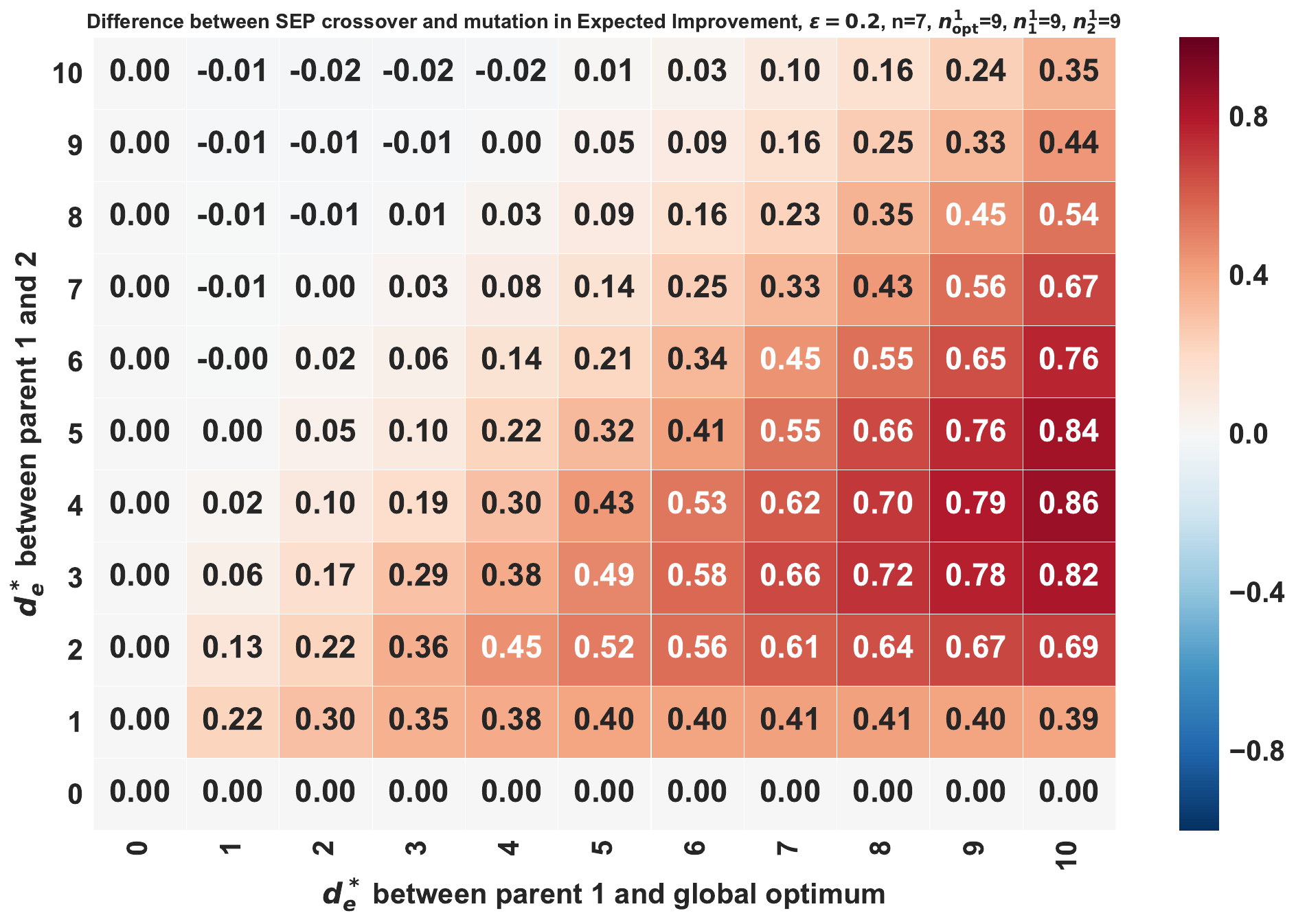}
	\includegraphics[width=0.32\linewidth]{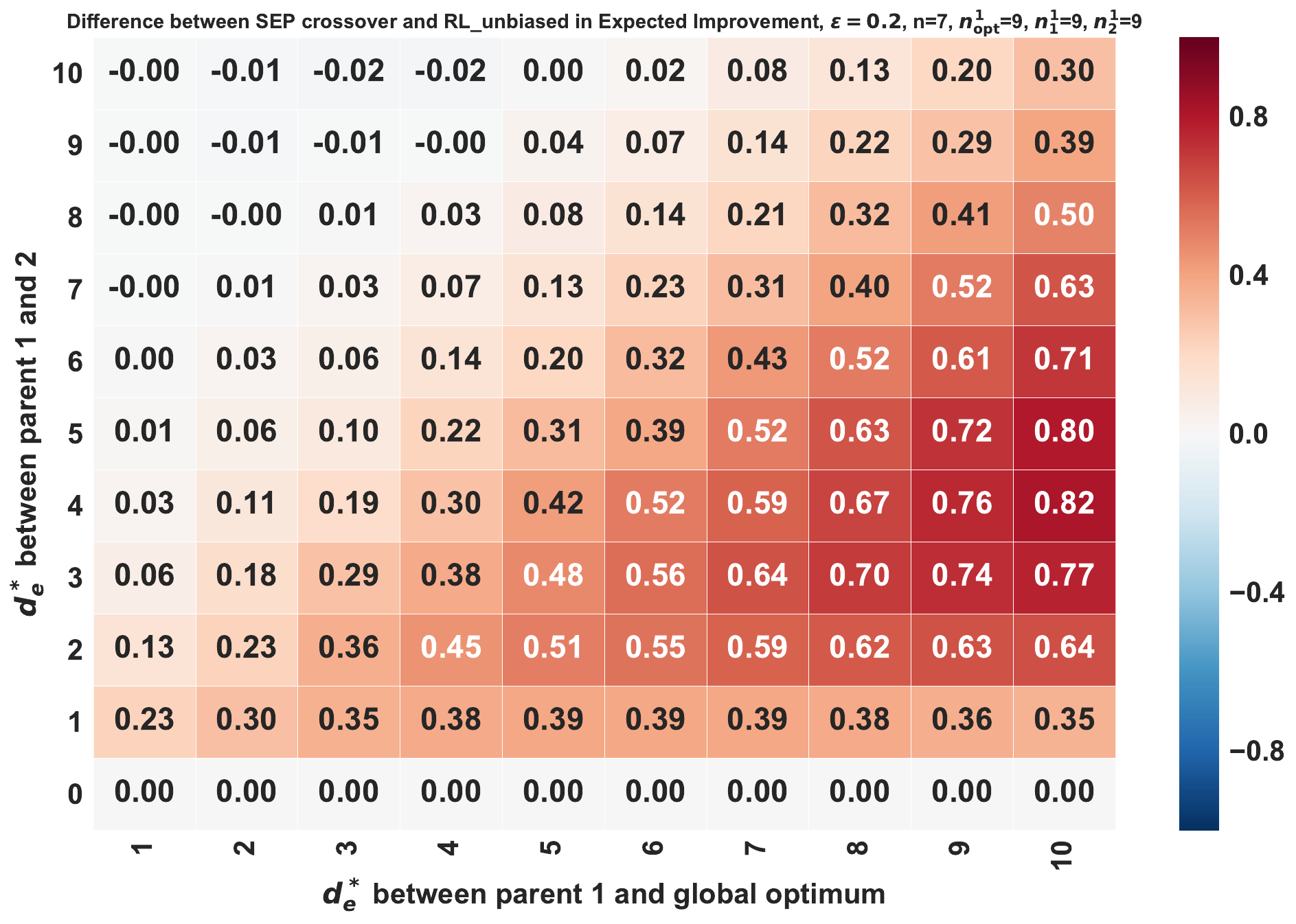}
	\includegraphics[width=0.32\linewidth]{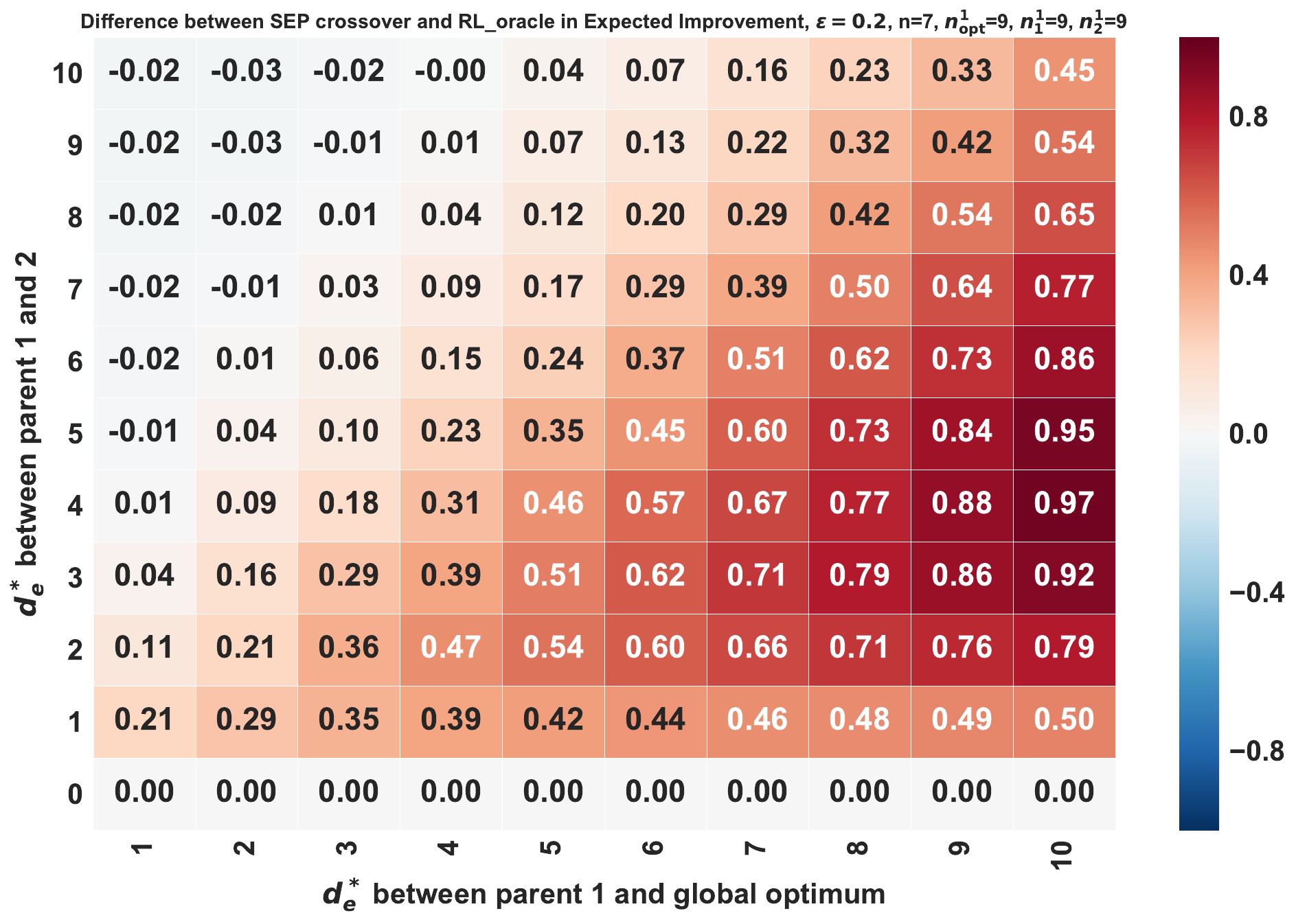}\\[-1ex]
	($b$) Error ratio \bm{$\epsilon=0.2$}\\
	\includegraphics[width=0.32\linewidth]{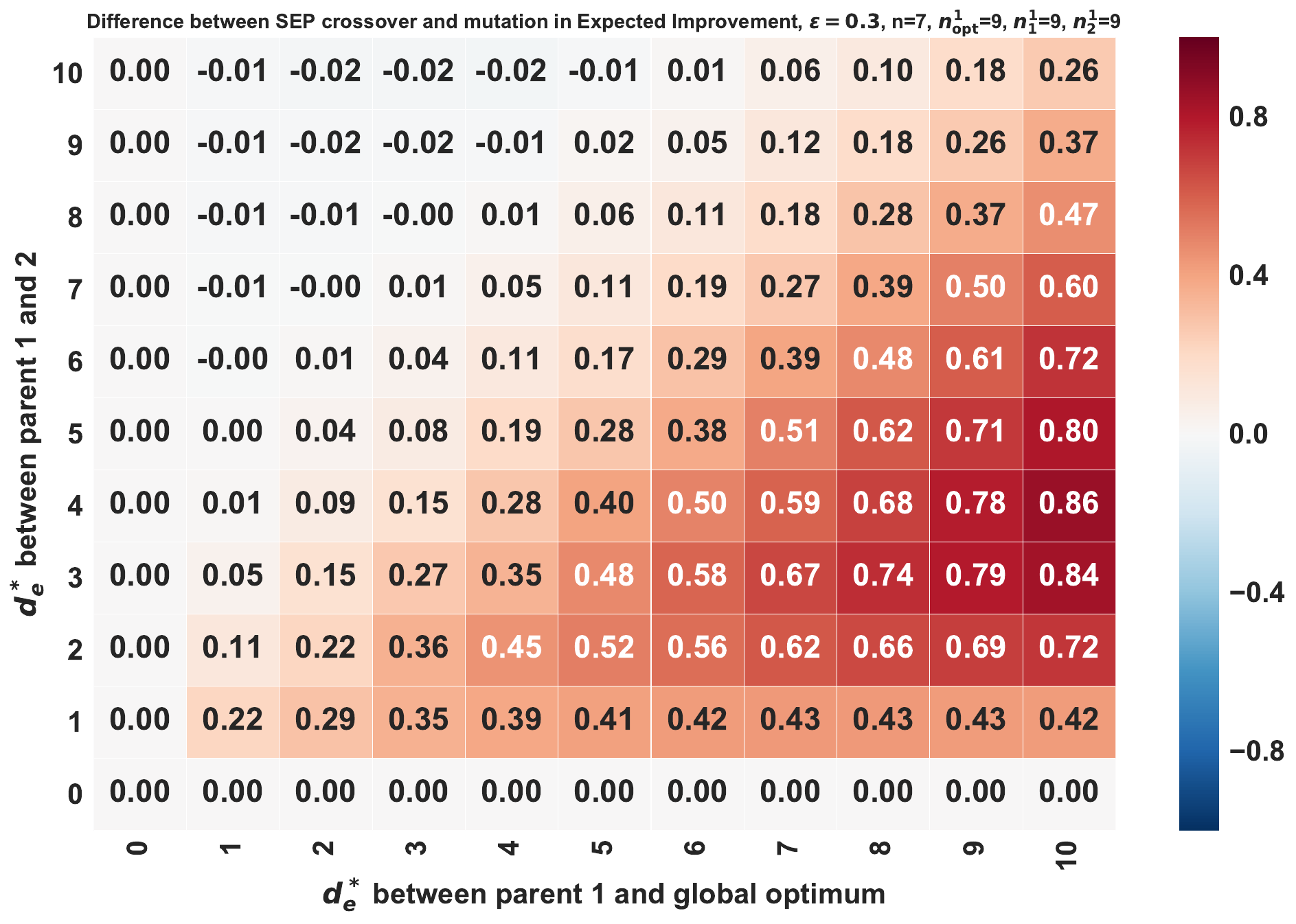}
	\includegraphics[width=0.32\linewidth]{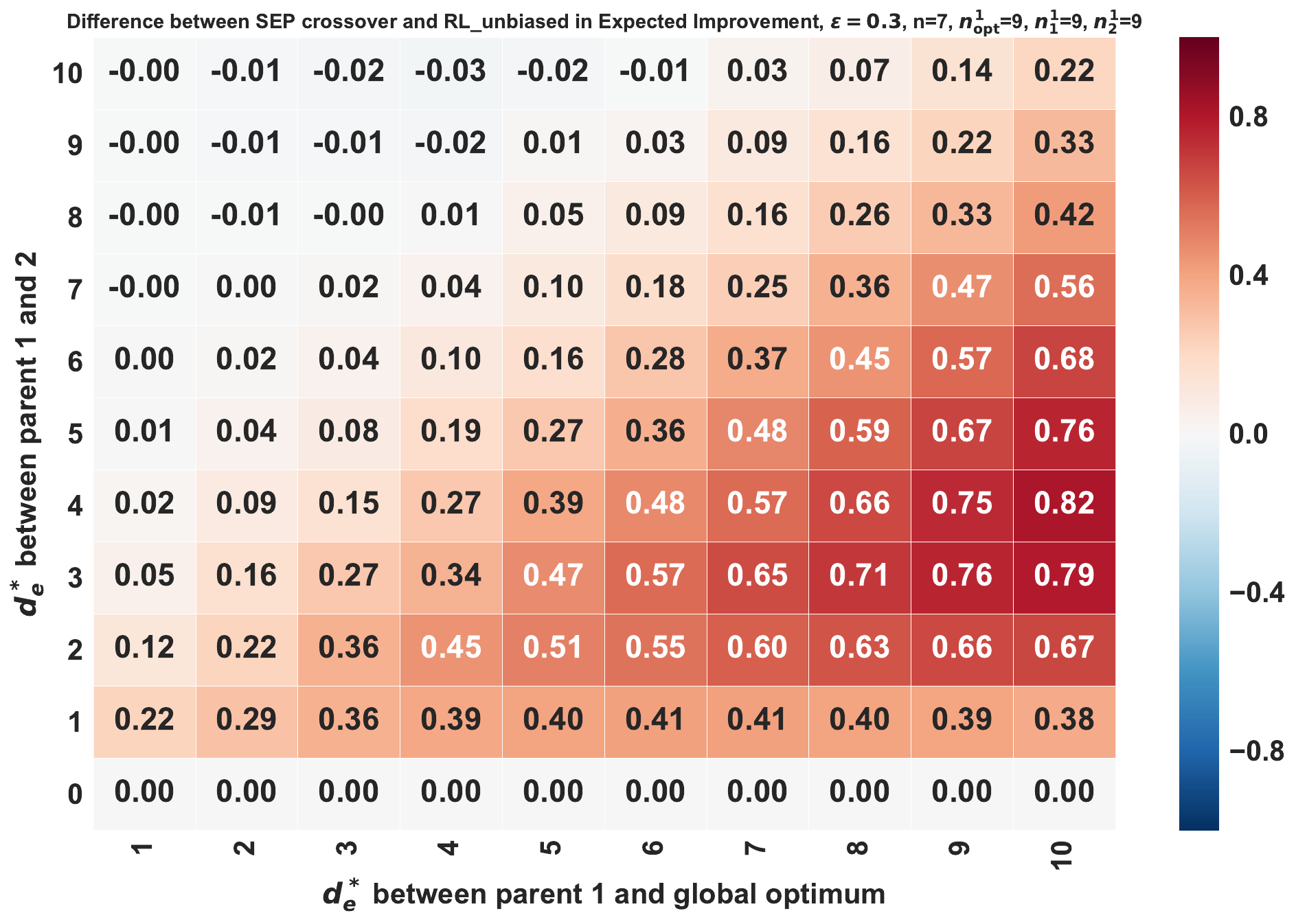}
	\includegraphics[width=0.32\linewidth]{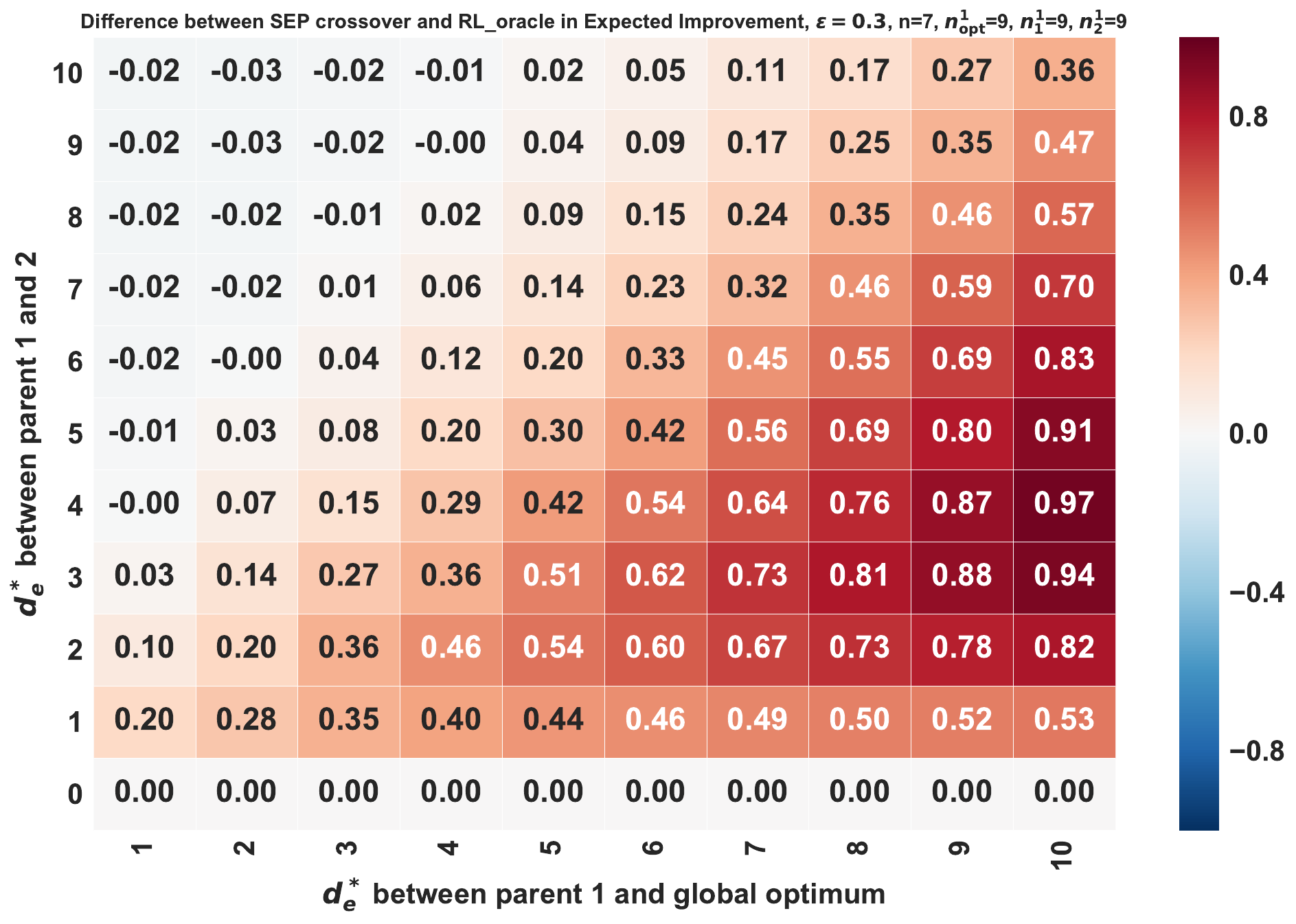}\\[-1ex]
	($c$) Error ratio \bm{$\epsilon=0.3$}\\
	\caption{\textbf{Comparison of expected improvement between SEP crossover, mutation, and RL in NAS-bench-101 at various level of GED calculation error.} (Left) Differences between $\mathrm{LBEI}_{\mathrm{SEPX}}^\epsilon$ and $\mathrm{LBEI}_{\mathrm{MUTA}}$ under different $d_{e,\hat{\gG}_1, \hat{\gG}_2}^*$ ($y$-axis) and $d_{e,\hat{\gG}_{\mathrm{opt}}, \hat{\gG}_1}^*$ ($x$-axis) combinations. (Middle) Differences between $\mathrm{LBEI}_{\mathrm{SEPX}}^\epsilon$ and $\mathrm{LBEI}_{\mathrm{RLU}}$. (Right) Differences between $\mathrm{LBEI}_{\mathrm{SEPX}}^\epsilon$ and $\mathrm{LBEI}_{\mathrm{RLO}}$. $\mathrm{LBEI}_{\mathrm{SEPX}}^\epsilon$ is larger (i.e.\ more red) than $\mathrm{LBEI}_{\mathrm{MUTA}}$, $\mathrm{LBEI}_{\mathrm{RLU}}$, and $\mathrm{LBEI}_{\mathrm{RLO}}$ in almost all cases. Thus, the SEP crossover has a theoretical advantage over mutation and RL even at a very high level of error in the GED calculations. Therefore, if needed, approximation methods can be used to reduce the computational cost of the SEP crossover.
		\label{fig:EI_101_error}
	}
 \vspace*{-2.5ex}
\end{figure*}

\subsection{Additional Figures for Section~\ref{subsec:applicability}}\label{subsec:add_applicability}
Figures~\ref{fig:freq_101} and~\ref{fig:freq_nlp} show relative frequencies of different parent combinations in NAS-bench-101 and NAS-bench-NLP, respectively. Note that the high relative frequencies of $d_{e,\hat{\gG}_{\mathrm{opt}}, \hat{\gG}_1}^*=0$ and $d_{e,\hat{\gG}_1, \hat{\gG}_2}^*=0$ are due to the convergence of the search algorithm, i.e.\ no further improvement can be made from them. The high-frequency areas match the assumptions of the theory, and thus the theoretical conclusions apply to real-world NAS.

\begin{figure}
	\centering
	\includegraphics[width=0.4\linewidth]{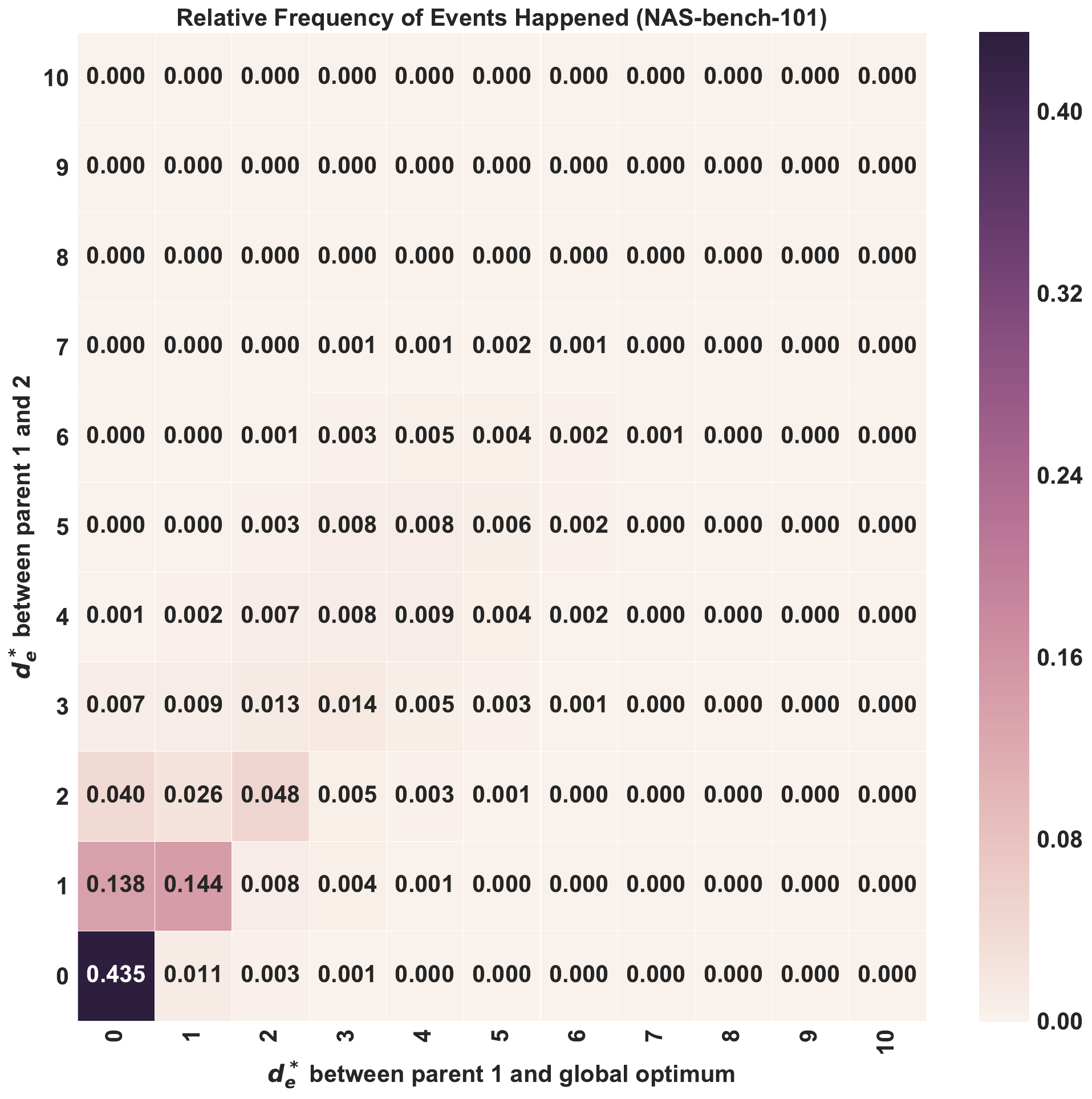}
	\includegraphics[width=0.4\linewidth]{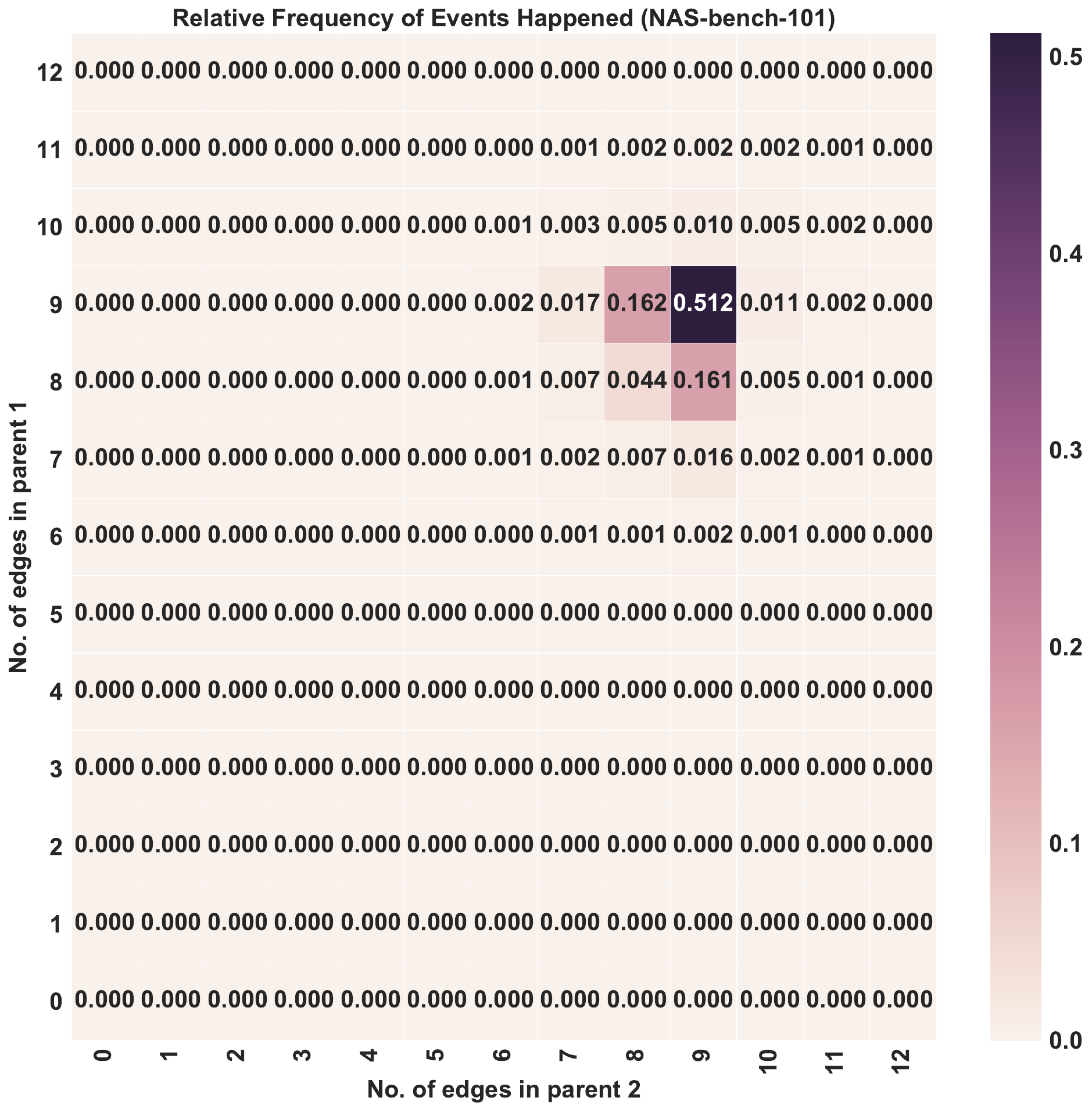}
	\caption{\textbf{Relative frequencies of different parent combinations in NAS-bench-101 experiments.} (Left) Relative frequencies of different $d_{e,\hat{\gG}_1, \hat{\gG}_2}^*$ ($y$-axis) and $d_{e,\hat{\gG}_{\mathrm{opt}}, \hat{\gG}_1}^*$ ($x$-axis) combinations. All events happen in the regions where the SEP crossover has a theoretical advantage in terms of expected improvement (as seen in Figure~\ref{fig:EI_101}). (Right) Relative frequencies of different $n_1^1$ and $n_2^1$ combinations. The event $n_1^1=9$ and $n_2^1=9$ happens most frequently during the experiments, and this setup is indeed used in the theoretical analysis. Thus, the theoretical analysis applies to situations that arise in NAS-bench-101.
		\label{fig:freq_101}
	}
\end{figure}
\begin{figure}
	\centering
	\includegraphics[width=0.4\linewidth]{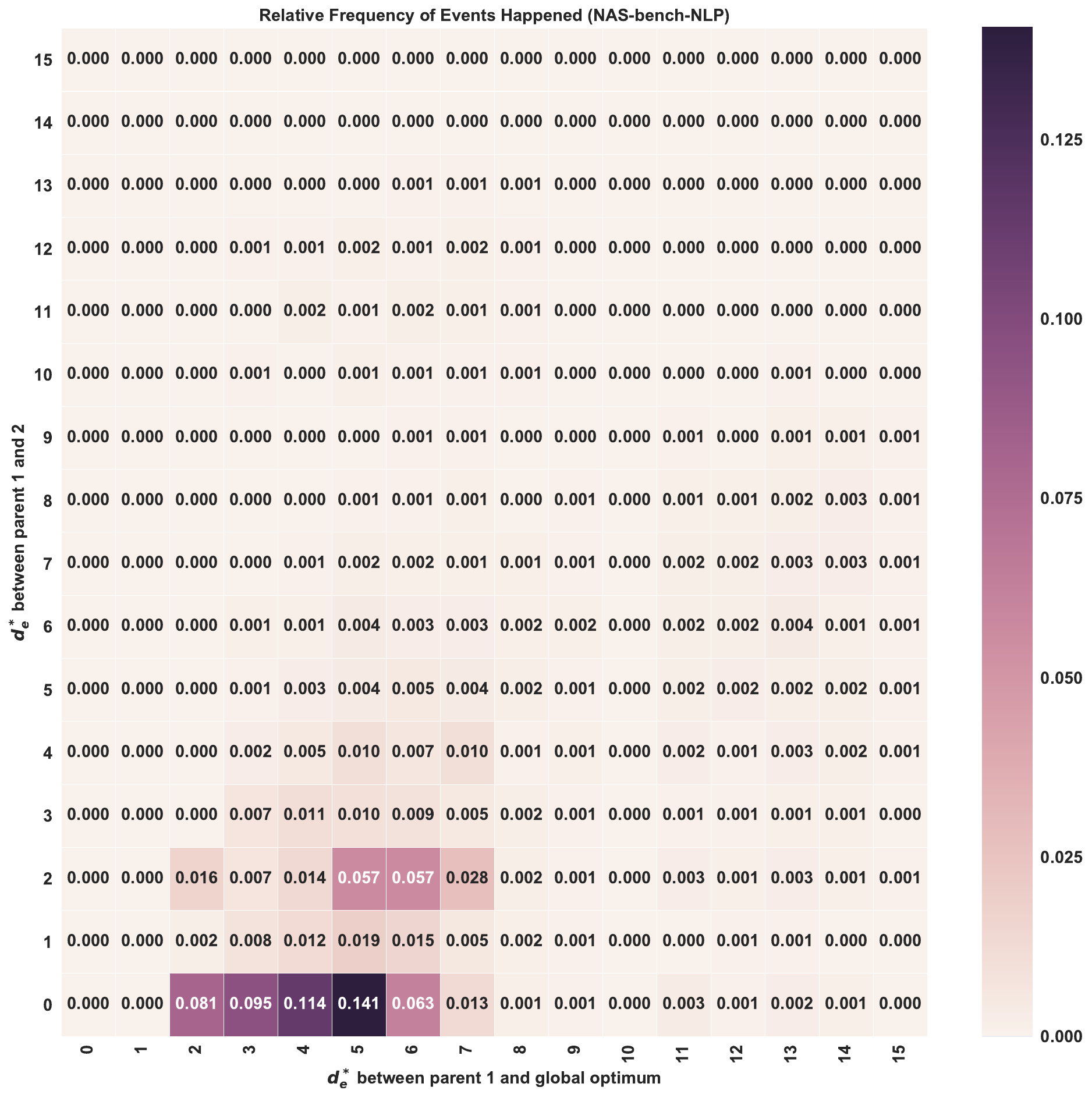}
	\includegraphics[width=0.4\linewidth]{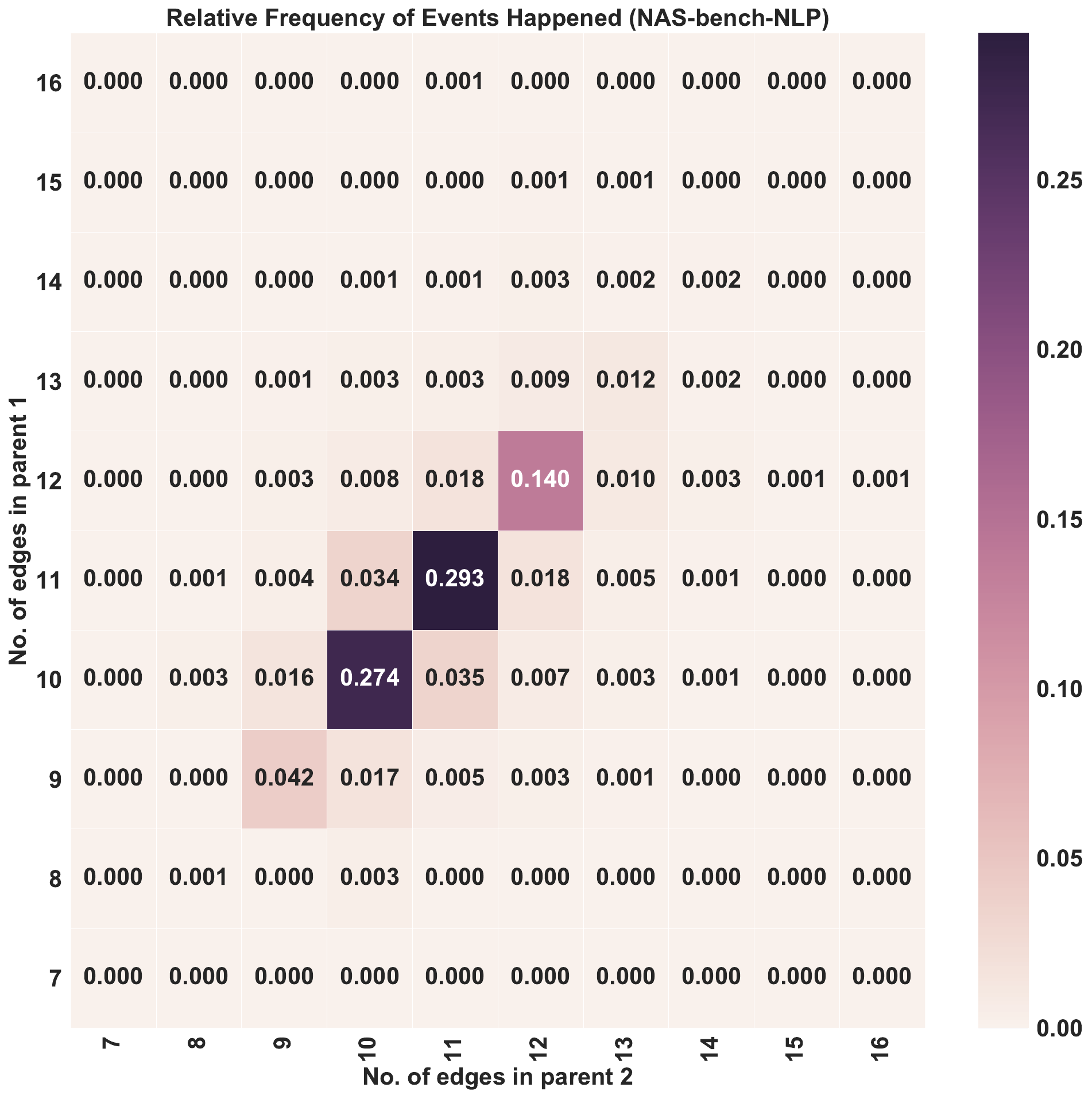}
	\caption{\textbf{Relative frequency of different parent combinations in NAS-bench-NLP experiments.} (Left) Relative frequencies of different $d_{e,\hat{\gG}_1, \hat{\gG}_2}^*$ ($y$-axis) and $d_{e,\hat{\gG}_{\mathrm{opt}}, \hat{\gG}_1}^*$ ($x$-axis) combinations. All the events happen in the regions where the SEP crossover has a theoretical advantage in terms of expected improvement (as seen in Figure~\ref{fig:EI_nlp}). (Right) Relative frequencies of different $n_1^1$ and $n_2^1$ combinations. Most events happen around the $n_1^1=11$ and $n_2^1=11$ combination, which is the setup used in the theoretical analysis. Together Figures~\ref{fig:freq_101} and~\ref{fig:freq_nlp} show that the theoretical analysis applies to the actual experimental settings.
		\label{fig:freq_nlp}
	}
\end{figure}

\subsection{Additional Results for Section~\ref{subsec:noisy}}\label{subsec:add_noise}

Figure~\ref{fig:empirical} shows that SEP crossover converges consistently faster than the other methods in all three benchmarks, i.e.\ NAS-bench-101, NAS-bench-NLP, and NAS-bench-301.

\begin{figure*}[t]
\centering
\hfill
a\includegraphics[width=0.30\linewidth]{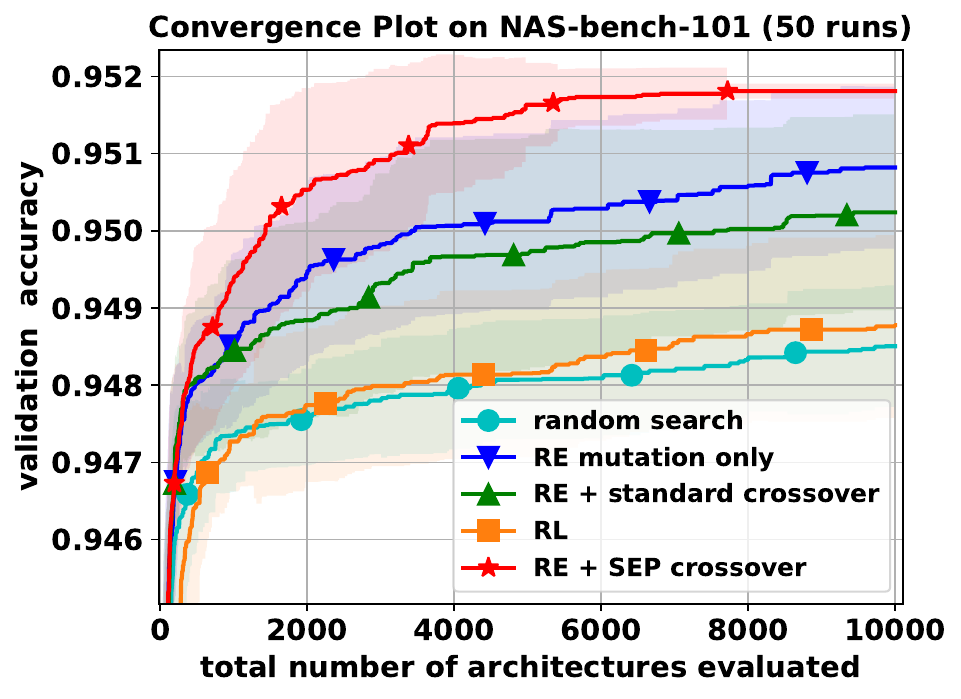}
\hfill
b\includegraphics[width=0.28\linewidth]{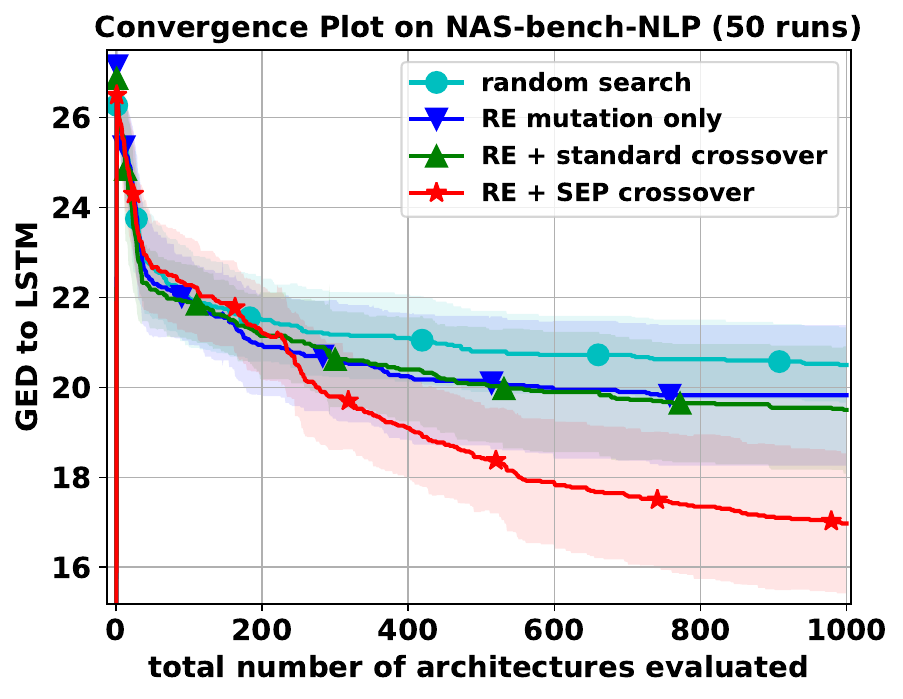}
\hfill
c\includegraphics[width=0.29\linewidth]{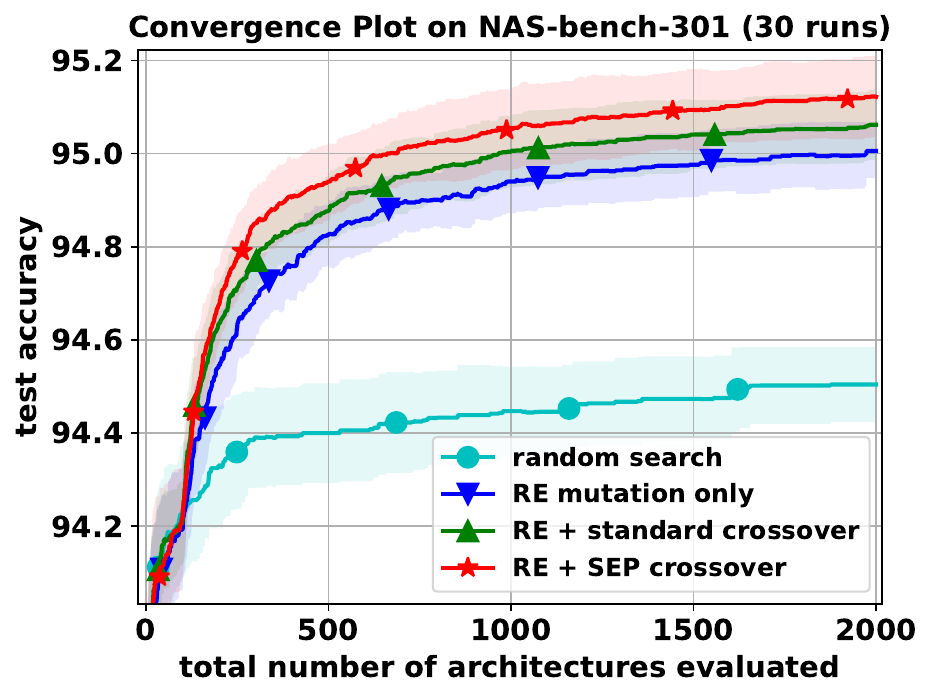}
\hfill
	\caption{\textbf{Performance of the search methods in three different NAS benchmarks.} 
(a) Convergence in NAS-bench-101. The plot shows the validation accuracy used as the direct fitness/reward for search strategies. The SEP crossover performs significantly better than the other approaches.
(b) Convergence in NAS-bench-NLP. This benchmark is a noise-free environment, and the plot shows the convergence of GED to LSTM. Again, the SEP crossover performs better than the other methods.
(c) Convergence in NAS-bench-301. The plot shows the surrogate-returned noise-free test accuracy. The noisy version of accuracy predicted by the surrogate model was used as the direct fitness/reward, and the noise-free accuracy (shown) as the final objective. The SEP crossover has consistently better search ability in this benchmark as well.
\label{fig:empirical}
	}
\end{figure*}

Note that since NAS-bench-NLP is not queryable, it is not possible to measure the prediction accuracy of each architecture in this benchmark the same way as that architecture is trained; the computational cost would be prohibitive given the scale of the experiments (i.e.\ multiple evaluations of multiple algorithms trained many times in each evaluation). Instead, the GED to GRU/LSTM is used as a noise-free fitness/reward for evaluating the performance of different methods in this complex search space.

Note also that NAS-bench-301 was not used in GED-related experiments because there are two separate graphs for each architecture (the normal cell and the reduction cell); currently the theory does not apply to pairs of graphs. Such an extension is left for future work.

\subsection{Comparison with Bayesian optimization}\label{subsec:add_bo}

Figure~\ref{fig:comparison_bo} compares the SEP crossover with two Bayesian optimization (BO) methods, namely BOHB \citep{Falkner18} and SMAC \citep{Hutter2011}. The parameter setup of BOHB and SMAC follows the guidelines in \citet{Ying19}. The SEP crossover consistently outperforms both BOHB and SMAC in NAS-bench-101.

\begin{figure*}[h]
	\centering
	\includegraphics[width=0.36\linewidth]{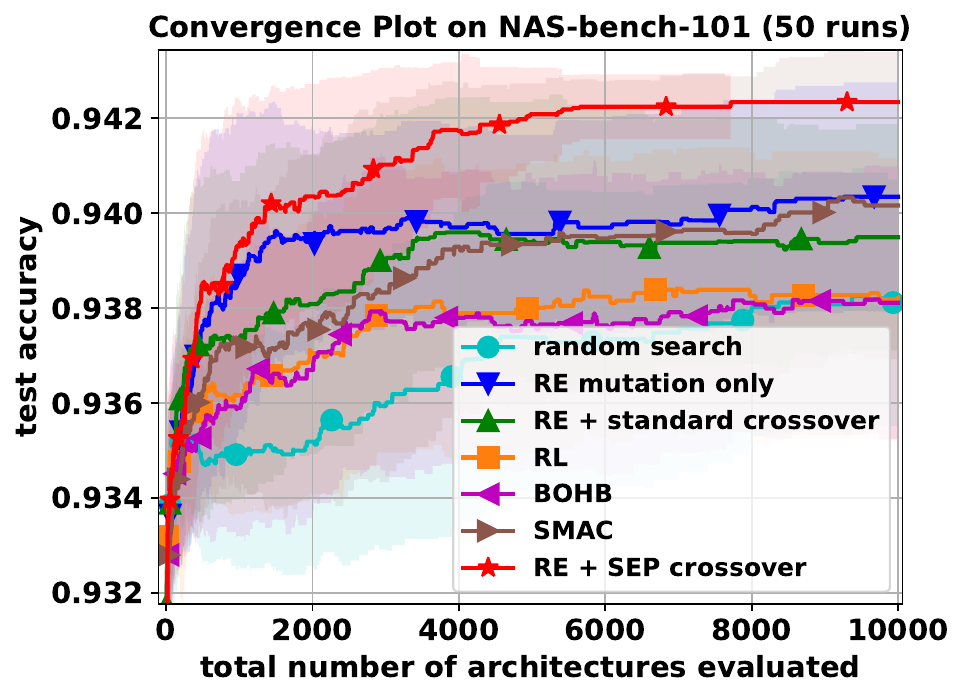}
	\caption{\textbf{Average test accuracy in NAS-bench-101.} The SEP crossover performs consistently better than the two Bayesian optimiation (BO) methods BOHB and SMAC.
		\label{fig:comparison_bo}
	}
\end{figure*}

\subsection{Comparison with path encoding}\label{subsec:add_path}

Figure~\ref{fig:comparison_path} compares the SEP crossover with a crossover operator based on path encoding \cite{White21b}. During a path encoding crossover, the offspring inherits the path with $100\%$ probability if this path is in both parents, and with $50\%$ if it is only in one of them. The SEP crossover significantly outperforms the path encoding crossover in NAS-bench-101.

\begin{figure*}[h]
	\centering
	\includegraphics[width=0.36\linewidth]{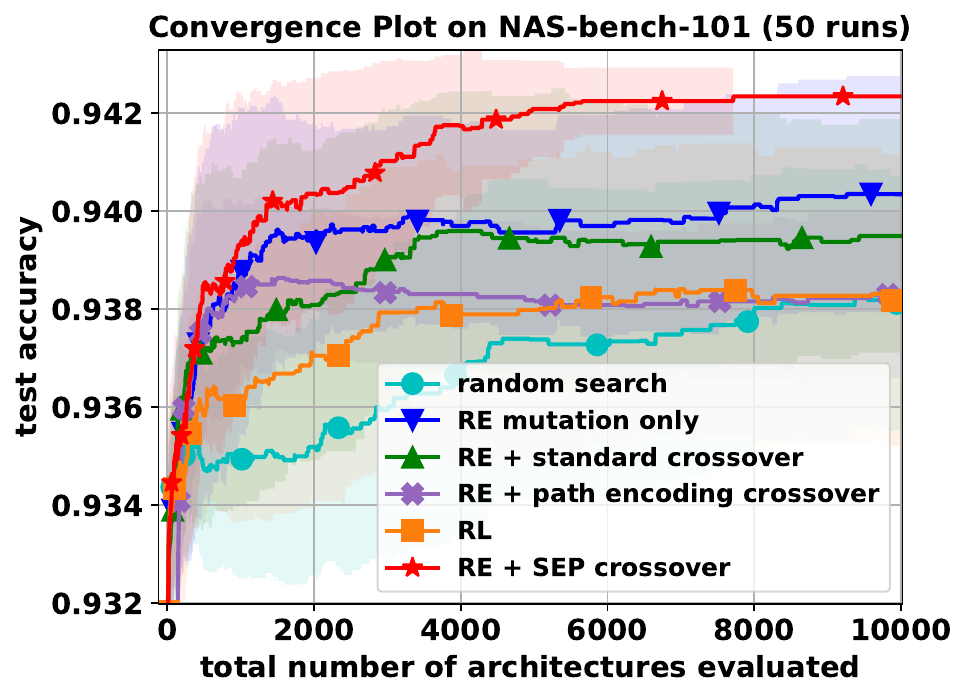}
	\caption{\textbf{Average test accuracy in NAS-bench-101.} The SEP crossover performs significantly better than the path encoding crossover.
		\label{fig:comparison_path}
	}
\end{figure*}

\subsection{Computational time of GED calculation}\label{subsec:add_ged_time}
The GED calculations in the experiments are based on the NetworkX library \citep{Hagberg08}, which implements an exact GED calculation method \citep{Abu-Aisheh15} with reasonable computational efficiency. In the experiments, the calculation time of GED was found to depend not only on the size of the two graphs, but also on the distance between them: If the two graphs are similar, the GED is calculated faster. As a characterization of the computational cost of GED calculations during SEP crossover, their average computation times under different parent distances and sizes are shown in Table~\ref{tab:ged_time}. These cases cover more than $99.9\%$ of the cases encountered in the experiments (according to Figure~\ref{fig:freq_101} and ~\ref{fig:freq_nlp}). All the GED computations ran on a single Intel(R) Xeon(R) Silver 4216 CPU @ 2.10GHz. According to Table~\ref{tab:ged_time}, the GED calculation time is almost negligible in NAS-bench-101 search space, and acceptable even in the largest NAS benchmark, i.e. NAS-bench-NLP.

\begin{table}[h]
	\centering
	\caption{\label{tab:ged_time} Computation Time of GED Calculation}
	\begin{tabular}{cccccccc}
		\toprule
		\multicolumn{8}{c}{NAS-Bench-101 (7 nodes)} \\
		\hline
		GED between parents & 1 & 2 & 3 & 4 & 5 & 6 & 7 \\
		computation time (s) & 0.009 & 0.012 & 0.019 & 0.029 & 0.041 & 0.060 & 0.084 \\
		\hline
		\multicolumn{8}{c}{NAS-Bench-NLP (12 nodes)	} \\
		\hline
		GED between parents & 1 & 3 & 5 & 7 & 9 & 11 & 13 \\
		computation time (s) & 0.015 & 0.046 & 0.281 & 1.156 & 3.190 & 10.374 & 21.957 \\
		\bottomrule
	\end{tabular}
\end{table}


\end{document}

%% file: math_commands.tex
\usepackage{amsmath,amsfonts,bm}

\def\eqref#1{equation~\ref{#1}}

\def\1{\bm{1}}

\def\mA{{\bm{A}}}
\def\mB{{\bm{B}}}
\def\mC{{\bm{C}}}

\def\mI{{\bm{I}}}

\def\mP{{\bm{P}}}
\def\mQ{{\bm{Q}}}

\DeclareMathAlphabet{\mathsfit}{\encodingdefault}{\sfdefault}{m}{sl}
\SetMathAlphabet{\mathsfit}{bold}{\encodingdefault}{\sfdefault}{bx}{n}

\def\gG{{\mathcal{G}}}

\def\sV{{\mathbb{V}}}

\newcommand{\E}{\mathbb{E}}

